\documentclass[twoside,11pt]{article}
\usepackage{journal2e_v1}

\usepackage{epsf}
\usepackage{epsfig}
\usepackage{amsmath}
\usepackage{subfigure}
\usepackage{amssymb}
\usepackage{algorithm}
\usepackage{multirow}
\usepackage{multicol}
\usepackage{array}
\usepackage{enumitem}
\usepackage{url}

\usepackage{color}
\usepackage[dvipsnames]{xcolor}

\makeatletter
\newcommand{\mylabel}[2]{#2\def\@currentlabel{#2}\label{#1}}
\makeatother

\usepackage{graphicx} % more modern
\usepackage{ltxtable} % needed for vertical text centering hack -cotter
\newcolumntype{L}{>{\centering\arraybackslash} m{0.04\columnwidth}} % -cotter
\newcolumntype{R}{>{\centering\arraybackslash} m{0.48\columnwidth}} % -cotter
\newcolumntype{S}{>{\centering\arraybackslash} m{0.32\columnwidth}} % -cotter

\usepackage{algorithm}
\usepackage{algorithmic}

\usepackage[pagebackref=true,breaklinks=true,letterpaper=true,colorlinks,bookmarks=yes]{hyperref}

\def\A{{\bf A}}
\def\a{{\bf a}}
\def\B{{\bf B}}
\def\bb{{\bf b}}

\def\I{{\bf I}}
\def\K{{\bf K}}
\def\k{{\bf k}}

\def\U{{\bf U}}
\def\u{{\bf u}}
\def\V{{\bf V}}
\def\v{{\bf v}}

\def\w{{\bf w}}

\def\x{{\bf x}}

\def\y{{\bf y}}
\def\Z{{\bf Z}}

\def\0{{\bf 0}}
\def\1{{\bf 1}}

\def\EM{{\mathcal E}}

\def\HM{{\mathcal H}}

\def\NM{{\mathcal N}}
\def\OM{{\mathcal O}}

\def\VM{{\mathcal V}}

\def\XM{{\mathcal X}}

\def\RB{{\mathbb R}}
\def\EB{{\mathbb E}}
\def\PB{{\mathbb P}}

\def\bet{\mbox{\boldmath$\beta$\unboldmath}}

\def\Ph{\mbox{\boldmath$\Phi$\unboldmath}}
\def\ph{\mbox{\boldmath$\phi$\unboldmath}}
\def\Ps{\mbox{\boldmath$\Psi$\unboldmath}}
\def\ps{\mbox{\boldmath$\psi$\unboldmath}}
\def\tha{\mbox{\boldmath$\theta$\unboldmath}}

\def\Si{\mbox{\boldmath$\Sigma$\unboldmath}}

\def\De{\mbox{\boldmath$\Delta$\unboldmath}}

\def\Xii{\mbox{\boldmath$\Xi$\unboldmath}}

\def\argmin{\mathop{\rm argmin}}

\def\mean{\mathsf{mean}}

\def\sgn{\mathrm{sgn}}

\newtheorem{assumption}{Assumption}
\begin{document}
\title{Simple and Almost Assumption-Free Out-of-Sample Bound for Random Feature Mapping}

\author{\name  Shusen Wang \\
        \addr shusen.wang@stevens.edu \\
            Department of Computer Science\\ 
            Stevens Institute of Technology \\
            Hoboken, NJ 07030, USA \\
%        \AND 
%        \name Name \\
%        \addr Addr \\
        }

%\date{Revised, Jan 31, 2017}
\date{}

\maketitle

\vspace{-15mm}

\begin{abstract}%
Random feature mapping (RFM) is a popular method for speeding up kernel methods at the cost of losing a little accuracy. 
We study kernel ridge regression with random feature mapping (RFM-KRR) and establish novel out-of-sample error upper and lower bounds.
While out-of-sample bounds for RFM-KRR have been established by prior work, this paper's theories are highly interesting for two reasons.
On the one hand, our theories are based on weak and valid assumptions.
In contrast, the existing theories are based on various uncheckable assumptions, which makes it unclear whether their bounds are the nature of RFM-KRR or simply the consequence of strong assumptions.
On the other hand, our analysis is completely based on elementary linear algebra and thereby easy to read and verify.
Finally, our experiments lend empirical supports to the theories.
\end{abstract}


\section{Introduction}
\label{sec:intro}

Supervised machine learning uses past experience (training data), such as a set of feature-label pairs $(\x_1, y_i), \cdots , (\x_n, y_n) \in \RB^d \times \RB$, to make prediction.
The objective is to learn a function $f: \RB^d \mapsto \RB$ from the training data and use $f$ to predict the target of a never-seen-before datum. 
The (generalized) linear models, where $f (\x) = g(\x^T \w)$, are simple and very popular.
Here $g$ is a link function, and the vector $\w$ is learned using the $n$ feature-label pairs.
Typical examples are the linear regression and logistic regression, where the $g$ function is respectively the identity function and logistic function.
Unfortunately, the generalized linear models lack expressive power, especially when $n \gg d$.
In real-world problems, the target can be a complicated function of the feature vector $\x$, in which case simple generalized linear models do not apply.

A simple and effective approach to higher expressive power is random feature mapping (RFM).
RFM was firstly proposed by \cite{rahimi2007random} for speeding-up kernel machines, and it won the NIPS Test-of-time award in 2017. 
RFM automatically maps input vectors into high-dimensional feature vectors, which can be then fed to any machine learning models such as the ridge regression, support vector machine, and $k$-means clustering, etc.
The high-dimensional random features generally improve the training and testing errors of the generalized linear models.
On the speech recognition dataset, TIMIT, RFM is reported to match deep neural networks \cite{huang2014kernel,may2017kernel}.

RFM was originally proposed to approximate large-scale kernel matrices in order to speed up kernel machines \citep{rahimi2007random}.
RFM can also be thought of as a two-layer neural network with a wide and randomly initialized hidden layer and a fine-tuned output layer.
From the machine learning perspective, the most important question is the generalization to never-seen-before test samples. 
If the kernel matrix is approximated using RFM, will the out-of-sample prediction be much different?

The generalization of kernel ridge regression with random feature mapping (RFM-KRR) has been studied by prior work such as \cite{avron2017random,cortes2010impact,rudi2017generalization,yang2012nystrom} (which we will discuss later.)
The strongest generalization bound was established by \cite{rudi2017generalization} which makes assumptions on the data, kernel function, and the RFM.
It is unclear whether the strong generalization property is the nature of RFM-KRR or a consequence of the assumptions.

To show the nature of RFM-KRR's generalization, we avoid making any assumption on the data and kernel function;
our sole assumption is that the random feature map is unbiased and bounded.
Such a property is enjoyed by the popular random Fourier features \citep{rahimi2007random} and random sign features \citep{tropp2015introduction}.
We show in Theorem~\ref{thm:main} that the out-of-sample prediction made by RFM-KRR is close to that by KRR.
We further establish a lower bound in Theorem~\ref{thm:lower1} that almost matches the upper bound in Theorem~\ref{thm:main}, indicating that our upper bound is near optimal.

The rest of this paper is organized as follows.
Section~\ref{sec:intro:rfm} briefly introduces RFM for kernel approximation.
Section~\ref{sec:intro:related} compares with related work.
Section~\ref{sec:intro:theory} presents our main theoretical findings.
Section~\ref{sec:notation} defines the notation used throughout and briefly introduces random feature mapping (RFM) and kernel ridge regression (KRR);
the most frequently used notation is listed in Table~\ref{tab:notation}.
Section~\ref{sec:main} formally present our main theorems.
Section~\ref{sec:analysis} proves the upper bound using the properties of RFM and random matrix theories.
Section~\ref{sec:lower} proves the lower bound.

\begin{table}
	\label{tab:notation}
	\setlength{\tabcolsep}{0.3pt}
	\caption{Commonly used notation.}
	\begin{center}
		\begin{footnotesize}
			\begin{tabular}{c l}
				\hline
				~{\bf Notation}~&~~~{\bf Definition}~~~\\
				\hline
				~~$n$~~ & ~~~total number of samples~~~\\
				~~$s$~~ & ~~~number of random features~~~\\
				~~$\lambda$~~ & ~~~ridge regularization parameter~~~\\
				~~$\VM_s$~~ & ~~~a set of $s$ random vectors for feature mapping~~~\\
				~~$\XM_n $~~ & ~~~set of traing samples, $\{\x_1, \cdots , \x_n \}$~~~\\
				~~$\y$~~ & ~~~target vector $[y_1 , \cdots , y_n]$~~~\\
				~~$\x'$~~ & ~~~test sample~~~\\
				~~$\kappa (\cdot , \cdot)$~~ & ~~~kernel function~~~\\
				~~$\K$~~ & ~~~$n\times n$ kernel matrix of the training samples~~~\\
				~~$\tilde{\K}$~~ & ~~~approximation to $\K$ by random features~~~\\
				~~$f_{\lambda} (\x'; \XM_n)$~~ & ~~~prediction made by kernel ridge regression (KRR)~~~\\
				~~$\tilde{f}_{\lambda} (\x'; \XM_n, \VM_s)$~~ & ~~~prediction made by RFM-KRR~~~\\
				
				\hline
			\end{tabular}
		\end{footnotesize}
	\end{center}
\end{table}

\subsection{RFM and kernel methods} \label{sec:intro:rfm}

In the original paper \citep{rahimi2007random},
RFM was designed for approximating the shift-invariant kernels, such as the radial basis function (RBF) kernel, to speed up the training and prediction of kernel machines \citep{scholkopf2002learning}.
The basic idea of RFM is representing a kernel function as an integral and approximate the integral by finite (say $s$) Monte Carlo samples.
Many kernel functions can be expressed as 
\begin{equation*}
\kappa (\x, \x')
\: = \: \int \psi (\x ; \v ) \psi (\x' ; \v) p (\v) d \v ,
\end{equation*}
for some functions $\psi (\cdot; \cdot)$ and $p (\cdot )$;
With $\v_1, \cdots , \v_s$ sampled according to the PDF $p (\cdot)$,
RFM approximates the kernel function by
\begin{equation*}
\kappa (\x, \x')
\: \approx \: \frac{1}{s} \sum_{i=1}^s \psi (\x ; \v_i ) \psi (\x' ; \v_i ) .
\end{equation*}
For $s \ll n$, RFM can significantly speed up the training and prediction of kernel machines.
Take the kernel ridge regression (KRR) for example.
Taking $(\x_1, y_i), \cdots , (\x_n, y_n) \in \RB^d \times \RB$ as training samples,
KRR costs $\OM (n^2 d)$ time to form the kernel matrix $\K$, $\OM (n^3)$ time to solve an $n\times n$ linear system for training, and $\OM (nd )$ time for making prediction for a single test sample.
Using $s$ random features to approximate $\K$, KRR with RFM (abbr.\ RFM-KRR) requires merely $\OM (n s^2)$ time for training and $\OM (s d)$ time for prediction.

\subsection{Related work} \label{sec:intro:related}

There have been many prior works on the theoretical properties of RFM.
One line of works studied linear algebraic objectives, in particular, 
some matrix norm errors $\| \K - \tilde{\K} \|$
where $\K$ is the true kernel matrix and $\tilde{\K}$ is the low-rank approximation by RFM.
\cite{rahimi2007random,le2013fastfood} showed infinity-norm bounds;
\cite{lopez2014randomized,tropp2015introduction} established spectral norm bounds;
\cite{sriperumbudur2015optimal} studied more general norms.

The goal of RFM is to make training and prediction more efficient without much hurting the prediction performance.
Therefore, compared with the matrix norm bounds, the impact on the out-of-sample prediction performance is more relevant to machine learning.
The pioneering work by \cite{cortes2010impact} reduces the prediction error of RFM-KRR to the spectral norm error; however, their bound is much weaker than the other mentioned work.
\cite{avron2017random} established a statistical risk bound for RFM-KRR which is conditioned on the assumption that RFM forms a spectral approximation to the kernel matrix;
unfortunately, it is unclear whether such a spectral approximation exists. 
\cite{yang2012nystrom} showed a $\tilde{\OM} ( n^{-\frac{1}{2}} + s^{-\frac{1}{2}} )$ generalization bound based on strong assumptions.
With assumptions on the data, kernel function, and random features, \citep{rudi2017generalization} established that
with $s = \Theta ( \sqrt{n} \log n )$ random features and $\lambda = \frac{1}{\sqrt{n}}$,
\begin{equation} \label{eq:rudi}
\EM (\tilde{f}_{\lambda} ) - \EM ( f_{\HM} ) 
\: = \: {\OM} \big(  \tfrac{ \sqrt{b} \|\y \|_\infty + b \| f_{\HM} \|_{\HM} }{\sqrt{n}} \big) ,
\end{equation}
where $\EM$ is the mean squared prediction error, $\HM$ is the RKHS, $f_{\HM} = \argmin_{f} \EM (f)$, and $b = \sup_{\x, \v} | \psi (\x; \v) |$.
This is the strongest bound for RFM-KRR.
\cite{bach2017equivalence,rahimi2009weighted} showed a generalization bound for RFM;
however, their work does not apply to KRR (because the quadratic loss function is not Lipschitz and violates their assumptions) and is thus less relevant to our work.

Aside from KRR, RFM has been applied to speed-up kernel principal component analysis (KPCA) \cite{scholkopf1998nonlinear}.
Statistical consistency of RFM-KPCA has been shown by \cite{shawetaylor2005onthe,blanchard2007statistical,sriperumbudur2017approximate}.
Convergence rates of RFM-KPCA have been established by \cite{lopez2014randomized,ghashami2016streaming,ullah2018streaming}.

An alternative approach to kernel approximation is the Nystr\"om method \citep{nystrom1930praktische,williams2001using} which forms low-rank approximation to the kernel matrix.
\cite{drineas2005nystrom,gittens2013revisiting,kumar2012sampling,wang2013improving,musco2017recursive,tropp2017,wang2018scalable} established matrix norm bounds for the Nystr\"om method.
\cite{bach2013sharp,alaoui2015fast,rudi2015less} provided statistical analysis for KRR with Nystr\"om approximation.

\subsection{Our main results and contributions} \label{sec:intro:theory}

\paragraph{Uppder bound.}
In this paper, we analyze the out-of-sample prediction by making only one mild assumption that $\psi (\cdot ; \cdot)$ is unbiased and bounded within $\pm \sqrt{b}$.
We do not make assumption on the data and kernel function.
We show in Theorem~\ref{thm:main} that for $\lambda = \tilde{\Theta} (\sqrt{n} )$, it holds with high probability that
\begin{equation} \label{eq:intro:upper}
\EB_{\x', \VM_s} \Big[  \Big(  \tilde{f}_{\lambda} (\x'; \XM_n, \VM_s) - f_{\lambda} (\x'; \XM_n) \Big)^2 \Big]
\: \leq \: \frac{4 b}{s } \, \big\| \K^{1/2} (\K + n \lambda \I_n)^{-1} \y \big\|_2^2 ,
\end{equation}
where the expectation is taken w.r.t.\ the random testing sample $\x'$ (from the same distribution as the training samples) and the random features $\VM_s$, and the failure probability is from the random training samples.
The bound indicates that the predictions made by RFM-KRR converges to KRR as $s$ grows.
The righthand-side of \eqref{eq:intro:upper} depends on the data and kernel function but is independent of RFM.

\paragraph{Lower bound.}
We further establish a lower bound:
for a carefully constructed data distribution and the random sign feature \citep{tropp2015introduction}, it holds with high probability that
\begin{equation} \label{eq:intro:lower}
\EB_{\x', \VM_s} \Big[  \Big(  \tilde{f}_{\lambda} (\x'; \XM_n, \VM_s) - f_{\lambda} (\x'; \XM_n) \Big)^2 \Big]
\: = \: \Omega \big( \tfrac{1}{s } \big) \cdot \big\| ( \K - \tfrac{1}{n} \K^2 )^{1/2} (\K + n \lambda \I_n)^{-1} \y \big\|_2^2 ,
\end{equation}
where the expectation is taken w.r.t.\ the random testing sample, $\x'$, and the random features, $\VM_s \triangleq \{ \v_1 , \cdots , \v_s \}$.
Note that the lower bound does not mean that RFM (in general) must be worse than the righthand-side of \eqref{eq:intro:lower}.
Instead, the lower bound indicates that without making additional assumptions on the data and feature map (for ruling out the adversarial example), the upper bound \eqref{eq:intro:upper} cannot be improved.

\paragraph{Contributions.}
Despite the existing results for RFM-KRR, this work offers unique contributions.
First and foremost, we establish an out-of-sample bound for RFM-KRR without making any uncheckable assumption.
Admittedly, without making those assumption, our bound cannot be as strong as \cite{rudi2017generalization}.
Second, our lower bound, which is based on a carefully constructed data example, shows that our upper bound is optimal, unless additional assumptions are made.
Third, our proof does not go beyond the scope of elementary matrix algebra~\citep{horn1991topics} and random matrix theories~\citep{tropp2015introduction}, and this work is therefore easy to follow and extend.

%Without making assumptions on the input space, targets, and kernel space, one can hardly bound the generalization error of RFM-KRR itself.
%For this reason, this work does not study the prediction error of RFM-KRR itself; instead, we study the gap between RFM-KRR and KRR.

\section{Notation and Preliminaries} \label{sec:notation}

Let $\I_n$ denote the $n\times n$ identity matrix.
Let $\A$ be any matrix, $\a_{i:}$ be the $i$-th row, $\a_{:j}$ be the $j$-th column, and $a_{ij}$ be the $(i,j)$-th entry.
Let $\| \A \|_2 $ be the spectral norm of $\A$ which is equal to the largest singular value of $\A$.

\subsection{Random feature mapping (RFM)}

Let $\XM_n = \{\x_1 , \cdots , \x_n \}$ contain the $n$ training samples
and $\VM_s = \{\v_1 , \cdots , \v_s \}$ be a set of mutually independent random vectors used for feature mapping.
Let $\psi (\x ; \v )$ be a random feature map (designed for the kernel function $\kappa$) and
\begin{equation} \label{eq:def:rfm_vec}
\ps (\x; \VM_s) 
\; = \; \tfrac{1}{ \sqrt{s} } \, 
\big[ \psi (\x_i ; \v_1 ) , \: \cdots , \: \psi (\x_i ; \v_s )  \big]
\; \in \; \RB^{s}
\end{equation}
be the resulting feature vector of $\x$.
A basic property of any valid random feature mapping $\psi$ is unbiasness, i.e.,
\begin{equation*}
\EB_{\v} [ \psi (\x ; \v) \psi (\x' ; \v) ] \: = \: \kappa (\x, \x') 
\end{equation*}
for any two vectors $\x$ and $\x'$.
A direct consequence is
\begin{equation*}
\EB_{\VM_s} \Big[ \ps (\x; \VM_s)^T \, \ps (\x' ; \VM_s)  \Big] 
\: = \: \frac{1}{s} \sum_{l=1}^s \EB_{\v_l} \Big[ \psi (\x ; \v_l) \, \psi (\x' ; \v_l ) \Big] 
\: = \: \kappa (\x, \x') .
\end{equation*}
Let $\ps_{i:} \in \RB^s$ be the abbreviation of $\ps (\x_i; \VM_s)$ and $\Ps \in \RB^{n\times s}$ be the stack of $\ps_{1:} , \cdots , \ps_{n:}$.
Let $\K = [ \kappa (\x_i , \x_j) ]_{ij}$ be the kernel matrix of the $n$ training samples.
Then another direct consequence of the unbiasness is
\begin{equation*}
\EB_{\VM_s} \big[ \Ps \Ps^T \big] 
\: = \: \K .
\end{equation*}

\subsection{Kernel Ridge Regression (KRR)} \label{sec:notation:krr}

Let $y_1 , \cdots , y_n \in \RB$ be the training targets associated with $\x_1 , \cdots , \x_n$.
Let $\ph (\x)$ be the representation of $\x$ in a high-dimensional feature space.
(It holds that $\kappa (\x, \x') = \ph (\x)^T \ph (\x')$ for all $\x, \x' \in \RB^d$.) 
Let $\Ph$ be the feature matrix whose $i$-th row is $\ph (\x_i)$.
In the training phase, the kernel ridge regression (KRR) solves the problem
\begin{small}
	\begin{eqnarray} \label{eq:def_krr_primal}
	\bet^\star
	\: = \:
	\argmin_{\mathbf{\beta}} \, 
	\bigg\{
	\frac{1}{n} \sum_{i=1}^n \big[  \ph (\x_i )^T \, \bet  -  y_i  \big]^2
	\, + \, \lambda \big\| \bet \big\|_2^2 \bigg\} ,
	\end{eqnarray}
\end{small}%
where $\lambda > 0$ is the regularization parameter and typically determined by cross-validation.
For a test sample $\x'$, KRR makes prediction by
\begin{align} \label{eq:f_krr}
& f_\lambda (\x' ; \XM_n)
\: = \:  \ph (\x' )^T  \, \bet^\star \nonumber \\
&  = \: \ph (\x' )^T \big( \Ph^T \Ph + n \lambda \I \big)^{-1} \Ph^T \y \nonumber \\
& = \: \k'^T (\K + n \lambda \I_n )^{-1} \y .
\end{align}
The $i$-th entry of $\k' \in \RB^n$ is $k'_i = \kappa (\x', \x_i) = \ph (\x' )^T \ph (\x_i)$.
With the kernel representation, 
KRR is tractable even if the feature map $\ph (\cdot)$ is infinite-dimensional.

RFM-KRR uses random features $\ps (\x_i ; \VM_s)$ intead of $\ph (\x_i)$, and its training phase solves
\begin{small}
	\begin{eqnarray}\label{eq:def_rfmkrr_primal}
	\w^\star 
	\: = \:
	\argmin_{\mathbf{w}} \, 
	\bigg\{ 
	\frac{1}{n} \sum_{i=1}^n \big( y_i -  \ps_i^T \w \big)^2
	\, + \, \lambda \big\| \w \big\|_2^2 \bigg\} .
	\end{eqnarray}
\end{small}%
Let $\x'$ be a test sample.
With $\VM_s = \{ \v_1 , \cdots , \v_s \}$ at hand, RFM-KRR forms the feature vector $\ps (\x'; \VM_s) \in \RB^s$
and then makes prediction by
\begin{align} \label{eq:f_rfm}
&\tilde{f}_\lambda (\x' ; \XM_n, \VM_s)
\: = \: \ps (\x'; \VM_s)^T \w^\star \nonumber \\
& = \: \ps (\x'; \VM_s)^T \big( \Ps^T \Ps + n \lambda \I_s \big)^{-1} \Ps^T \y \nonumber \\
& = \: \tilde{\k}'^T \big( \tilde{\K} + n \lambda \I_n \big)^{-1} \y .
\end{align}
Here, $\tilde{\K} \triangleq \Ps \Ps^T \in \RB^{n\times n}$ is the approximate kernel matrix,
and $\tilde{\k}' \triangleq \Ps \, \ps (\x'; \VM_s) \in \RB^n$ is the approximate kernel vector of the test sample $\x'$.

KRR has been studied by early work such as \cite{scholkopf2002learning}.
Learning theories of KRR has been established by the prior works \citep{zhang2005learning,caponnetto2007optimal,steinwart2009optimal,zhang2014divide}

\section{Main results} \label{sec:main}

This paper studies the generalization of RFM-KRR and compare it with KRR.
For a test sample $\x'$, let $f_{\lambda} (\x'; \XM_n)$ and $\tilde{f}_{\lambda} (\x'; \XM_n,\VM_s)$ be the prediction made by KRR and RFM-KRR, respectively.
Theorems~\ref{thm:main} and \ref{thm:lower1} study the mean squared difference between the two predictions and establish a near-optimal bound.

\subsection{Upper bound}

The key point of this work is almost assumption-free.
We make only a reasonable and checkable assumption which is satisfied by many popular types of RFMs.
For example, random Fourier feature for translation invariant kernel satisfies the assumption with $b = 2$ \citep{rahimi2007random};
random sign feature for the angular similarity kernel satisfies the assumption with $b=1$ \citep{tropp2015introduction}.

\begin{assumption} \label{assumption:b}
	The random feature mapping is an unbiased estimate of the kernel function:
	\begin{equation*}
	\kappa (\x, \x') \: = \: \int \psi (\x, \v) \psi (\x', \v) p (\v) \, d (\v) ,
	\end{equation*}
	where $p (\cdot )$ is a PDF.
	For some constant $b > 0$, $ \psi^2 (\cdot ; \cdot) \leq b$ holds almost surely.
\end{assumption}

Under Assumption~\ref{assumption:b}, we have the following main theorem showing that the prediction made by RFM-KRR converges to KRR at a rate of $\frac{1}{s}$. (Here, $s$ is the number of random features.)

\begin{theorem} \label{thm:main}
	Let Assumption~\ref{assumption:b} hold.
	Let $\delta  \in (0, 1)$ be any user-specified constants.
	The regularization parameter, $\lambda$, is sufficiently large:
	\begin{equation*}
	\lambda \: \geq \: \tfrac{2 b }{\sqrt{n}  }  \, \log^{\frac{1}{2}} \big( \tfrac{s }{ \delta } \big) .
	\end{equation*}
	If the training and testing samples are from the same (unknown) distribution, then it holds with probability at least $1 - \delta $ that
	\begin{align*}
	&\EB_{\x', \VM} \Big[
	\big( \tilde{f}_{\lambda} (\x' ; \XM_n , \VM_s) - {f}_{\lambda} (\x' ; \XM_n ) \big)^2 \Big] 
	\: \leq \:  \frac{ 4 b  }{s}  \,  
	\big\| \K^{\frac{1}{2}} ( {\K} + n \lambda \I_n )^{-1} \y \big\|_2^2  .
	\end{align*}
	The failure probability arises from the randomness in the training data.
\end{theorem}

\begin{remark}
	The righthand side of the bound, $\big\| \K^{\frac{1}{2}} ( {\K} + n \lambda \I_n )^{-1} \y \big\|_2^2 $, is independent of the RFM;
	it depends only on the training data and kernel.
	It satisfies 
	\begin{equation*}
	\| \K^{\frac{1}{2}} ( {\K} + n \lambda \I_n )^{-1} \y \|_2^2 
	\: \leq \: \frac{1}{4 n \lambda} \|\y \|_2^2 
	\: = \: \OM (\sqrt{n}) .
	\end{equation*}
	The equality can be reached if there is such an index $t \in [n]$ that $\sigma_t = n \lambda$ and $\u_{:t}^T \y = \|\y\|_2$; here, $\sigma_t$ and $\u_{:t}$ are the $t$-th singular value and singular vector of $\K$.
\end{remark}

\begin{remark}
Theorem~\ref{thm:main} matches the lower bound in Theorem~\ref{thm:lower1} and is thus optimal.
Without making additional assumptions, Theorem~\ref{thm:main} cannot be improved.
Our bound (Theorem~\ref{thm:main}) is weaker than \cite{rudi2017generalization} which is however based on many assumptions.
\end{remark}

\subsection{Lower bounds}

We use the angular similarity kernel $\kappa (\x, \x') = \frac{1}{2\pi} \arcsin \frac{ \x^T \x' }{ \|\x\|_2 \|\x'\|_2 }$ to establish a lower bound that matches the upper bound. 
\cite{tropp2015introduction} showed that the random sign feature $\psi (\x; \v) = \sgn (\x^T \v)$, with $\v$ drawn uniformly from the unit sphere, enjoys our Assumption~\ref{assumption:b} with $b=1$.
The lower bound indicates Theorem~\ref{thm:main} is near optimal.

\begin{theorem} 
	\label{thm:lower1}
	Let $\XM_n = \{ \x_1 , \cdots , \x_n \}$ be a set of training samples uniformly from the unit sphere. 
	Assume $s \lambda = \Omega (1)$.\footnote{If $\lambda = \Omega (1/\sqrt{n})$ and $s = \Omega (\sqrt{n})$, then $s \lambda = \Omega (1)$.}
	If the test sample $\x'$ is drawn uniformly from $\XM_n$, then it holds with high probability that
	\begin{align}\label{eq:thm:lower1}
	&\EB_{\x', \VM_s} \Big[
	\big( \tilde{f}_{\lambda} (\x' ; \XM_n , \VM_s) - {f}_{\lambda} (\x' ; \XM_n ) \big)^2 \Big] 
	\: = \: \Omega \Big(  \frac{ 1  }{s}   \Big) \cdot
	\big\| ( \K - \tfrac{1}{n} \K^2 )^{\frac{1}{2}} ( {\K} + n \lambda \I_n )^{-1} \y \big\|_2^2 ,
	\end{align}
	where the failure probability arises from the random training samples.
\end{theorem}

\begin{remark}
The theorem does not imply that the gap between RFM-KRR and KRR is always bigger than the righthand-side of \eqref{eq:thm:lower1}.
In practice, the error is oftentimes smaller than the righthand-side because real-world data are different from the carefully designed data in the theorem.
The only purpose of the theorem is to show without making additional assumptions, Theorem~\ref{thm:main} cannot be improved.
\end{remark}

\section{Analysis of Upper Bound} \label{sec:analysis}

In this section we prove that the prediction made by RFM-KRR, $\tilde{f}_{\lambda} (\x'; \XM_n, \VM_s)$, converges to the prediction by KRR, $f_{\lambda} (\x', \XM_n)$.
The theories developed in this section may have independent interest.
Section~\ref{sec:analysis:rand} discusses the three sources of randomness in the analysis.
Section~\ref{sec:analysis:training} analyzes the randomness in the training samples and establishes Lemma~\ref{lem:Xii}.
Section~\ref{sec:analysis:rfm} analyzes the randomness in the feature map and establishes Lemmas~\ref{lem:rfm1} and \ref{lem:rfm2}.
Finally, Section~\ref{sec:analysis:thm} uses the lemmas to prove our main theorem.

\subsection{Three sources of randomness} \label{sec:analysis:rand}

Let $\XM_n = \{\x_1, \cdots , \x_n \}$ be the set of training samples and $\x'$ be a test sample.
Let $f_{\lambda} (\x' ; \XM_n)$ and $\tilde{f}_{\lambda} (\x' ; \XM_n , \VM_s)$ be the predictions made by KRR and RFM-KRR, respectively.
We study the worst-case bound on the gap between the two predictions:
\begin{equation} \label{eq:def_error}
\EB_{\x',\VM_s} \Big[ \big( f_{\lambda} (\x'; \XM_n) - \tilde{f}_{\lambda} (\x' ; \XM_n, \VM_s) \big)^2 \Big] ,
\end{equation}
a vanishing value of which means RFM-KRR makes almost the same prediction as KRR on the unseen test samples.
There are three sources of randomness:
\begin{itemize}
	\item 
	First, the elements in $\VM_s = \{ \v_1 , \cdots , \v_s \}$ are randomly drawn from a designed distribution, and $\tilde{f}_{\lambda}$ depends on $\VM_s$.
	\item
	Second, the training samples are randomly drawn from an unknown distribution, and $f_{\lambda}$ and $\tilde{f}_{\lambda}$ both depend on $\XM_n$.
	\item
	Third, the test sample, $\x'$, is randomly drawn from the same distribution as the training samples, and $f_{\lambda}$ and $\tilde{f}_{\lambda}$ both depend on $\x'$.
\end{itemize}
The expectation in \eqref{eq:def_error} integrates out the randomness in $\x'$ and $\VM_s$.
In the following, we analyze the randomness from RFM and the training samples.

\subsection{Analyzing random training samples} \label{sec:analysis:training}

Recall that $\psi (\x ; \v )$ is a feature map. 
Let $\ps (\x; \VM_s) \in \RB^s$ be the random feature vector defined in \eqref{eq:def:rfm_vec}.
Define the second moment in the feature space as
\begin{equation*}
\Xii \: = \: \int \ps (\x; \VM_s) \, \ps (\x; \VM_s)^T  \, \rho (\x) \,  d \x 
\; \in \; \RB^{s\times s} .
\end{equation*}
where $\rho$ is the PDF of the distribution of $\x$.
Let $\ps_{i:} \in \RB^s$ be the abbrevation of $\ps (\x_{i}; \VM_s)$ and $\Ps \in \RB^{n\times s}$ be the stack of $\ps_{1:} , \cdots , \ps_{n:}$.

We show in Lemma~\ref{lem:Xii} that the empirical second moment in the feature space, $\tfrac{1}{n} \Ps^T \Ps $, converges to $\Xii$ as the number of samples, $n$, grows.
We merely analyze the randomness in the training set $\XM_n = \{ \x_1, \cdots , \x_n \}$; we suppose the feature mapping is given.

\begin{lemma} \label{lem:Xii}
	Let Assumption~\ref{assumption:b} hold and $b$ be defined therein.
	The training samples $\x_1, \cdots , \x_n \in \RB^d$ are i.i.d.\ from some distribution.
	Let $\epsilon > 0$ and $\delta \in (0, 1)$ be arbitrary.
	For $n \geq \frac{8 b^2 }{3 \epsilon^2  }  \, \log \frac{ s }{ \delta } $,
	it holds with probability at $1-\delta$ that
	\begin{equation*}
	\big\| \Xii - \tfrac{1}{n} \Ps^T \Ps \big\|_2 \, \leq \, \epsilon .
	\end{equation*}
	Here, the uncertainty is from the randomness in the $n$ training samples.
\end{lemma}

\begin{proof}
	Let $\ps_{i:} \in \RB^{s}$ be the $i$-th row of the random matrix $\Ps \in \RB^{n\times s}$.
	It follows from Assumption~\ref{assumption:b} that
	\begin{equation} \label{eq:lem:Xii:1}
	\big\| \ps_{i:} \big\|_2^2
	\: = \: \sum_{j=1}^s \big[ \tfrac{1}{\sqrt{s}}  \psi (\x_i, \v_j ) \big]^2
	\: \leq \: b .
	\end{equation}
	Let $\rho$ be the PDF of the unknown distribution of $\x$.
	Thus $\EB_{\x_i} [\ps_{i:} \ps_{i:}^T] = \int \ps_{i:} \ps_{i:}^T \rho (\x_i) d {\x_i} = \Xii$.
	We define the zero-mean random matrix
	\begin{equation*}
	\Z_{i} \, = \, \tfrac{1}{n} (\ps_{i:} \ps_{i:}^T - \Xii) .
	\end{equation*}
	Here we study the the randomness from $\x_i \sim \rho (\cdot)$;
	we suppose $\v_1 , \cdots , \v_s$ are observed and do not consider their randomness.
	Since $- \tfrac{1}{n} \Xii \preceq \Z_i \preceq \tfrac{1}{n} \ps_{i:} \ps_{i:}^T$, 
	The spectral norm of $n\Z_{i}$ is bounded by 
	\begin{align*}
	&\big\| n \Z_i \big\|_2
	\: \leq \: \max \Big\{ \big\| \ps_{i:} \ps_{i:}^T \big\|_2 , \, \big\| \Xii \big\|_2 \Big\}
	\: = \: \max \Big\{ \big\| \ps_{i:} \ps_{i:}^T \big\|_2 , \, \big\| \EB_{\x_i} [\ps_{i:} \ps_{i:}^T] \big\|_2 \Big\} \\
	& \leq \: \max \Big\{ \big\| \ps_{i:} \ps_{i:}^T \big\|_2 , \, \EB_{\x_i}  \big\| \ps_{i:} \ps_{i:}^T \big\|_2 \Big\}
	\: \leq \: \sup_{\x_i} \big\| \ps_{i:} \ps_{i:}^T \big\|_2
	\: = \: \sup_{\x_i} \big\| \ps_{i:} \big\|_2^2
	\: \leq \: b ,
	\end{align*}
	where the last inequality follows from \eqref{eq:lem:Xii:1}.
	Thus 
	\begin{align*}
	& \| \Z_i \|_2
	\, \leq \,  \tfrac{b }{n}
	\, \triangleq \, L .
	\end{align*}
	We then bound the variance of $\sum_{i=1}^n \Z_i^2$. 
	It follows from \eqref{eq:lem:Xii:1} that $\ps_{i:}^T \ps_{i:} \leq b$, and thus
	\begin{align*}
	& \EB_{\x_i} \big[ (n \, \Z_i )^2 \big]
	\, = \, \EB_{\x_i} \big[ \ps_{i:} \ps_{i:}^T \ps_{i:} \ps_{i:}^T + \Xii^2 
	- \ps_{i:} \ps_{i:}^T \Xii - \Xii \ps_{i:} \ps_{i:}^T \big] \\
	& = \, \EB_{\x_i} [ \ps_{i:} (\ps_{i:}^T \ps_{i:} ) \ps_{i:}^T ] + \Xii^2 
	- \EB_{\x_i} [ \ps_{i:} \ps_{i:}^T ] \, \Xii -  \Xii \, \EB_{\x_i} [\ps_{i:} \ps_{i:}^T ] \\
	& \preceq \, b \, \EB_{\x_i} [ \ps_{i:}  \ps_{i:}^T ] + \Xii^2 - \Xii^2 - \Xii^2 \\
	& \preceq \, b  \, \Xii .
	\end{align*}
	Thus $\EB_{\XM_n} \big[ \sum_{i=1}^n \Z_i^2  \big]
	= n \, \EB_{\x_i} [ \Z_i^2 ] \preceq \frac{ b  }{ n } \Xii
	\triangleq \V $.
	\begin{equation*}
	v \, \triangleq \, \| \V \|_2 
	\, = \, \tfrac{ b  }{ n } \, \| \Xii \|_2 
	\, \leq \, \tfrac{ b^2  }{ n } .
	\end{equation*}
	Finally, applying the matrix Bernstein~\citep{tropp2015introduction}, 
	we obtain that for any $t \geq 0$, 
	\begin{equation*}
	\PB \Big\{ \Big\| \sum_{i=1}^n \Z_i \Big\|_2 \, \geq \, t \Big\}
	\; \leq \; s \cdot \exp \Big( \frac{ - t^2 / 2 }{ v + L t / 3 } \Big) .
	\end{equation*}
	It follows from the definition of $\Z_i$ that
	\begin{equation*}
	\PB \Big\{ \big\| \Xii - \tfrac{1}{n} \Ps^T \Ps \big\|_2 \, \geq \, t \Big\}
	\; \leq \; s \cdot \exp \bigg( \frac{ - t^2 / 2 }{ v + L t / 3 } \bigg)
	\; \triangleq \; \delta .
	\end{equation*}
	Hence, for $n \geq \frac{8 b^2 }{3 \epsilon^2  }  \, \log \frac{s }{ \delta } $,
	it holds with probability at $1-\delta$ that
	\begin{equation*}
	\big\| \Xii - \tfrac{1}{n} \Ps^T \Ps \big\|_2 \, \leq \, \epsilon ,
	\end{equation*}
	by which the lemma follows.
\end{proof}

\subsection{Analyzing random feature mapping} \label{sec:analysis:rfm}

Lemma~\ref{lem:rfm1} establishes an upper bound for the symmetric positive semi-definite (SPSD) matrix
$\EB_{\VM_s} \big[ \Ps \Ps^T \Ps \Ps^T \big]$.
The randomness is from the random feature mapping.

\begin{lemma} \label{lem:rfm1}
	Let $\VM_s = \{\v_1, \cdots , \v_s \} $ be the set of random vectors for feature mapping.
	Let Assumption~\ref{assumption:b} hold and $b$ defined therein.
	Then
	\begin{align*} 
	\EB_{\VM_s} \big[ \Ps \Ps^T \Ps \Ps^T \big]
	\: \preceq \:  \tfrac{s-1}{s}  \K^2  
	+ \tfrac{nb}{s}  \K .
	\end{align*}
\end{lemma}

\begin{proof}
	Recall that $\ps_{i:} = \tfrac{1}{\sqrt{s} } [ \psi (\x_i ; \v_1) , \cdots , \psi (\x_i ; \v_s) ] \in \RB^s $ 
	is the $i$-th row of $\Ps \in \RB^{n\times s}$.
	It holds that
	\begin{small}
		\begin{align} \label{eq:lem:rfm1:2}
		& \EB_{\VM_s} \big[ \Ps \Ps^T \Ps \Ps^T \big]_{ij} 
		\: = \: \EB_{\VM_s} \big[  \big( \Ps \ps_{i:} \big)^T \big(\Ps \ps_{j:} \big) \big] \nonumber \\
		& = \: \EB_{\VM_s} \Big\{ \big[ \ps_{1:}^T \ps_{i:} , \cdots , \ps_{n:}^T \ps_{i:}  \big]^T 
		\big[ \ps_{1:}^T \ps_{j:} , \cdots , \ps_{n:}^T \ps_{j:}  \big] \Big\} \nonumber \\
		& = \: \sum_{l=1}^n  \EB_{\VM_s} \Big[  \big( \ps_{l:}^T \ps_{i:} \big)  \big( \ps_{l:}^T \ps_{j:} \big) \Big] \nonumber \\
		& = \: 
		\sum_{l=1}^n \EB_{\VM_s}  \bigg\{  \bigg[ \frac{1}{s}  \sum_{p=1}^s \psi (\x_l ; \v_p) \psi (\x_i; \v_p)  \bigg] \, \bigg[  \frac{1}{s} \sum_{q=1}^s \psi (\x_l; \v_q) \psi (\x_j; \v_q)  \bigg] \bigg\}  \nonumber \\
		& = \:  \frac{1}{s^2} \sum_{l=1}^n \EB_{\VM_s} \bigg\{ 
		\sum_{p\neq q}  \Big[ \psi (\x_l; \v_p) \psi (\x_i;\v_p) \psi (\x_l; \v_q) \psi (\x_j; \v_q) \Big] \, + \, \sum_{p=1}^s  \Big[  \psi^2 (\x_l; \v_p) \psi (\x_i; \v_p) \psi (\x_j; \v_p) \Big] \bigg\} .
		\end{align}
	\end{small}%
	Since $\v_p$ and $\v_q$ are independent,
	the former term in \eqref{eq:lem:rfm1:2} can be bounded by
	\begin{align} \label{eq:lem:rfm1:3}
	&  \sum_{l=1}^n \EB_{\VM_s} 
	\sum_{p\neq q}  \Big[ \psi (\x_l; \v_p) \psi (\x_i; \v_p) \psi (\x_l; \v_q) \psi (\x_j; \v_q) \Big] \nonumber \\
	& = \:  \sum_{l=1}^n \sum_{p\neq q} \Big\{ \EB_{\v_p}  \big[ \psi (\x_l; \v_p) \psi (\x_i; \v_p)  \big]\, \EB_{\v_q}  \big[ \psi (\x_l; \v_p) \psi (\x_j; \v_p)  \big] \Big\} \nonumber \\
	& = \:  \sum_{l=1}^n \sum_{p\neq q} \kappa ( \x_l; \x_i ) \kappa ( \x_l; \x_j ) \nonumber \\
	& = \:  \sum_{l=1}^n (s^2 - s) \kappa ( \x_l; \x_i ) \kappa ( \x_l; \x_j )  \nonumber \\
	& = \: (s^2 - s) [ \K^2 ]_{ij} .
	\end{align}
	Le $\ps_{:p} = \frac{1}{ \sqrt{s} } [ \psi (\x_1 ; \v_p) , \cdots , \psi (\x_n ; \v_p) ] \in \RB^n$ 
	be the $p$-th column of $\Ps \in \RB^{n\times s}$.
	The latter term in \eqref{eq:lem:rfm1:2} can be bounded by
	\begin{align} \label{eq:lem:rfm1:4}
	& \sum_{l=1}^n \EB_{\VM_s} \sum_{p=1}^s  \Big[  \psi^2 (\x_l; \v_p) \, \psi (\x_i; \v_p) \, \psi (\x_j; \v_p) \Big] \nonumber \\
	& = \:\sum_{p=1}^s   \EB_{\v_p}  \Bigg\{
	\bigg[ \sum_{l=1}^n   \psi^2 (\x_l; \v_p)\bigg] 
	\psi (\x_i; \v_p) \psi (\x_j; \v_p) 
	\Bigg\}  \nonumber \\
	& = \: s \cdot   \EB_{\v_p}  \Bigg\{
	\bigg[ \sum_{l=1}^n   \psi^2 (\x_l; \v_p)\bigg] 
	\cdot s \, \big[\ps_{:p} \ps_{:p}^T \big]_{ij}
	\Bigg\}  
	\end{align}
	It follows from \eqref{eq:lem:rfm1:2}, \eqref{eq:lem:rfm1:3}, and \eqref{eq:lem:rfm1:4} that
	\begin{align*}
	& \EB_{\VM_s} \big[ \Ps \Ps^T \Ps \Ps^T \big]_{ij}  
	\: = \: \frac{s-1}{s} \big[ \K^2 \big]_{ij}  
	+ \EB_{\v_p}  \Bigg\{
	\bigg[ \sum_{l=1}^n   \psi^2 (\x_l; \v_p)\bigg] 
	\, \big[\ps_{:p} \ps_{:p}^T \big]_{ij}
	\Bigg\}  ,
	\end{align*}
	and thus
	\begin{align*}
	& \EB_{\VM_s} \big[ \Ps \Ps^T \Ps \Ps^T \big] 
	\: = \: \frac{s-1}{s}  \K^2  
	+ \EB_{\v_p}  \Bigg\{
	\bigg[ \sum_{l=1}^n   \psi^2 (\x_l; \v_p)\bigg] 
	\, \Big[\ps_{:p} \ps_{:p}^T \Big]
	\Bigg\} .
	\end{align*}
	Since $0\leq \psi^2  (\cdot , \cdot ) \leq b$ (by Assumption~\ref{assumption:b}) and $\ps_{:p} \ps_{:p}^T$ is SPSD,
	we have that
	\begin{align*}
	& \EB_{\v_p}  \Bigg\{
	\bigg[ \sum_{l=1}^n   \psi^2 (\x_l; \v_p)\bigg] 
	\, \Big[\ps_{:p} \ps_{:p}^T \Big]
	\Bigg\} 
	\: \preceq \:\EB_{\v_p}  \Big[  n b \cdot \ps_{:p} \ps_{:p}^T \Big]
	\end{align*}
	It follows that
	\begin{align*}
	& \EB_{\VM_s} \big[ \Ps \Ps^T \Ps \Ps^T \big] 
	\: \preceq \: \tfrac{s-1}{s}  \K^2  
	+ nb \cdot \EB_{\v_p}  \big[\ps_{:p} \ps_{:p}^T \big]  
	\: = \: \tfrac{s-1}{s}  \K^2  
	+ \tfrac{nb}{s}  \K .
	\end{align*}
	Here the identity follows from 
	\begin{align} \label{eq:lem:rfm1:5}
	&\EB_{\v_p} \big[\ps_{:p} \ps_{:p}^T \big]_{ij} 
	\: = \: \EB_{\v_p} \big[ \tfrac{1}{s} \psi (\x_i; \v_p) \psi (\x_j; \v_p)  \big]
	\: = \: \tfrac{1}{s} \kappa (\x_i, \x_j) .
	\end{align}
\end{proof}

Let us recall the following notation.
Let $\x'$ be a test sample. (Here we do not need its randomness.)
The feature vector of $\x_i$ is $\ps_{i:} = \tfrac{1}{\sqrt{s} } [ \psi (\x_i ; \v_1) , \cdots , \psi (\x_i ; \v_s) ] \in \RB^s $;
the feature vector of $\x'$ is $\ps' = \tfrac{1}{\sqrt{s} } [ \psi (\x' ; \v_1) , \cdots , \psi (\x' ; \v_s) ] \in \RB^s $.
The $i$-th entries of $\k' \in \RB^n$ and $\tilde{\k}'$ are respectively $\kappa (\x_i, \x') $ and $\ps_{i:}^T \ps'$.

\begin{lemma} \label{lem:rfm2}
	Let $\VM_s = \{\v_1, \cdots , \v_s \} $ be the set of random vectors for feature mapping.
	Let Assumption~\ref{assumption:b} hold and $b$ be defined therein.
	Then
	\begin{align*}
	&  \EB_{\VM_s} \big[  ( \tilde{\k}' - \k') ( \tilde{\k}' - \k')^T \big]  
	\: \preceq \: \tfrac{b}{s} \K - \tfrac{1}{s}  \k' \k'^T  .
	\end{align*}
\end{lemma}

\begin{proof}
	The unbiasness property in Assumption~\ref{assumption:b} ensures that
	\begin{equation*}
	\EB_{\VM_s} \big[ \ps_{i:}^T \ps' \big] 
	\: = \: \frac{1}{s} \sum_{p=1}^s \EB_{\v_p} \Big[\psi (\x_i; \v_p) \, \psi (\x'; \v_p) \Big] 
	\: = \: \kappa (\x_i , \x') ,
	\end{equation*}
	and thus $\EB_{\VM_s} [ \tilde{\k}' ] = \k'$. 
	It follows that
	\begin{align} \label{eq:lem:rfm2:1}
	&  \EB_{\VM_s} \big[  ( \tilde{\k}' - \k') ( \tilde{\k}' - \k')^T \big] \nonumber \\
	& = \: \EB_{\VM_s} \big[ \tilde{\k}' \tilde{\k}'^T \big] 
	+ \k' \k'^T - \EB_{\VM_s} \big[ \tilde{\k}' \big]  \k'^T
	- \k' \EB_{\VM_s} \big[ \tilde{\k}'^T \big]  \nonumber \\
	& = \:  \EB_{\VM_s} \big[ \tilde{\k}' \tilde{\k}'^T  \big] - \k' \k'^T    .
	\end{align}
	The $(i,j)$-th entry of the former term in \eqref{eq:lem:rfm2:1} is
	\begin{small}
		\begin{align} 
		& \EB_{ \VM_s} \big[ \tilde{\k}' \tilde{\k}'^T \big]_{ij}
		\: = \: \EB_{\VM_s} \big[ \big(\ps_i^T \ps' \big) \big(\ps_j^T \ps' \big) \big] \nonumber \\
		& = \: \EB_{\VM_s} \Bigg\{
		\bigg[ \frac{1}{s} \sum_{p=1}^s \psi (\x_i , \v_p) \, \psi (\x', \v_p) \bigg]
		\bigg[ \frac{1}{s} \sum_{q=1}^s \psi (\x_j , \v_q) \, \psi (\x', \v_q) \bigg] \Bigg\} \nonumber \\
		& = \: \frac{1}{s^2} \EB_{\VM_s} \Bigg\{
		\sum_{p\neq q}  \Big[ \psi (\x_i , \v_p) \, \psi (\x', \v_p) \, \psi (\x_j , \v_q) \, \psi (\x', \v_q) \Big] 
		+ \sum_{p=1}^s  \Big[ \psi (\x_i , \v_p) \, \psi (\x_j , \v_p) \, \psi^2 (\x', \v_p) \Big] \Bigg\} \nonumber \\
		& = \: \frac{1}{s^2}\sum_{p\neq q}  
		\EB_{\v_p} \big[ \psi (\x_i , \v_p) \, \psi (\x', \v_p) \big] \cdot
		\EB_{\v_q} \big[ \psi (\x_j , \v_q) \, \psi (\x', \v_q) \big] \nonumber \\
		& \qquad + \frac{1}{s^2} \sum_{p=1}^s \EB_{\v_p} \Big[ \psi (\x_i , \v_p) \, \psi (\x_j , \v_p) \, \psi^2 (\x', \v_p) \Big] \nonumber \\
		& = \: \tfrac{ s^2 - s}{s^2 } \kappa (\x_i, \x') \, \kappa (\x_j , \x')
		+ \tfrac{1}{s}\EB_{\v_p} \Big[ \psi (\x_i , \v_p) \, \psi (\x_j , \v_p) \, \psi^2 (\x', \v_p) \Big]. \nonumber
		\end{align}
	\end{small}%
	Recall that $\ps_{:p} = \tfrac{1}{s} [ \psi (\x_1 , \v_p) , \cdots , \psi (\x_n , \v_p) ] \in \RB^n $ 
	is the $p$-th column of $\Ps \in \RB^{n\times s}$.
	It follows that
	\begin{align}\label{eq:lem:rfm2:2}
	& \EB_{ \VM_s} \Big[ \tilde{\k}' \tilde{\k}'^T \Big] 
	\: = \: \tfrac{ s^2 - s}{s^2 } \k' \k'^T
	+ \tfrac{1}{s} \EB_{\v_p} \Big[ s \cdot \ps_{:p} \ps_{:p}^T  \cdot \psi^2 (\x', \v_p) \Big] \nonumber \\
	& \preceq \: \tfrac{ s - 1}{s } \k' \k'^T 
	+ \tfrac{b}{s} \cdot \EB_{\v_p} \Big[ s \cdot \ps_{:p} \ps_{:p}^T \Big] 
	\: = \: \tfrac{ s - 1}{s } \k' \k'^T + \tfrac{b}{s} \K ,
	\end{align}
	where the inequality follows from that $0 \leq \psi^2 (\cdot ; \cdot) \leq b$ (by Assumption~\ref{assumption:b}),
	and the last identity follows from \eqref{eq:lem:rfm1:5}.
	It follows from \eqref{eq:lem:rfm2:1} and \eqref{eq:lem:rfm2:2} that
	\begin{align*}
	&  \EB_{ \VM_s} \big[  ( \tilde{\k}' - \k') ( \tilde{\k}' - \k')^T \big]  
	\: = \: \EB_{\VM_s} \big[ \tilde{\k}' \tilde{\k}'^T  \big] - \k' \k'^T   
	\: \preceq \:   \tfrac{ s - 1}{s } \k' \k'^T + \tfrac{b}{s} \K - \k' \k'^T   
	\: = \: \tfrac{b}{s} \K - \tfrac{1}{s} \k' \k'^T  ,
	\end{align*}
	by which the lemma follows.
\end{proof}

\subsection{Completing the proof of Theorem~\ref{thm:main}} \label{sec:analysis:thm}

Now we complete the proof of Theorem~\ref{thm:main} using the lemmas in this section.
Recall that $f$ and $\tilde{f}_s$ are the predictions made by KRR and RFM-KRR, as defined in \eqref{eq:f_krr} and \eqref{eq:f_rfm}, respectively.
Recall the notation that $\ps (\x ; \VM_s) = \frac{1}{ \sqrt{s} } [ \psi (\x; \v_1) , \cdots , \psi (\x; \v_s) ] \in \RB^s$ and $\Ps \in \RB^{n\times s}$ is the stack of $\ps (\x_1 ; \VM_s) , \cdots , \ps (\x_n ; \VM_s)$.
Since the training samples, $\x_1 , \cdots , \x_n$, and the test sample, $\x'$, are randomly drawn according to the PDF $\rho (\cdot )$,
Lemma~\ref{lem:Xii} ensures that for 
$n \geq \frac{8 b^2 }{3 \lambda^2   }  \, \log \frac{s }{ \delta } $,
it holds with probability at least $1-\delta$ that
\begin{equation}\label{eq:condition1} 
\EB_{\x' \sim \rho} \big[ \ps (\x' ; \VM_s) \, \ps (\x' ; \VM_s)^T \big]  
\: \preceq \: 
\tfrac{1}{n} \Ps^T \Ps + \lambda \I_s ,
\end{equation}
where the failure probability is from the randomness in $\XM_n = \{ \x_1 , \cdots , \x_n \}$.
The rest of the proof is conditioned on the event \eqref{eq:condition1} and does not use the randomness in $\XM_n$.

It follows from the definition of $f$ and $\tilde{f}_s$ in \eqref{eq:f_krr} and \eqref{eq:f_rfm} that
\begin{align*}
& \big[ \tilde{f}_{\lambda} (\x'; \XM_n, \VM_s) - {f}_{\lambda} (\x' ; \XM_n) \big]^2  \nonumber \\
& = \, \Big[ \tilde{\k}'^T ( \tilde{\K} + n \lambda \I_n )^{-1} \y
- {\k}'^T ( {\K} + n \lambda \I_n )^{-1} \y \Big]^2  \nonumber \\
& = \,  \Big[ \tilde{\k}'^T ( \tilde{\K} + n \lambda \I_n )^{-1} \y
- \tilde{\k}'^T ( {\K} + n \lambda \I_n )^{-1} \y 
+ \tilde{\k}'^T ( {\K} + n \lambda \I_n )^{-1} \y
- {\k}'^T ( {\K} + n \lambda \I_n )^{-1} \y \Big]^2   \nonumber \\ 
& \leq \, 2  \Big[ \tilde{\k}'^T \Big( \big( \tilde{\K} + n \lambda \I_n \big)^{-1} - \big( {\K} + n \lambda \I_n \big)^{-1} \Big) \y \Big]^2  
+ 2 \Big[ \big( \tilde{\k}'^T - {\k}'^T \big) \big( {\K} + n \lambda \I_n \big)^{-1} \y \Big]^2   \nonumber \\ 
& = \, 2  \Big[ \tilde{\k}'^T  ( \tilde{\K} + n \lambda \I_n )^{-1} 
( \tilde{\K} - \K )
( {\K} + n \lambda \I_n )^{-1}  \y \Big]^2 
+ 2 \Big[ \big( \tilde{\k}'^T - {\k}'^T \big) \big( {\K} + n \lambda \I_n \big)^{-1} \y \Big]^2  ,
\end{align*}
where the last identity follows from that $\A^{-1} - \B^{-1} = \A^{-1} (\B - \A) \B$.
We define the notation:
\begin{align*} 
&\De_1
\: = \:
( \tilde{\K} + n \lambda \I_n )^{-1} \tilde{\k}' \tilde{\k}'^T ( \tilde{\K} + n \lambda \I_n )^{-1} , \\
&\Delta_2
\: = \:\big\| ( \tilde{\K} - \K ) ( {\K} + n \lambda \I_n )^{-1}  \y \big\|_2^2 ,\\
&\Delta_3 
\: = \: \big[ (\tilde{\k}' - {\k}')^T ( {\K} + n \lambda \I_n )^{-1} \y \big]^2. 
\end{align*}
It follows that
\begin{align} \label{eq:thm:full:1}
\big[ \tilde{f}_{\lambda} (\x'; \XM_n, \VM_s) - {f}_{\lambda} (\x' ; \XM_n) \big]^2
\: \leq \: 2 \| \De_1 \|_2 \Delta_2 + 2 \Delta_3 ,
\end{align}
We bound the three terms in the following.

{\bf Analysis of $\De_1$.}
Recall the definition $\tilde{\k}' = \Ps \, \ps (\x'; \VM_s) \in \RB^n$.
It follows that
\begin{align*}
\EB_{\x'} \big[ \tilde{\k}' \tilde{\k}'^T  \big]
\: = \: \EB_{\x'} \big[ \Ps \, \ps (\x'; \VM_s) \, \ps (\x'; \VM_s)^T \Ps^T \big]
\: = \: \Ps \, \EB_{\x'} \big[  \ps (\x'; \VM_s) \, \ps (\x'; \VM_s)^T \big] \Ps^T .
\end{align*}
If the event \eqref{eq:condition1} happens, then
\begin{align*}
\EB_{\x'} \big[ \tilde{\k}' \tilde{\k}'^T  \big]
\: \preceq \: \tfrac{1}{n}  \Ps \big(\Ps^T \Ps  + n \lambda \I_s \big) \Ps^T .
\end{align*}
Let $\Ps = \U \Si \V^T$ be the full singular value decomposition (SVD),
where $\U$, $\Si$, and $\V$ are $n\times n$, $n\times n$, and $s\times n$ matrices.
It follows that
\begin{align}\label{eq:def_Delta1}
&\EB_{\x'} \big[ \De_1 \big]
\: = \: \EB_{\x'} \Big[ ( \tilde{\K} + n \lambda \I_n )^{-1} \tilde{\k} \tilde{\k}'^T  ( \tilde{\K} + n \lambda \I_n )^{-1}  \Big] \nonumber \\
& = \: ( \tilde{\K} + n \lambda \I_n )^{-1} \EB_{\x'} \big[\tilde{\k} \tilde{\k}'^T  \big] ( \tilde{\K} + n \lambda \I_n )^{-1} \nonumber \\
& \preceq \: \tfrac{1}{n} \big( \Ps \Ps^T + n \lambda \I_n \big)^{-1} 
\Ps \big( \Ps^T \Ps + n \lambda \I_s \big) \Ps^T 
\big( \Ps \Ps^T + n \lambda \I_n \big)^{-1}
\nonumber \\
& = \: \tfrac{1}{n} \U \big( {\Si}^2 + n \lambda \I_n \big)^{-1} {\Si} 
\big( {\Si}^2 + n \lambda [\I_s \oplus \0_{n-s}]  \big) 
{\Si} \big( {\Si}^2 + n \lambda \I_n \big)^{-1}  \U^T \nonumber \\
& \preceq \: \tfrac{1}{n} \U \big( {\Si}^2 + n \lambda \I_n \big)^{-1} {\Si} 
\big( {\Si}^2 + n \lambda \I_n \big) 
{\Si} \big( {\Si}^2 + n \lambda \I_n \big)^{-1}  \U^T  \nonumber \\
& \preceq \: \tfrac{1}{n}  \I_n .
\end{align}
Here, $\oplus $ denotes the direct sum of matrices; obviously, $[\I_s \oplus \0_{n-s}] \preceq \I_n$.

{\bf Analysis of $\Delta_2$.}
Lemma~\ref{lem:rfm1} that shows that
$\EB_{\VM_s} \big[ \Ps \Ps^T \Ps \Ps^T \big]
\preceq  \tfrac{s-1}{s}  \K^2  
+ \tfrac{nb}{s}  \K $.
Since $\EB_{\VM_s} [\tilde{\K}] = \K$ and $\tilde{\K} = \Ps \Ps^T$, 
it follows that
\begin{align} \label{eq:def_Delta2} 
&\EB_{\VM_s} \big[ \Delta_2  \big]
\: = \: \y^T ( {\K} + n \lambda \I_n )^{-1} 
\EB_{\VM_s} \big[ ( \tilde{\K} - \K )^2  \big]
( {\K} + n \lambda \I_n )^{-1}  \y \nonumber \\
& = \: \y^T ( {\K} + n \lambda \I_n )^{-1} 
\EB_{\VM_s} \big[ \tilde{\K}^2 + \K^2 - \K \tilde{\K} - \tilde{\K} \K   \big]
( {\K} + n \lambda \I_n )^{-1}  \y \nonumber \\
& = \: \y^T ( {\K} + n \lambda \I_n )^{-1} 
\Big[ \EB_{\VM_s} \big[ \Ps \Ps^T \Ps \Ps^T \big] -  \K^2   \Big]
( {\K} + n \lambda \I_n )^{-1} \y \nonumber \\
& \leq \: \y^T ( {\K} + n \lambda \I_n )^{-1} 
( \tfrac{n b}{s } \K - \tfrac{1}{s } \K^2 )
( {\K} + n \lambda \I_n )^{-1} \y \nonumber \\
& \leq \: \tfrac{n b}{s } \big\| \K^{\frac{1}{2}} ( {\K} + n \lambda \I_n )^{-1} \y \big\|_2^2 .
\end{align}

{\bf Analysis of $\Delta_3$.}
Lemma~\ref{lem:rfm2} shows that $\EB_{\VM_s} \big[  ( \tilde{\k}' - \k') ( \tilde{\k}' - \k')^T \big]  
\preceq \tfrac{b}{s} \K - \tfrac{1}{s}  \k' \k'^T \preceq \tfrac{b}{s} \K $. 
It follows from the definition of $\Delta_3$ that 
\begin{align}\label{eq:def_Delta3}
& \EB_{\VM_s } [ \Delta_3 ] \nonumber \\
& = \, \y^T ( {\K} + n \lambda \I_n )^{-1} 
\EB_{ \VM_s} \Big[ (\tilde{\k}' - {\k}') (\tilde{\k}' - {\k}')^T  \Big] 
( {\K} + n \lambda \I_n )^{-1} \y \nonumber \\
& \leq \, \tfrac{b}{s}  \,
\y^T ( {\K} + n \lambda \I_n )^{-1} 
\, \K \,
( {\K} + n \lambda \I_n )^{-1} \y  \nonumber \\
& \leq \, \tfrac{b}{s}  \,
\big\| \K^{\frac{1}{2}} ( {\K} + n \lambda \I_n )^{-1} \y \big\|_2^2 .  
\end{align}

{\bf Completing the proof.}
Now, we prove the theorem using the bounds on $\De_1$, $\Delta_2$, and $\Delta_3$.
It follows from \eqref{eq:thm:full:1} and \eqref{eq:def_Delta1} that if the event \eqref{eq:condition1} happens, then
\begin{align} \label{eq:thm:full:2}
&\EB_{\VM_s, \x'} \Big[ \big( \tilde{f}_{\lambda}(\x'; \XM_n, \VM_s) - {f}_{\lambda} (\x' ; \XM_n ) \big)^2 \Big] \nonumber \\
& \leq \: 2 \, \EB_{\VM_s} \Big\{ \EB_{\x'} \Big[ \Delta_2 \cdot \| \Delta_1 \|_2 \, \Big| \, \VM_s \Big] \Big\} 
+ 2 \, \EB_{\VM_s, \x'}  \big[ \Delta_3  \big] \nonumber \\
& \leq \: 2 \, \EB_{\VM_s} \Big\{ \Delta_2 \cdot \EB_{\x'} \Big[ \| \Delta_1 \|_2 \, \Big| \, \VM_s \Big] \Big\} 
+ 2 \, \EB_{\VM_s, \x'}  \big[ \Delta_3  \big] \nonumber \\
& \leq \: \tfrac{2}{n} \, \EB_{\VM_s} [ \Delta_2 ] 
+ 2 \, \EB_{\VM_s, \x'}  \big[ \Delta_3  \big] .
\end{align}
It follows from \eqref{eq:def_Delta2}, \eqref{eq:def_Delta3}, and \eqref{eq:thm:full:2} that
\begin{align*} 
&\EB_{\VM_s, \x'} \Big[ \big( \tilde{f}_{\lambda} (\x'; \XM_n, \VM_s) - {f}_{\lambda} (\x' ; \XM_n) \big)^2 \Big] 
\: \leq \: \tfrac{4 b}{s } \big\| \K^{\frac{1}{2}} ( {\K} + n \lambda \I_n )^{-1} \y \big\|_2^2 .
\end{align*}
%Thus if the event \eqref{eq:condition1} happens, it holds with probability $1-\delta_2$ that
%\begin{align*} 
%&\EB_{\x'} \Big[ \big( \tilde{f}_{\lambda} (\x'; \XM_n, \VM_s) - {f}_{\lambda} (\x' ; \XM_n) \big)^2 \Big] 
%\: \leq \: \tfrac{4 b}{s \delta_2} \big\| \K^{\frac{1}{2}} ( {\K} + n \lambda \I_n )^{-1} \y \big\|_2^2 .
%\end{align*}
Since the event \eqref{eq:condition1} happens with probability at least $1-\delta$, the theorem follows from the above inequality and the union bound.

\section{Analysis of Lower Bound} \label{sec:lower}

In this section, we prove Theorem~\ref{thm:lower1}.
We use the angular similarity kernel $\kappa (\x, \x') = \frac{2}{\pi} \arcsin \frac{ \x^T \x' }{ \|\x\|_2 \|\x'\|_2 }$ to establish a lower bound that matches the upper bound. 
The random sign feature $\psi (\x; \v) = \sgn (\x^T \v)$, with $\v$ drawn uniformly from the unit sphere, enjoys Assumption~\ref{assumption:b} with $b=1$.

It follows from the definition of $f$ and $\tilde{f}_s$ in \eqref{eq:f_krr} and \eqref{eq:f_rfm} that
\begin{align} \label{eq:thm:lower1:proof1}
& \big[ \tilde{f}_{\lambda} (\x'; \XM_n, \VM_s) - {f}_{\lambda} (\x' ; \XM_n) \big]^2  \nonumber \\
& = \, \Big[ \tilde{\k}'^T ( \tilde{\K} + n \lambda \I_n )^{-1} \y
- {\k}'^T ( {\K} + n \lambda \I_n )^{-1} \y \Big]^2  \nonumber \\
& = \,   \Big[ \tilde{\k}'^T \Big( \big( \tilde{\K} + n \lambda \I_n \big)^{-1} - \big( {\K} + n \lambda \I_n \big)^{-1} \Big) \y 
+  \big( \tilde{\k}' - {\k}' \big)^T \big( {\K} + n \lambda \I_n \big)^{-1} \y \Big]^2   \nonumber \\ 
& = \,  \Big[ \tilde{\k}'^T  ( \tilde{\K} + n \lambda \I_n )^{-1} 
(  \K - \tilde{\K} )
( {\K} + n \lambda \I_n )^{-1}  \y 
+ \big( \tilde{\k}' - {\k}' \big)^T \big( {\K} + n \lambda \I_n \big)^{-1} \y \Big]^2 \nonumber \\
& = \, 
\y^T \big( {\K} + n \lambda \I_n \big)^{-1}
\tha \tha^T
\big( {\K} + n \lambda \I_n \big)^{-1} \y ,
\end{align}
where we define $\tha$ as
\begin{align*}
\tha \: = \: (  \K - \tilde{\K} ) ( \tilde{\K} + n \lambda \I_n )^{-1} \tilde{\k}' + \big( \tilde{\k}' - {\k}' \big)
\end{align*}
We have
\begin{align*}
& \EB_{\x' } \big[ \tha \tha^T \big]
\: = \: 
(  \K - \tilde{\K} ) ( \tilde{\K} + n \lambda \I_n )^{-1} 
\EB_{\x' } \Big[ \tilde{\k}'  \big( \tilde{\k}' - {\k}' \big)^T  \Big] \\
& \qquad + \EB_{\x' } \Big[  \big( \tilde{\k}' - {\k}' \big) \tilde{\k}'^T  \Big]
( \tilde{\K} + n \lambda \I_n )^{-1} (  \K - \tilde{\K} )\\
& \qquad + (  \K - \tilde{\K} ) ( \tilde{\K} + n \lambda \I_n )^{-1} 
\EB_{\x' } \Big[ \tilde{\k}' \tilde{\k}'^T  \Big]
( \tilde{\K} + n \lambda \I_n )^{-1} (  \K - \tilde{\K} ) \\
& \qquad + \EB_{\x' } \Big[ \big( \tilde{\k}' - {\k}' \big) \big( \tilde{\k}' - {\k}' \big)^T \Big] .
\end{align*}
Since $\x'$ is uniformly drawn from $\XM_n$, we have
\begin{align*}
& \EB_{\x' } \big[  {\k}' {\k}'^T \big]
\: = \: \frac{1}{n} \sum_{i=1}^n {\k}_i {\k}_i^T 
\: = \: \frac{1}{n} {\K}^2 ,\\
& \EB_{\x' } \big[  \tilde{\k}' \tilde{\k}'^T \big]
\: = \: \frac{1}{n} \sum_{i=1}^n \tilde{\k}_i \tilde{\k}_i^T 
\: = \: \frac{1}{n} \tilde{\K}^2 ,\\
& \EB_{\x' } \big[  \tilde{\k}' {\k}'^T \big]
\: = \: \frac{1}{n} \sum_{i=1}^n \tilde{\k}_i {\k}_i^T 
\: = \: \frac{1}{n} \tilde{\K} \K .
\end{align*}
It follows that
\begin{align} \label{eq:thm:lower1:proof3}
& n \cdot \EB_{\x' } \big[ \tha \tha^T \big]
\: = \: 
2 (  \K - \tilde{\K} ) ( \tilde{\K} + n \lambda \I_n )^{-1}  \tilde{\K} \big( \tilde{\K} - \K \big) \nonumber \\
& \qquad + (  \K - \tilde{\K} ) ( \tilde{\K} + n \lambda \I_n )^{-1} 
\tilde{\K}^2
( \tilde{\K} + n \lambda \I_n )^{-1} (  \K - \tilde{\K} ) 
+ \big( \tilde{\K} - \K \big)^2 \nonumber \\
& = \: \big( \tilde{\K} - \K \big)
\Big[ - 2( \tilde{\K} + n \lambda \I_n )^{-1}  \tilde{\K}
+ ( \tilde{\K} + n \lambda \I_n )^{-1} \tilde{\K}^2 ( \tilde{\K} + n \lambda \I_n )^{-1}
+ \I_n \Big]
\big( \tilde{\K} - \K \big) \nonumber \\
& = \: \big( \tilde{\K} - \K \big)
\big[ ( \tilde{\K} + n \lambda \I_n )^{-1}  \tilde{\K}-  \I_n \big]^2
\big( \tilde{\K} - \K \big) \nonumber \\
& = \: \big( \tilde{\K} - \K \big)
\big[ n \lambda ( \tilde{\K} + n \lambda \I_n )^{-1}   \big]^2
\big( \tilde{\K} - \K \big) \nonumber \\
& \geq \: \Big( \frac{n \lambda}{ \| \tilde{\K} \|_2 + n \lambda } \Big)^2 \big( \tilde{\K} - \K \big)^2
\end{align}
It follows from \eqref{eq:thm:lower1:proof1} and \eqref{eq:thm:lower1:proof3} that
\begin{align} \label{eq:thm:lower1:proof4}
&\EB_{\x' }  \Big\{  \big[ \tilde{f}_{\lambda} (\x'; \XM_n, \VM_s) - {f}_{\lambda} (\x' ; \XM_n) \big]^2  \Big\} \\
& = \: \y^T \big( {\K} + n \lambda \I_n \big)^{-1}
\tha \tha^T
\big( {\K} + n \lambda \I_n \big)^{-1} \y \\
& \geq \: \frac{1}{n} \Big( \frac{n \lambda}{ \| \tilde{\K} \|_2 + n \lambda } \Big)^2 \,
\y^T \big( {\K} + n \lambda \I_n \big)^{-1}
\big( \tilde{\K} - \K \big)^2
\big( {\K} + n \lambda \I_n \big)^{-1} \y  .
\end{align}
In the theorem we assume that $\XM_n = \{ \x_1 , \cdots , \x_n \}$ is a set of training samples uniformly from the unit sphere.
For an observed $\VM_s$, the randomness of $\XM_n$ makes the feature matrix $\sqrt{s} \Ps$, whose $(i,l)$-th entry is $\psi (\x_i; \v_l) $, a random sign matrix (aka Bernoulli random matrix).
It is well known that the spectral norm of an $n\times s$ (with $n > s$) random sign matrix is concentrated around $ \sqrt{n} + \sqrt{s} $ with high probability, and thus 
\begin{equation*}
\| \tilde{\K} \|_2 = \| \Ps \|_2^2 \geq \tfrac{ n }{s} (1 - o(1)).
\end{equation*}
It follows from \eqref{eq:thm:lower1:proof4} that
\begin{align}  \label{eq:thm:lower1:proof5}
&\EB_{\x' }  \Big\{  \big[ \tilde{f}_{\lambda} (\x'; \XM_n, \VM_s) - {f}_{\lambda} (\x' ; \XM_n) \big]^2  \Big\} \nonumber \\
& \geq \: \frac{1}{n} (1 - o(1)) \Big( \frac{s \lambda}{ 1 +  s \lambda } \Big)^2 \,
\y^T \big( {\K} + n \lambda \I_n \big)^{-1}
\big( \tilde{\K} - \K \big)^2
\big( {\K} + n \lambda \I_n \big)^{-1} \y  .
\end{align}
The proof of Lemma~\ref{lem:rfm1} shows that
\begin{align*}
& \EB_{\VM_s} \big[ \Ps \Ps^T \Ps \Ps^T \big] 
\: = \: \frac{s-1}{s}  \K^2  
+ \EB_{\v_p}  \Bigg\{
\bigg[ \sum_{l=1}^n   \psi^2 (\x_l; \v_p)\bigg] 
\, \Big[\ps_{:p} \ps_{:p}^T \Big]
\Bigg\} .
\end{align*}
It can be easily show that $\EB_{\VM_s} \big[ (\tilde{\K} - \K)^2 \big] = \EB_{\VM_s} [\tilde{\K}^2 ] - \K^2 = \EB_{\VM_s} [\Ps \Ps^T \Ps \Ps^T] - \K^2$.
Additionally using $\psi^2 (\cdot ; \cdot ) = 1$, we obtain
\begin{eqnarray} \label{eq:thm:lower1:proof6}
\EB_{\VM_s} \big[ (\tilde{\K} - \K)^2 \big] 
\: = \: \frac{s-1}{s}  \K^2  
+ n \cdot \EB_{\v_p} 
\Big[\ps_{:p} \ps_{:p}^T \Big] - \K^2 
\: = \: - \frac{1}{s} \K^2 + \frac{n}{s} \K .
\end{eqnarray}
Finally, it follows from \eqref{eq:thm:lower1:proof5} and \eqref{eq:thm:lower1:proof6} that with high probability,
\begin{align*}
	& \EB_{\x', \VM_s} \Big[ \big( \tilde{f}_{\lambda} (\x'; \XM_n, \VM_s) - {f}_{\lambda} (\x' ; \XM_n) \big)^2 \Big]  \\
	& \geq \:  \frac{1}{s} (1 - o(1)) \Big( \frac{s \lambda}{ 1 +  s \lambda } \Big)^2 \,
	\y^T ( {\K} + n \lambda \I_n )^{-1} 
	( \K - \tfrac{1}{n} \K^2 )
	( {\K} + n \lambda \I_n )^{-1}  \y  ,
\end{align*}
by which the theorem follows.

\begin{table}
	\label{tab:data}
	\setlength{\tabcolsep}{0.3pt}
	\caption{Descriptions of the used datasets.}
	\begin{center}
		\begin{footnotesize}
			\begin{tabular}{c c c}
				\hline
				~{\bf Data}~&~~~{\bf \#Instances}~~~&~~~{\bf \#Features}~~~\\
				\hline
				~~Cadata~~ & $20,640$ & $8$ \\
				~~Covtype~~ & $581,012$ & $54$ \\
				~~Cpusmall~~ & $8,192$ & $12$ \\
				~~MSD~~ & $463,715$ & $90$ \\
				\hline
			\end{tabular}
		\end{footnotesize}
	\end{center}
\end{table}

\section{Experiments} \label{sec:exp}

We conduct experiments on real data to verify our theories.
In Section~\ref{sec:exp:setting}, we describe the experiment settings.
In Section~\ref{sec:exp:rate}, we show that the $\frac{1}{s}$ convergence rate in our theory matches empirical observations.
In Section~\ref{sec:exp:bound}, we demonstrate that our upper bound does not much overestimate the true error.

\begin{figure}[!t]
	\centering
	\subfigure[MSD, $\lambda = 0.2/\sqrt{n}$]{\includegraphics[width=0.3\textwidth]{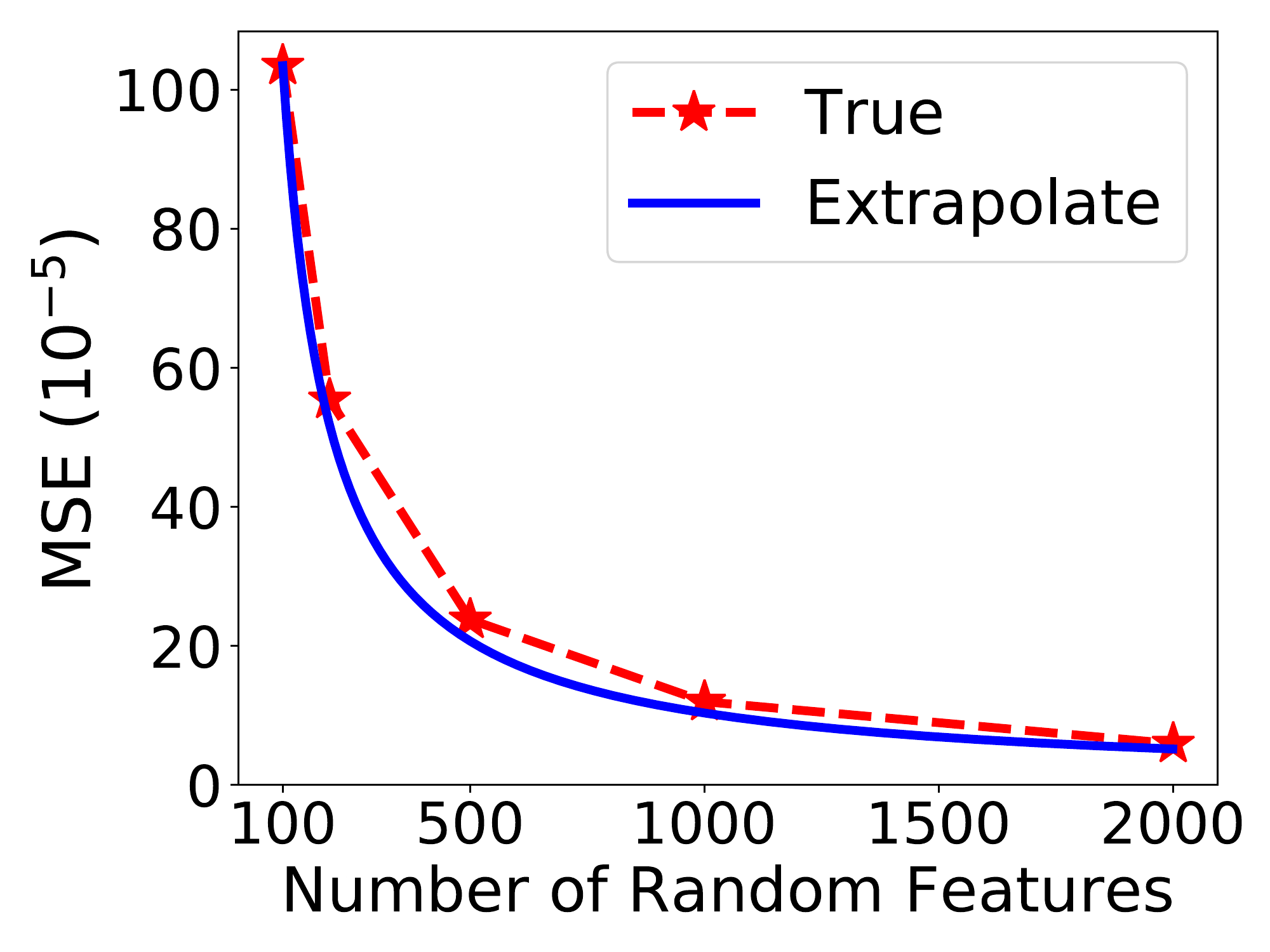}}
	\subfigure[MSD, $\lambda = 1/\sqrt{n}$]{\includegraphics[width=0.3\textwidth]{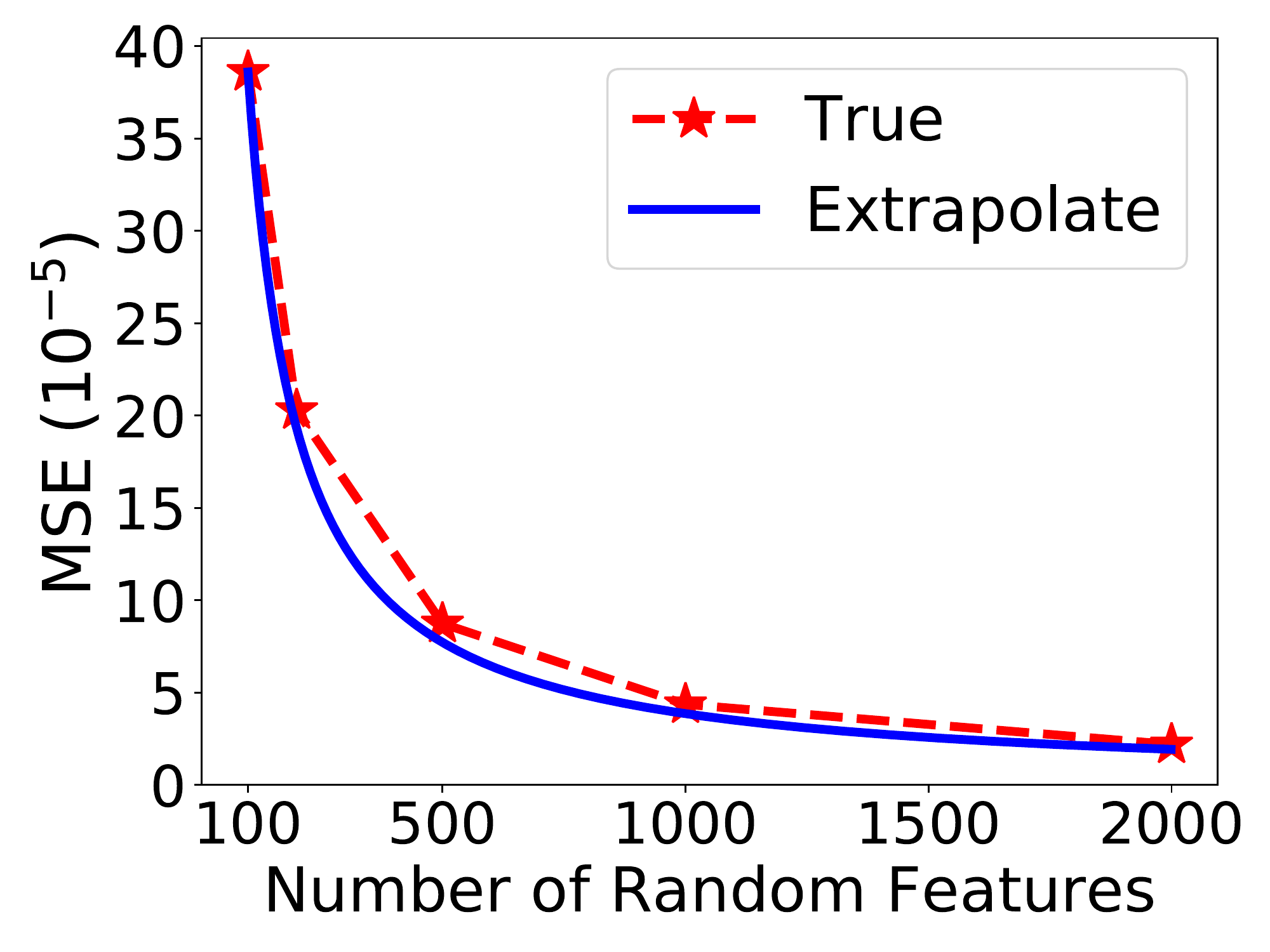}}
	\subfigure[MSD, $\lambda = 5/\sqrt{n}$]{\includegraphics[width=0.3\textwidth]{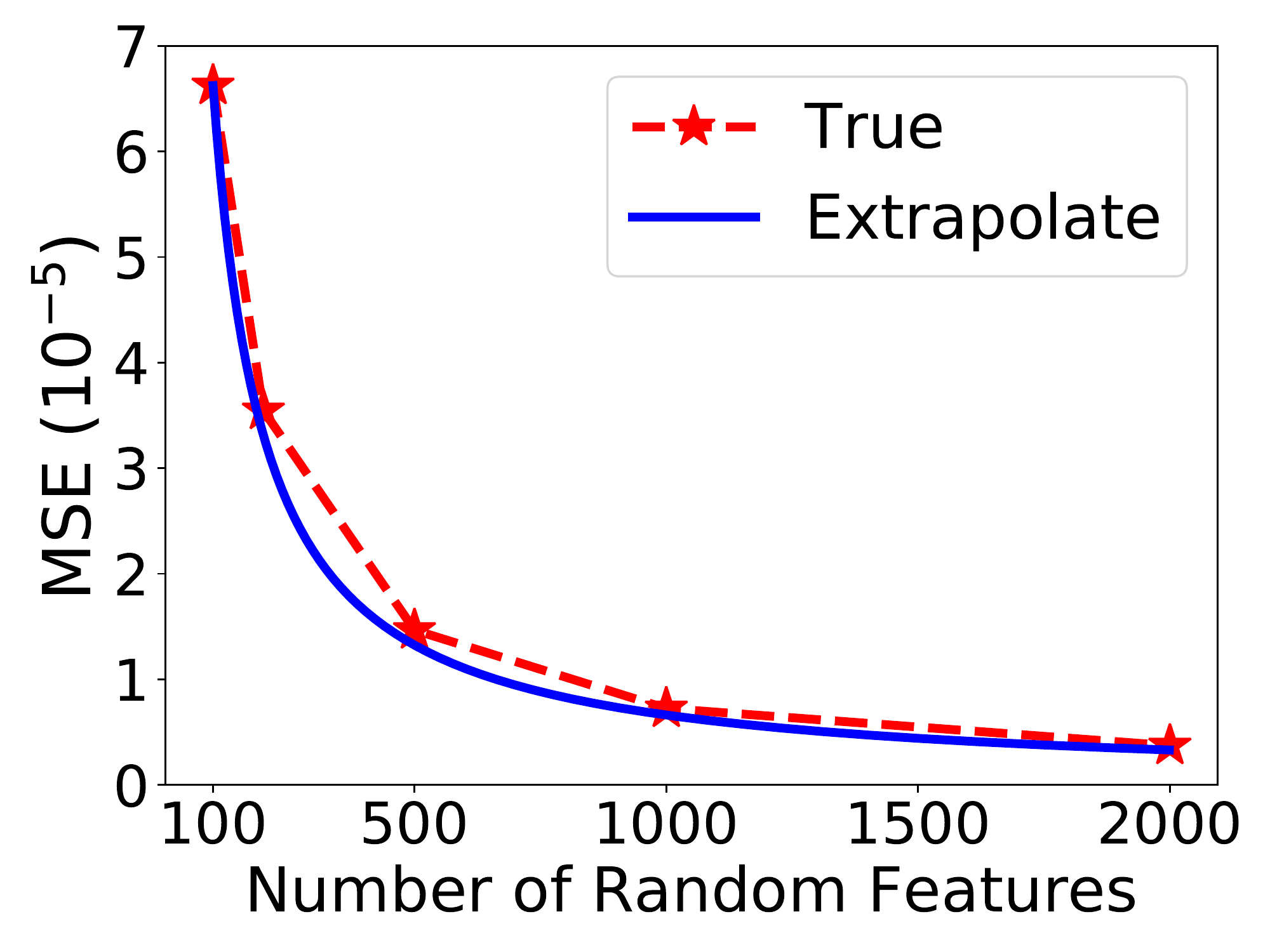}}
	\subfigure[Cadata, $\lambda = 0.2/\sqrt{n}$]{\includegraphics[width=0.3\textwidth]{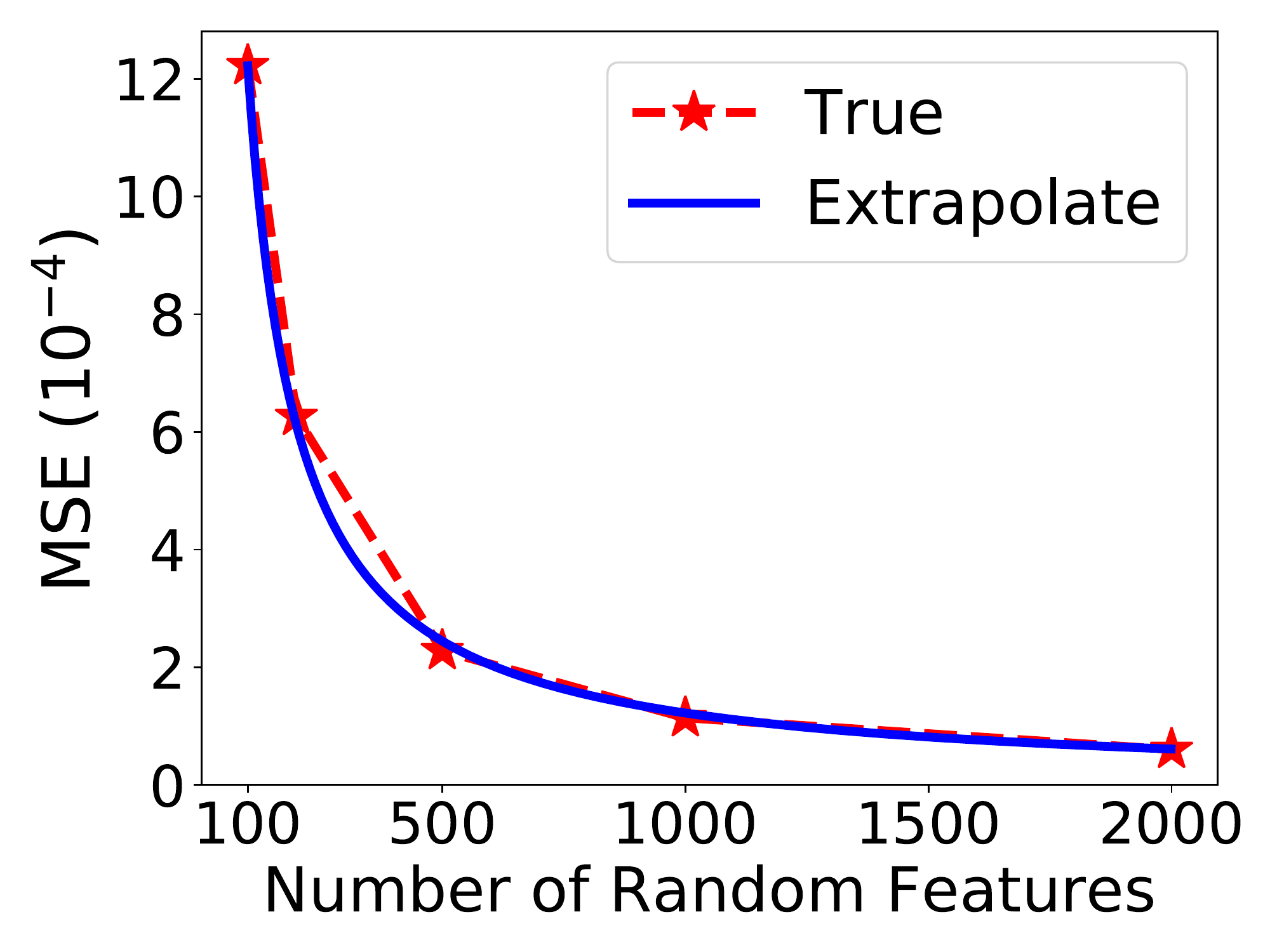}}
	\subfigure[Cadata, $\lambda = 1/\sqrt{n}$]{\includegraphics[width=0.3\textwidth]{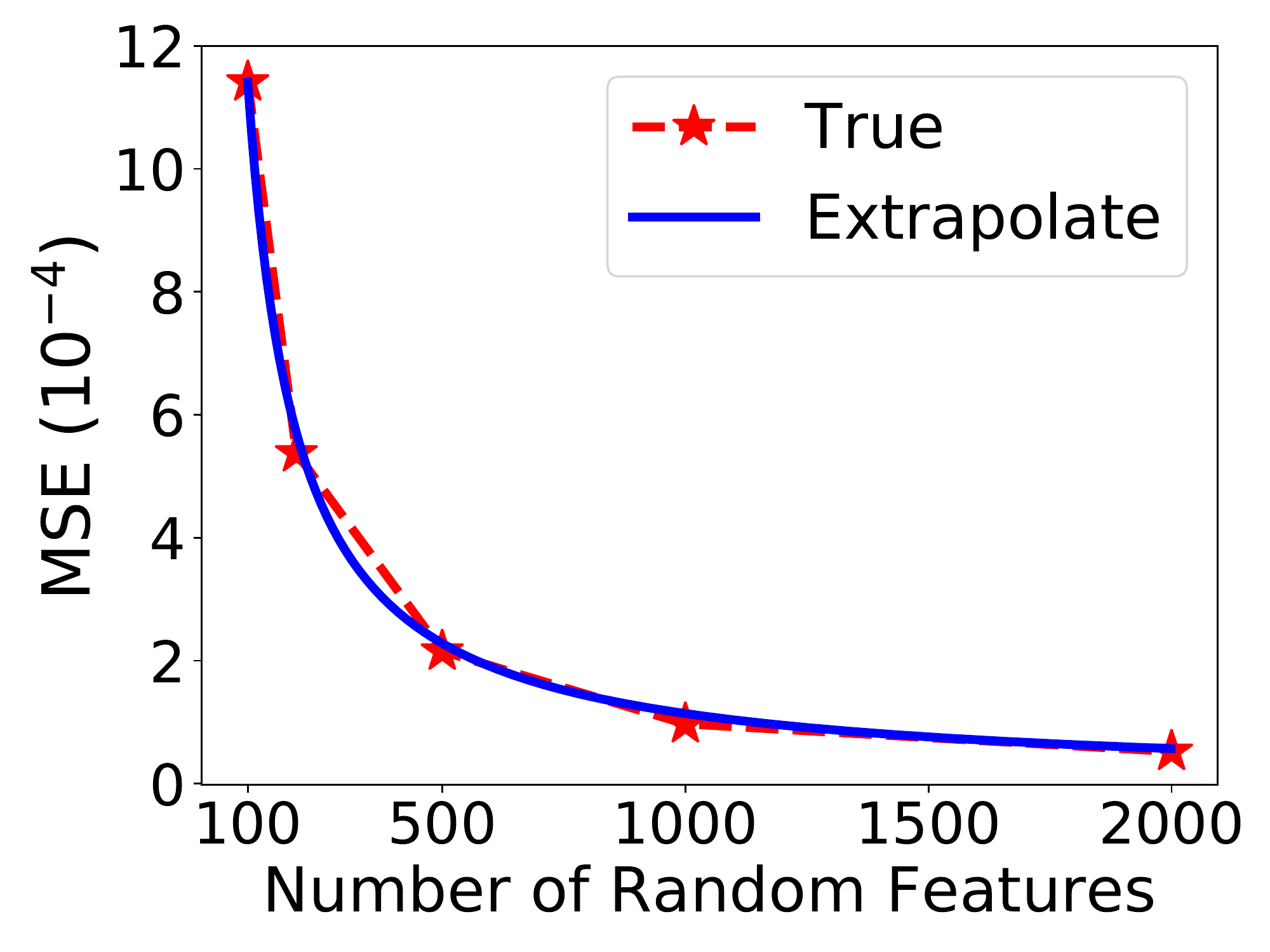}}
	\subfigure[Cadata, $\lambda = 5/\sqrt{n}$]{\includegraphics[width=0.3\textwidth]{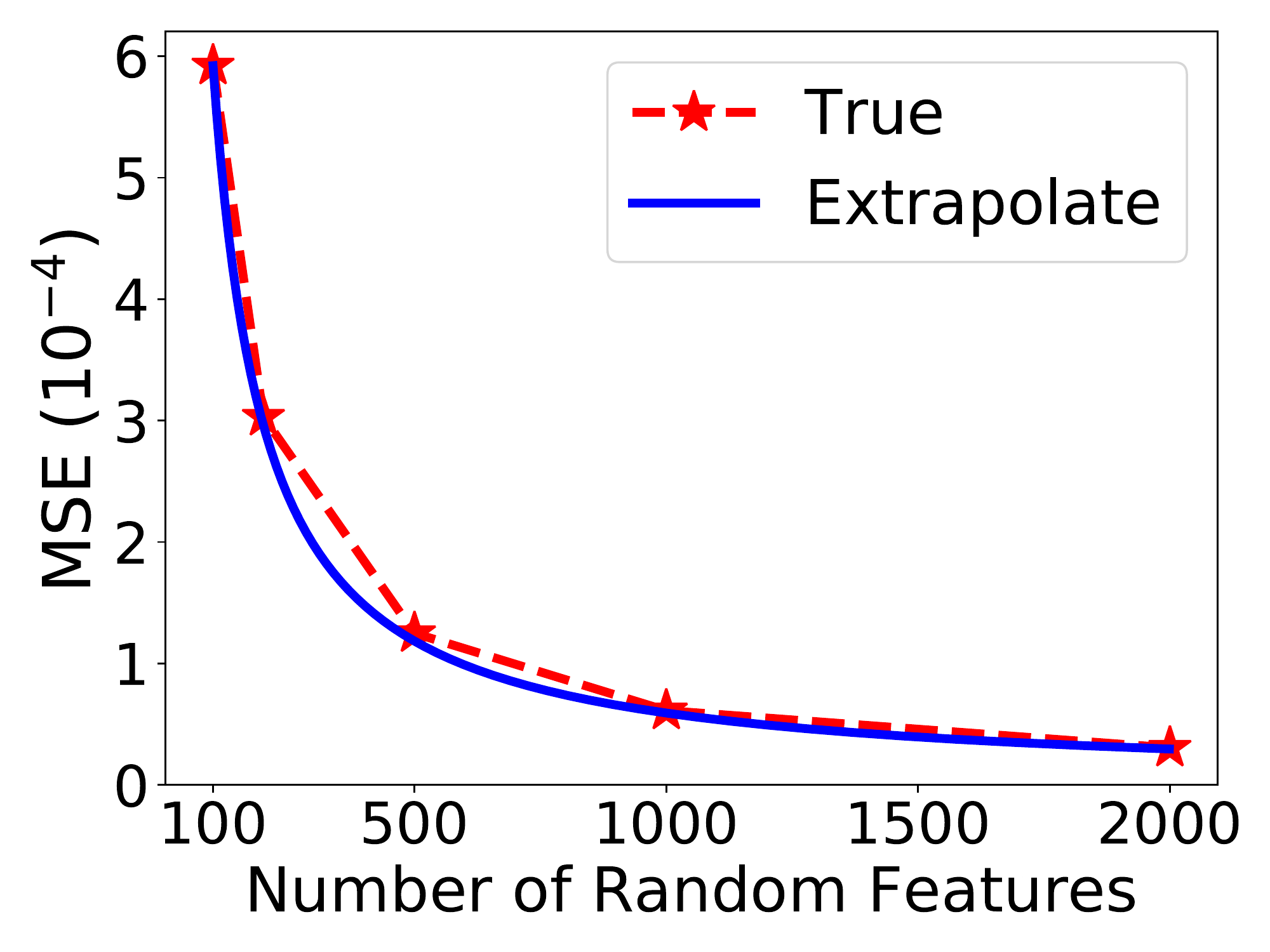}}
	\subfigure[Cpusmall, $\lambda = 0.2/\sqrt{n}$]{\includegraphics[width=0.3\textwidth]{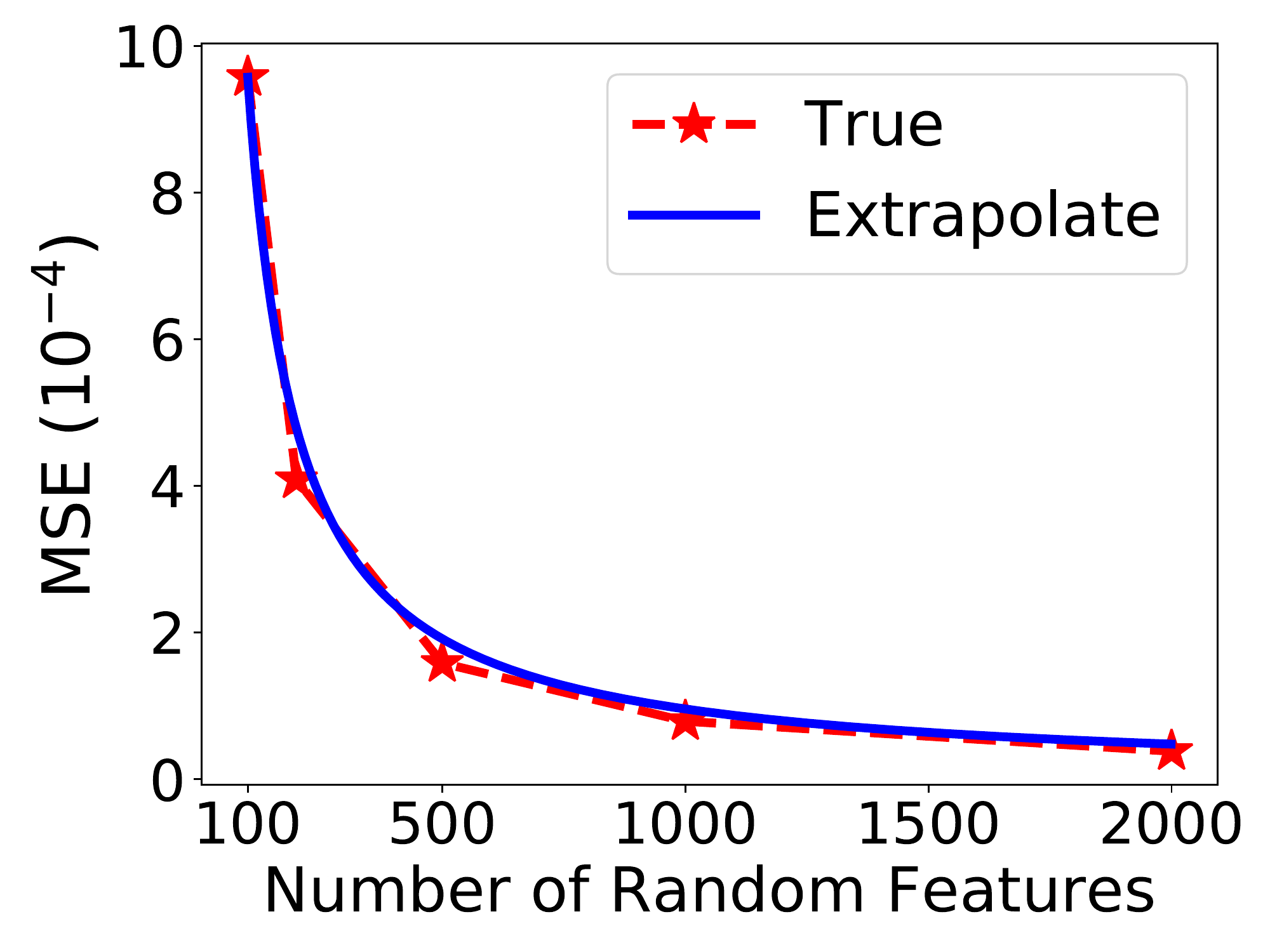}}
	\subfigure[Cpusmall, $\lambda = 1/\sqrt{n}$]{\includegraphics[width=0.3\textwidth]{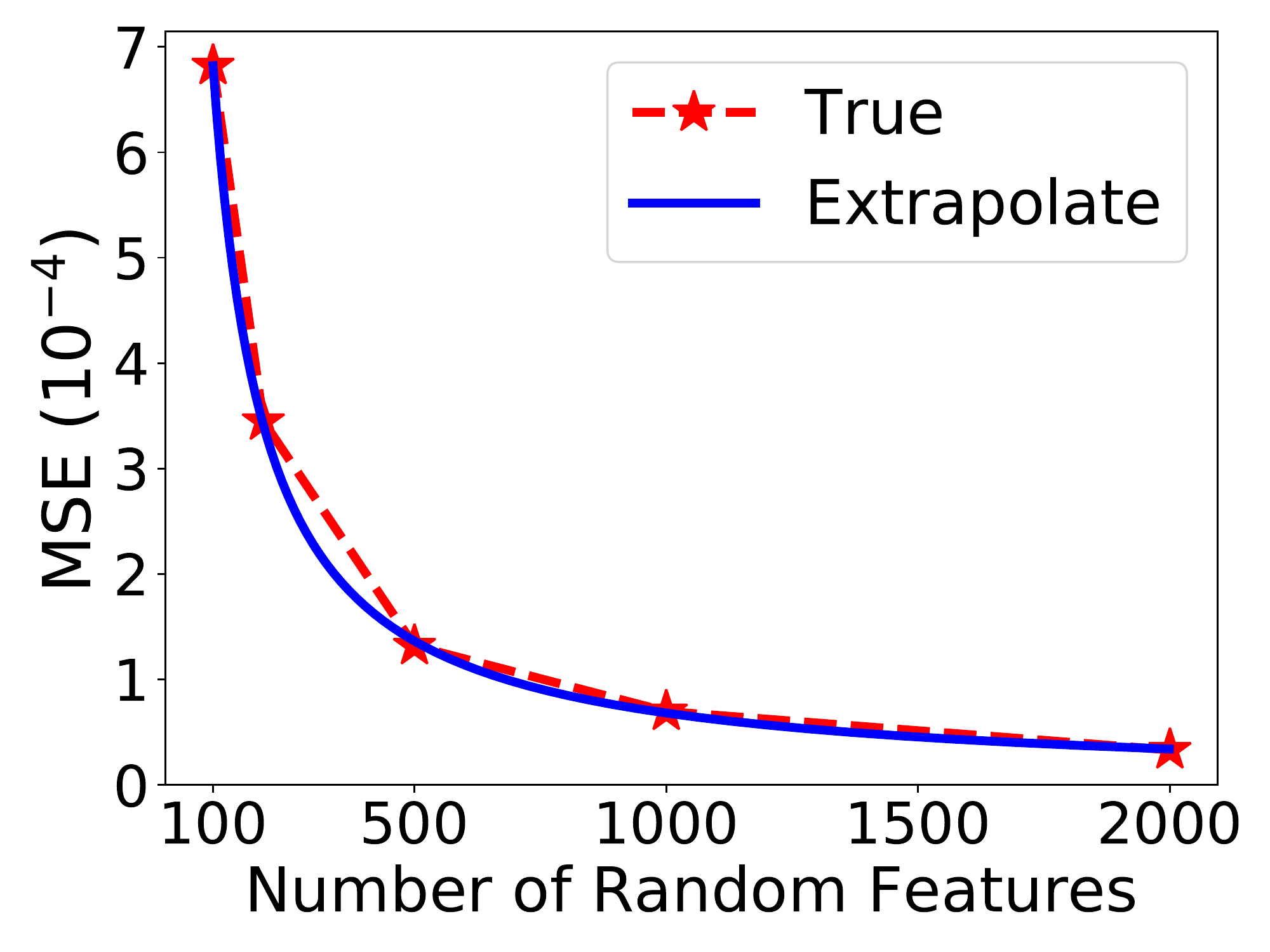}}
	\subfigure[Cpusmall, $\lambda = 5/\sqrt{n}$]{\includegraphics[width=0.3\textwidth]{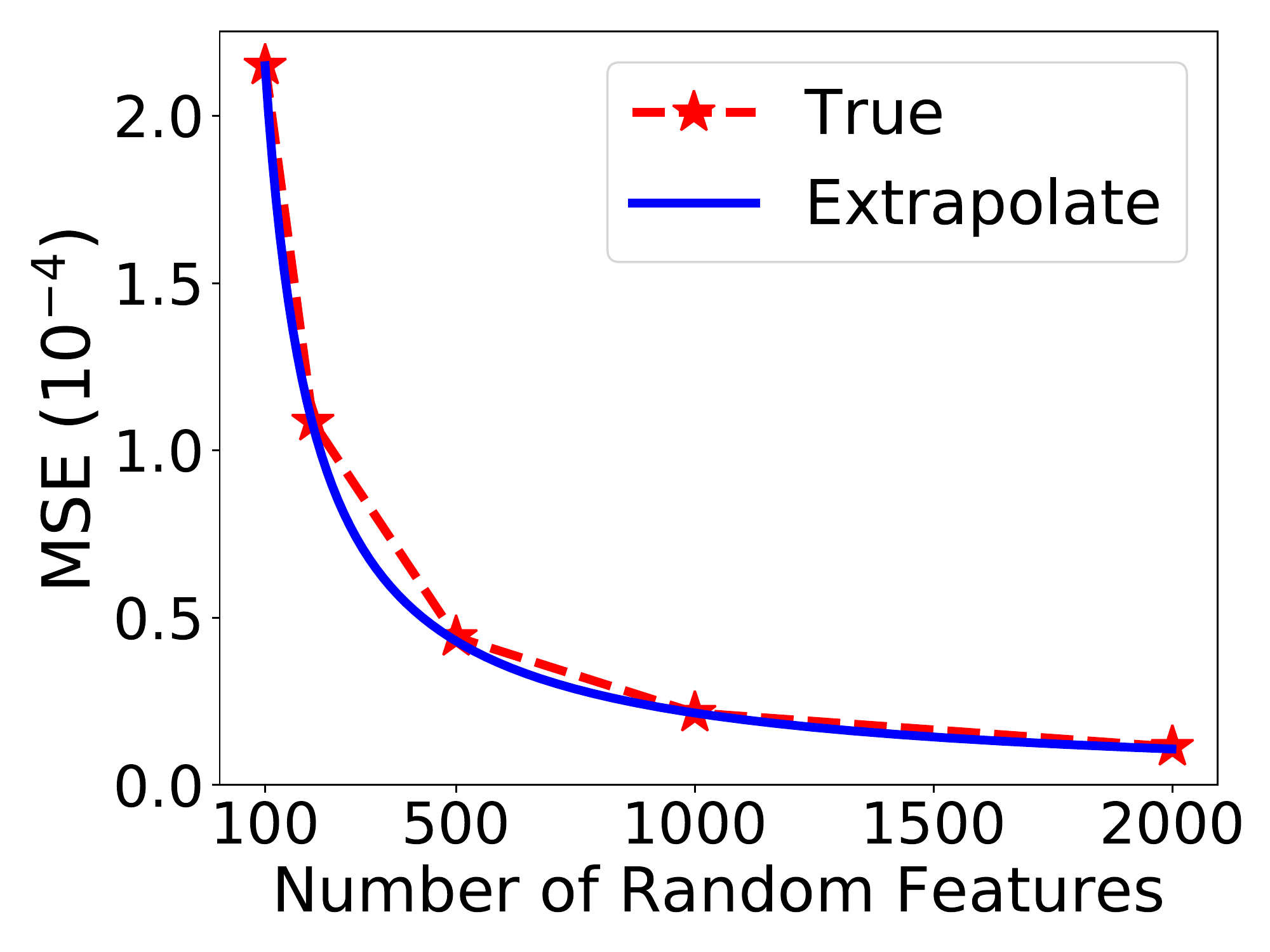}}
	\subfigure[Covtype, $\lambda = 0.2/\sqrt{n}$]{\includegraphics[width=0.3\textwidth]{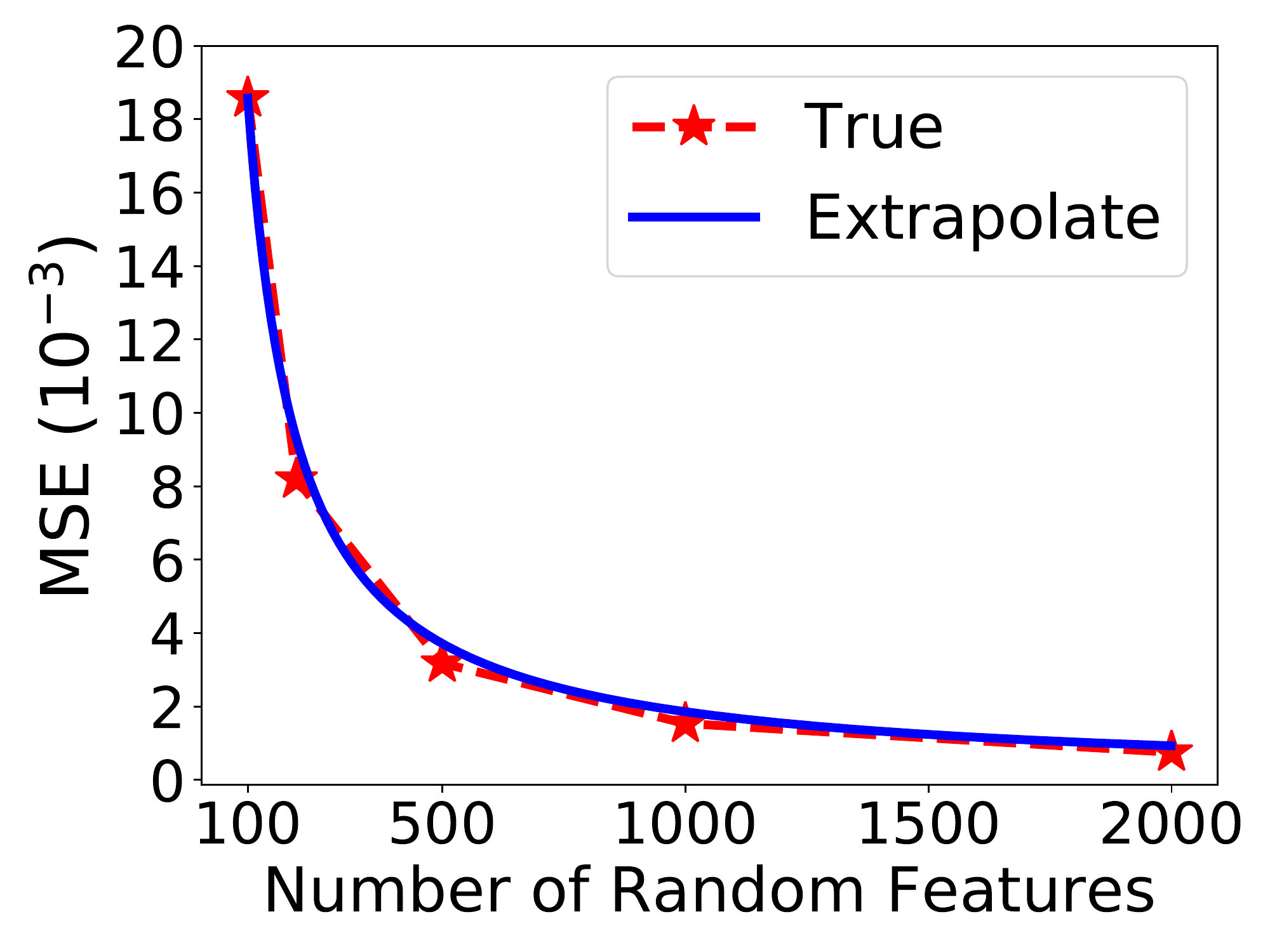}}
	\subfigure[Covtype, $\lambda = 1/\sqrt{n}$]{\includegraphics[width=0.3\textwidth]{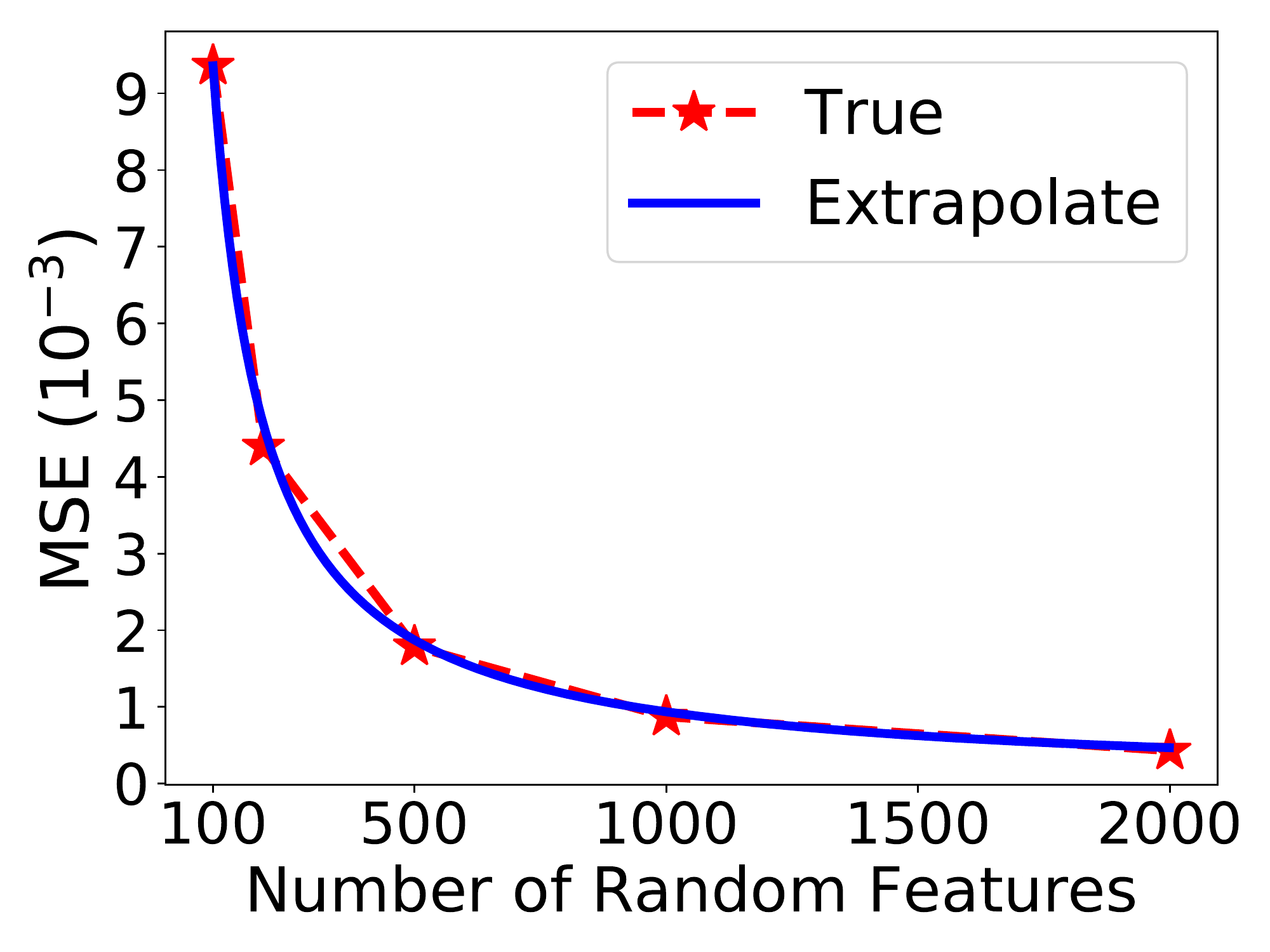}}
	\subfigure[Covtype, $\lambda = 5/\sqrt{n}$]{\includegraphics[width=0.3\textwidth]{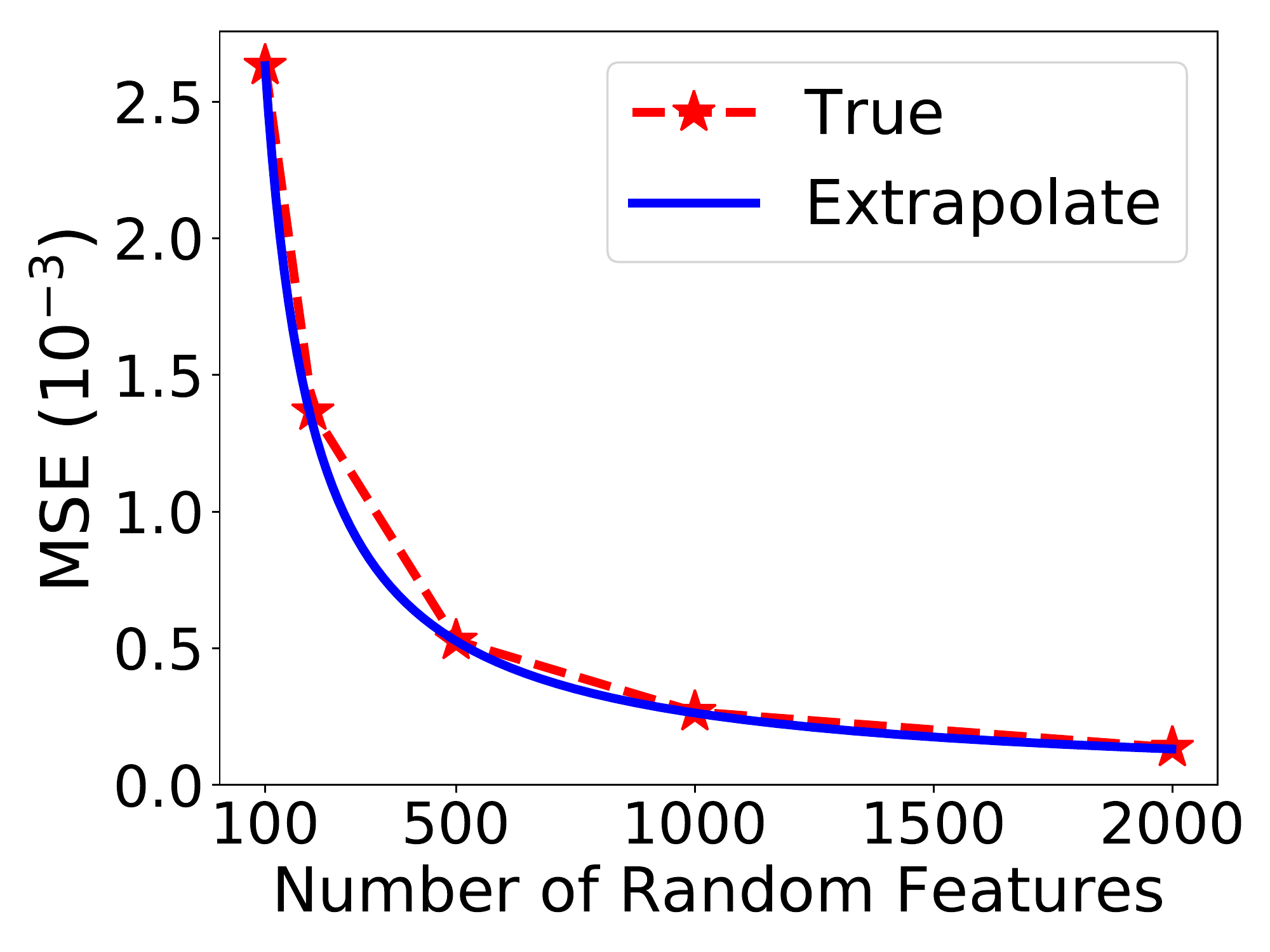}}
	\caption{Plot of the MSE $\EB \big[ ( f_{\lambda} - \tilde{f}_\lambda )^2 \big]$ against $s$ (using the RBF kernel.)
		For MSD, Covtype, and Cadata, we use $10,000$ samples for training and $10,000$ for test.
		For Cpusmall, we use $5,000$ for training and $3,192$ for test.}
	\label{fig:rbf_s}
\end{figure}

\begin{figure}[!t]
	\centering
	\subfigure[MSD, $\lambda = 0.2/\sqrt{n}$]{\includegraphics[width=0.3\textwidth]{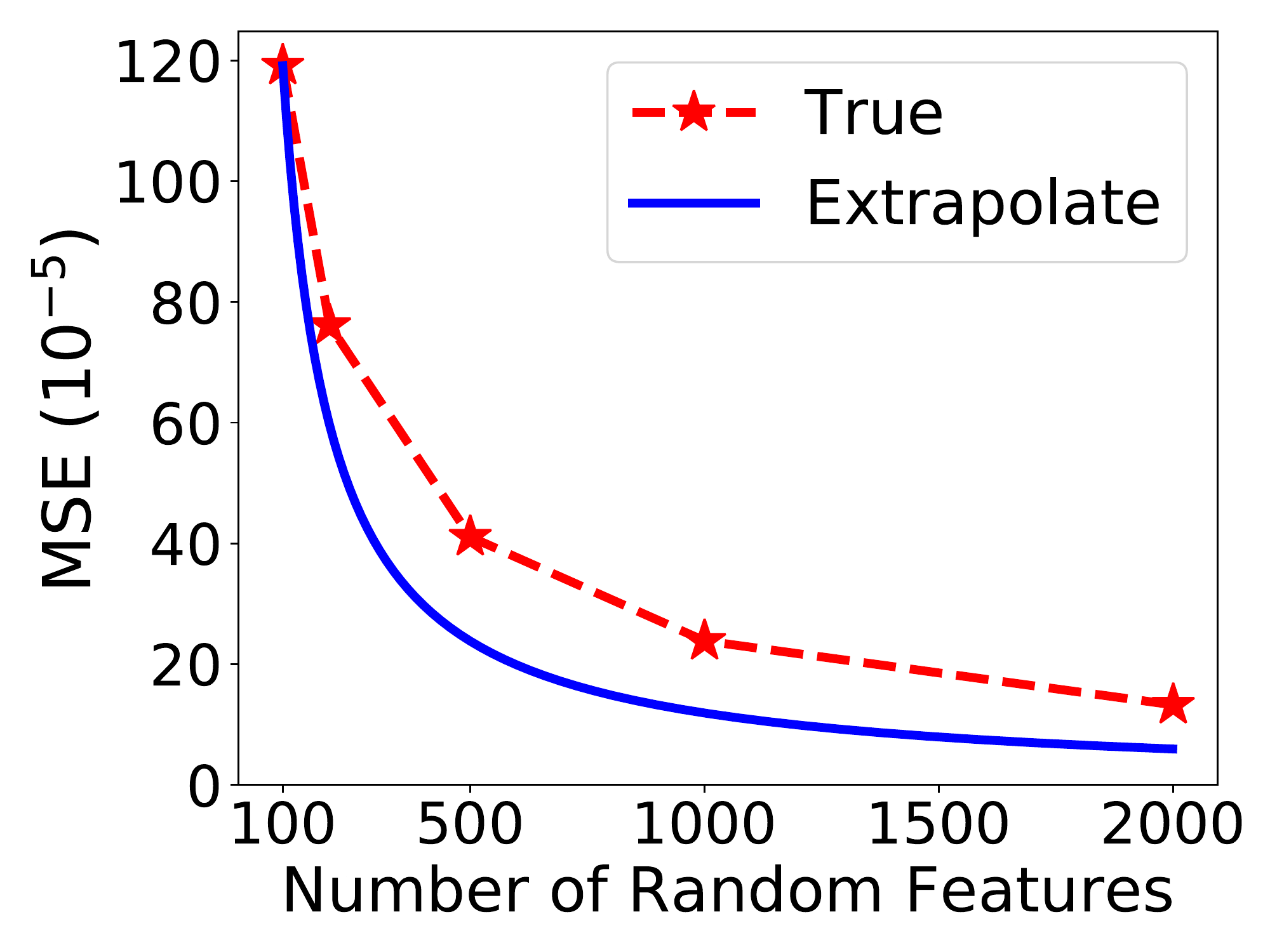}}
	\subfigure[MSD, $\lambda = 1/\sqrt{n}$]{\includegraphics[width=0.3\textwidth]{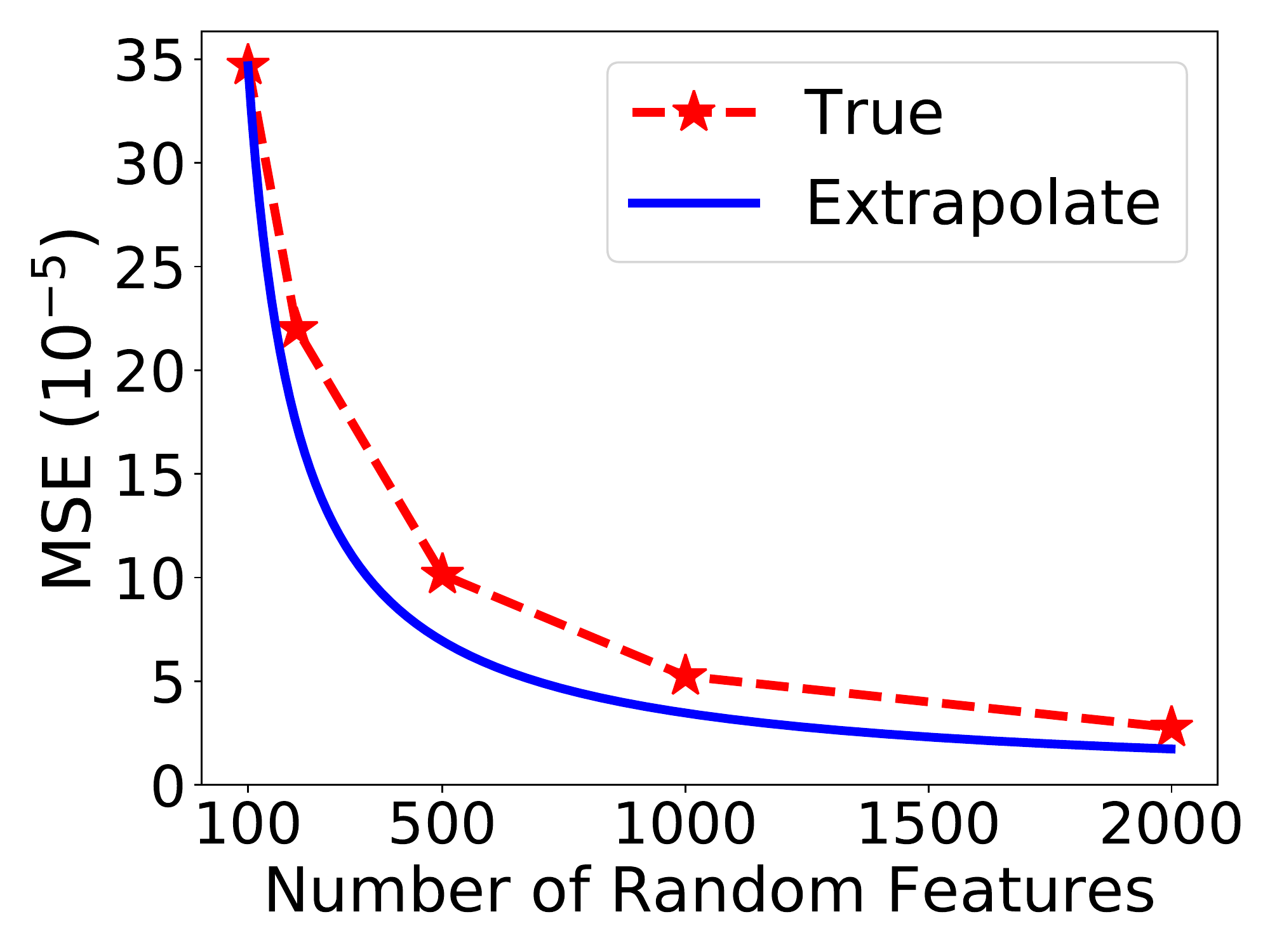}}
	\subfigure[MSD, $\lambda = 5/\sqrt{n}$]{\includegraphics[width=0.3\textwidth]{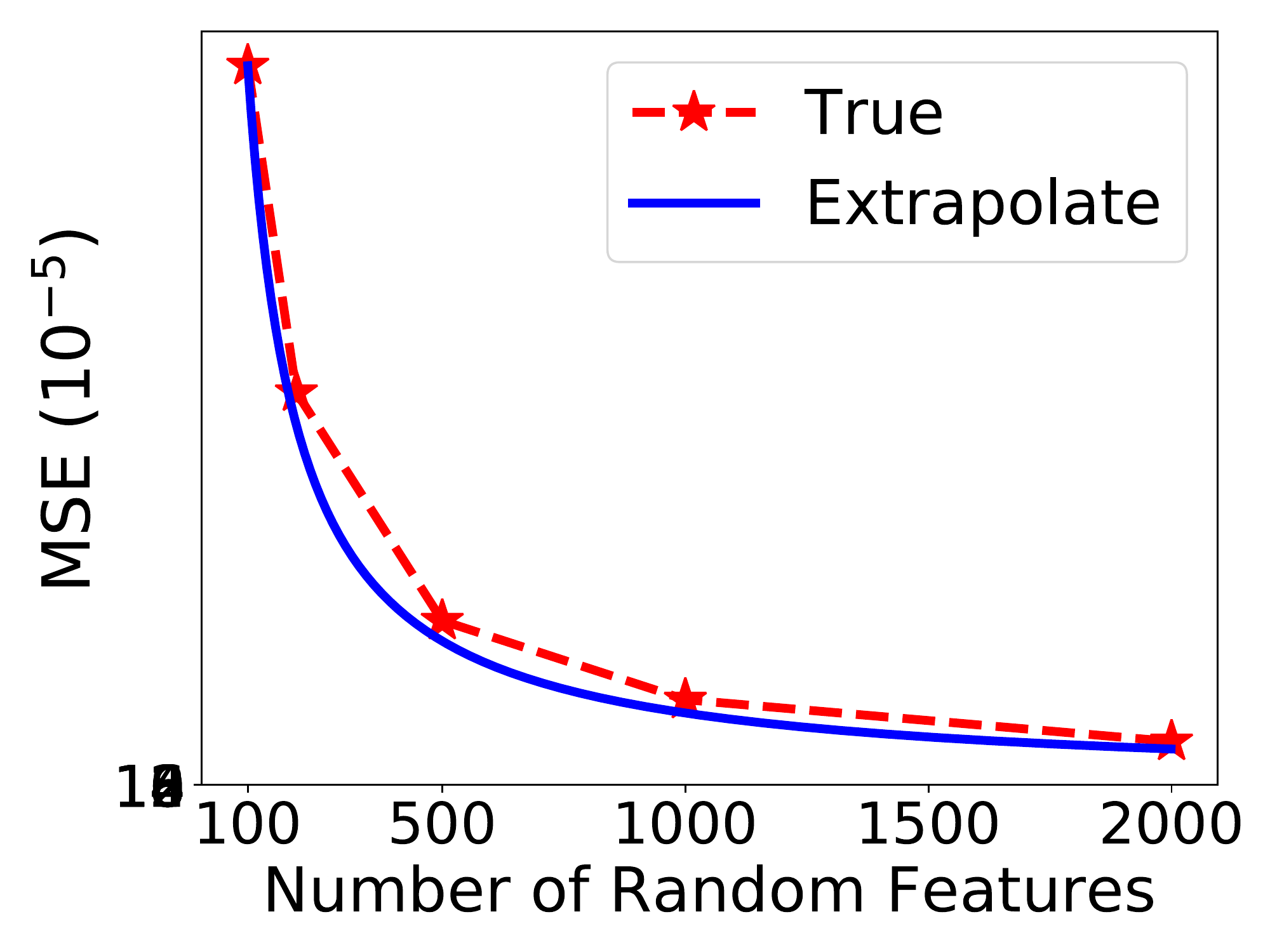}}
	\subfigure[Cadata, $\lambda = 0.2/\sqrt{n}$]{\includegraphics[width=0.3\textwidth]{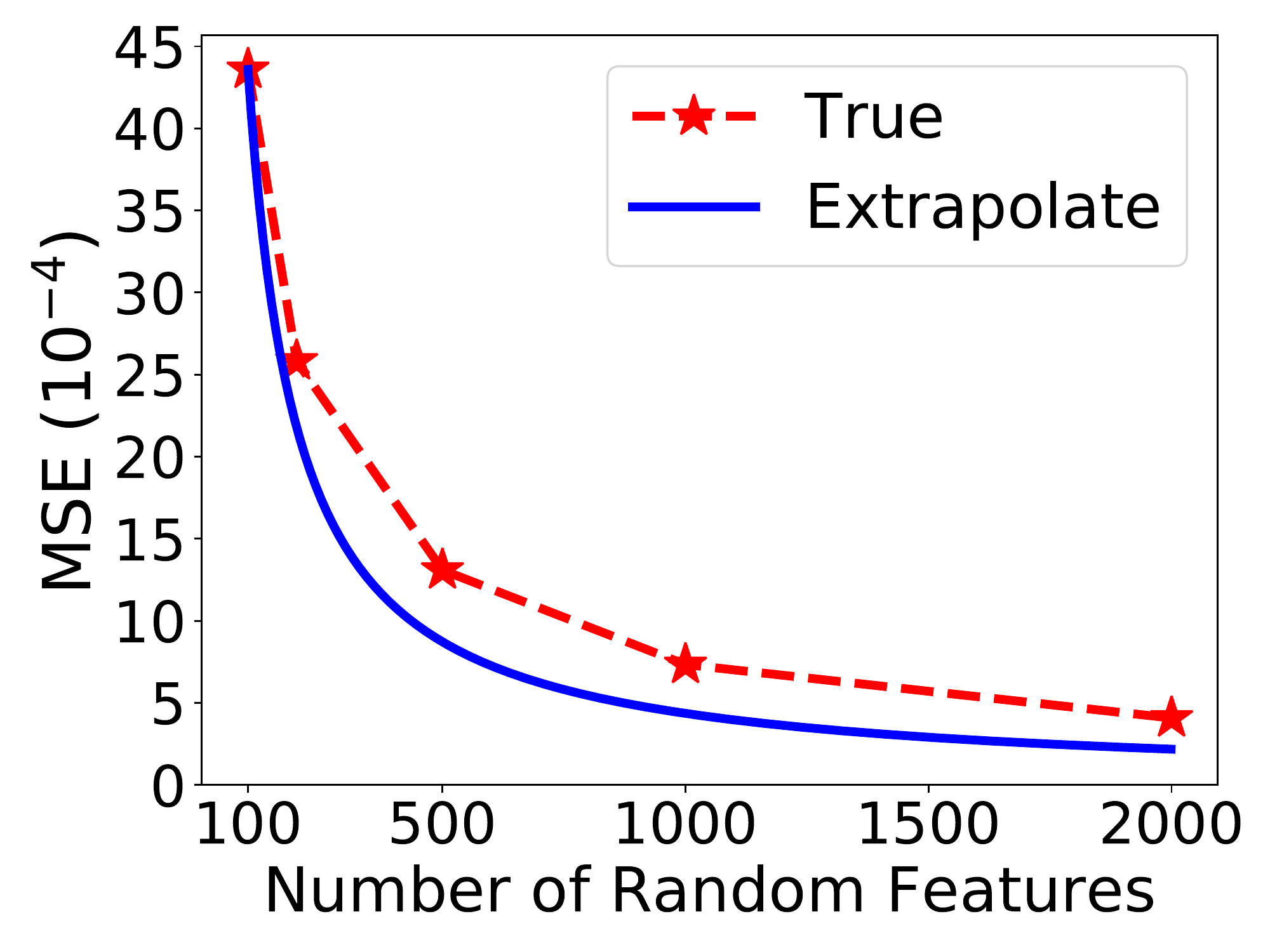}}
	\subfigure[Cadata, $\lambda = 1/\sqrt{n}$]{\includegraphics[width=0.3\textwidth]{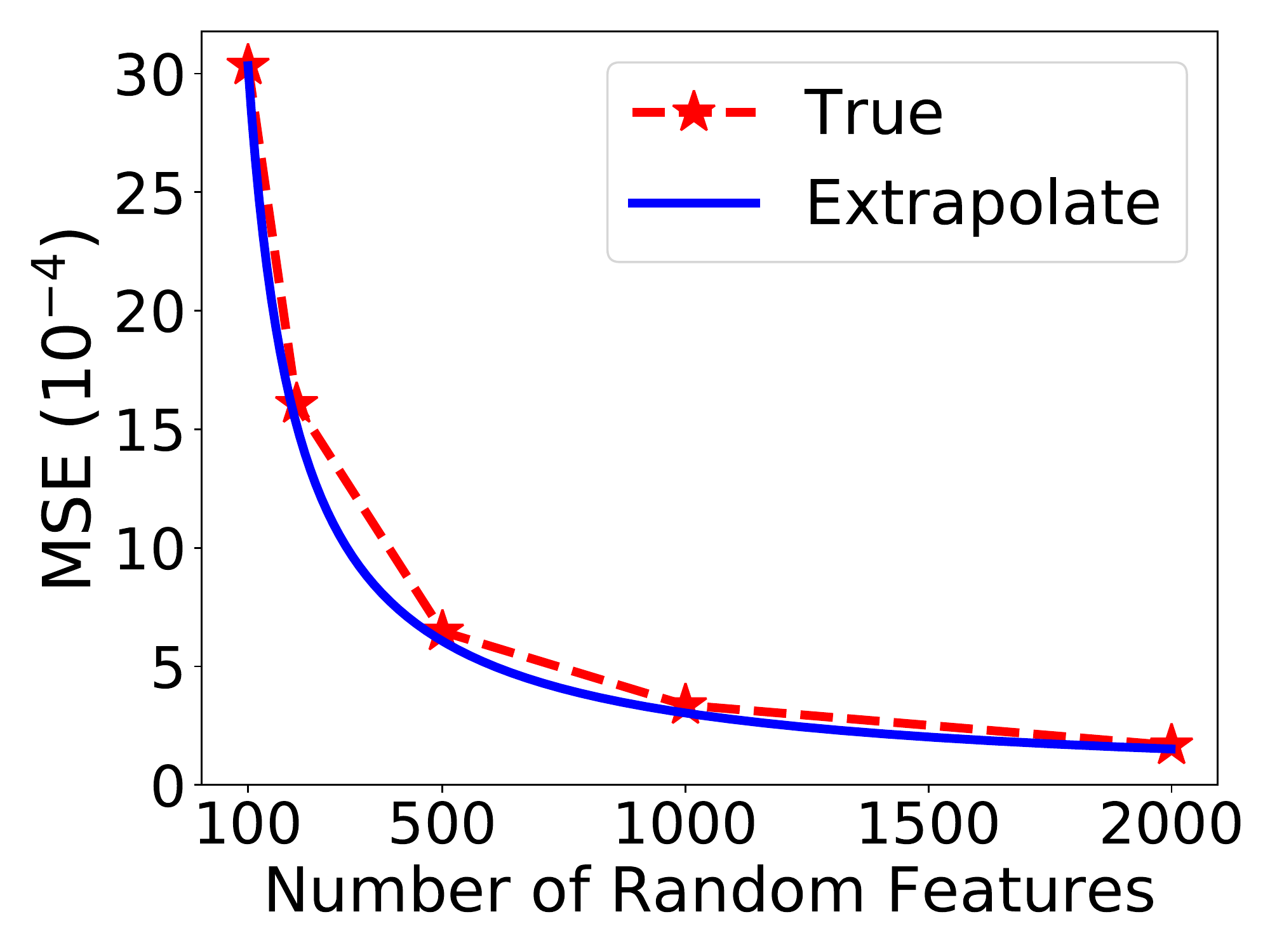}}
	\subfigure[Cadata, $\lambda = 5/\sqrt{n}$]{\includegraphics[width=0.3\textwidth]{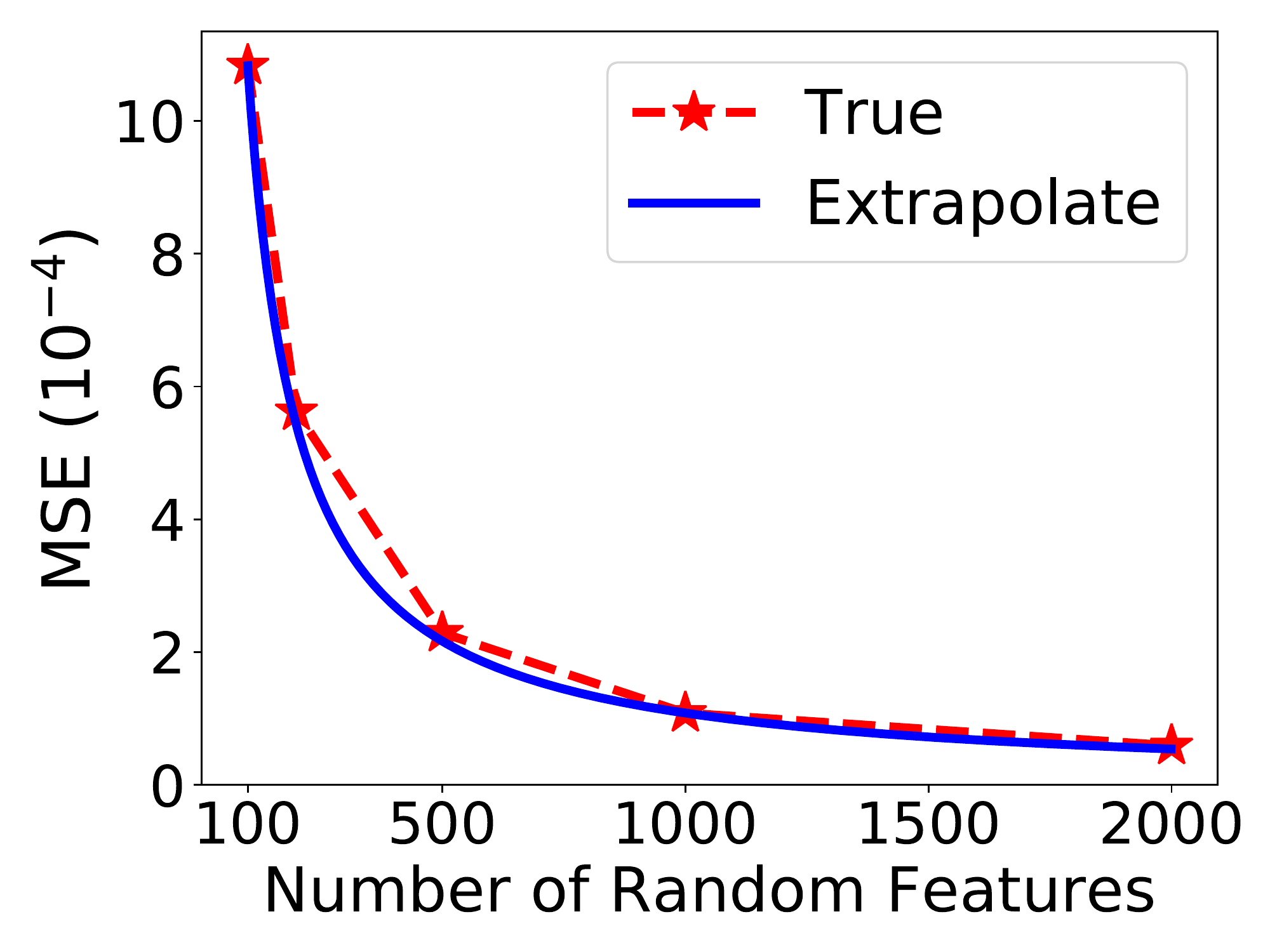}}
	\subfigure[Cpusmall, $\lambda = 0.2/\sqrt{n}$]{\includegraphics[width=0.3\textwidth]{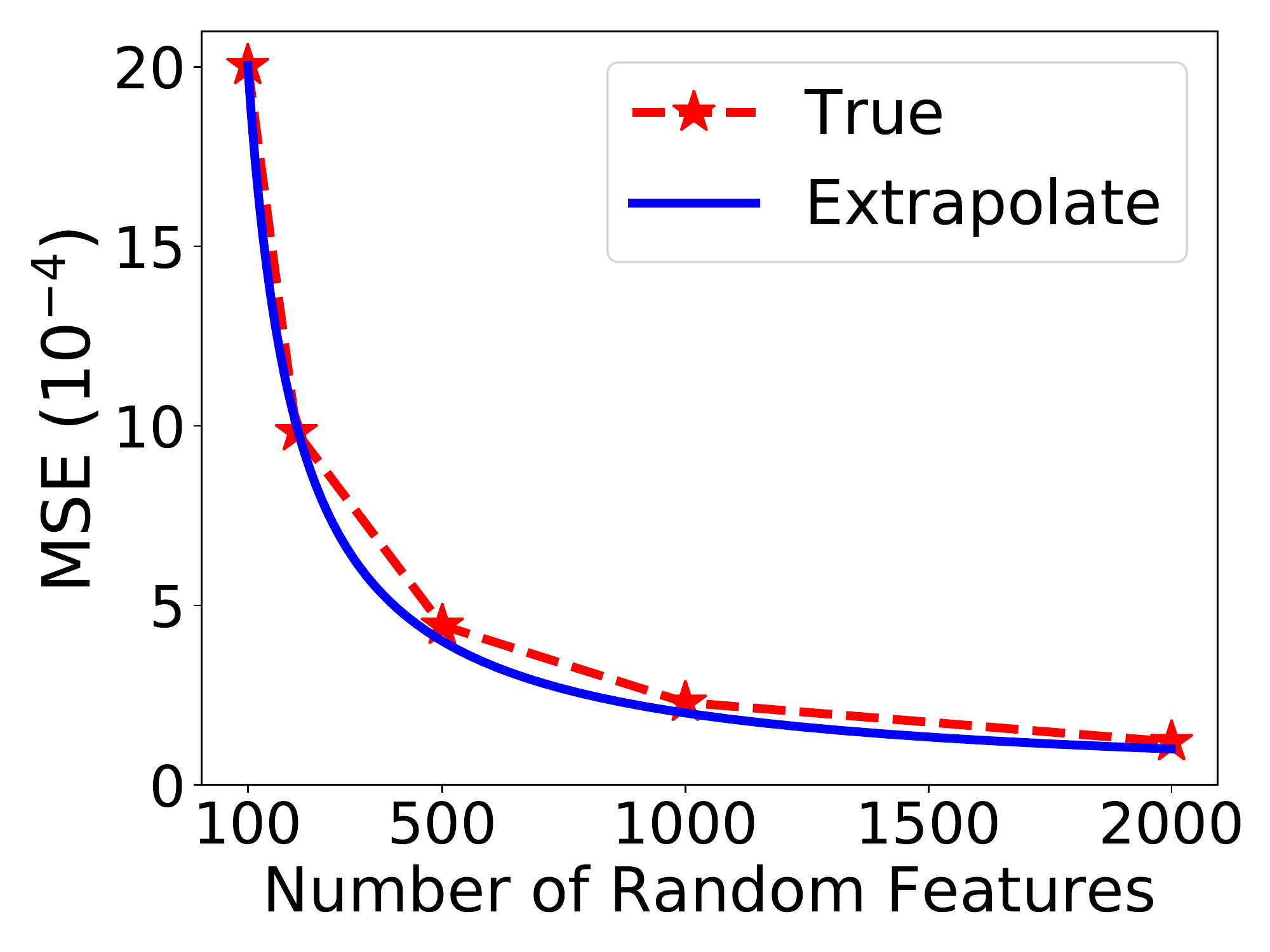}}
	\subfigure[Cpusmall, $\lambda = 1/\sqrt{n}$]{\includegraphics[width=0.3\textwidth]{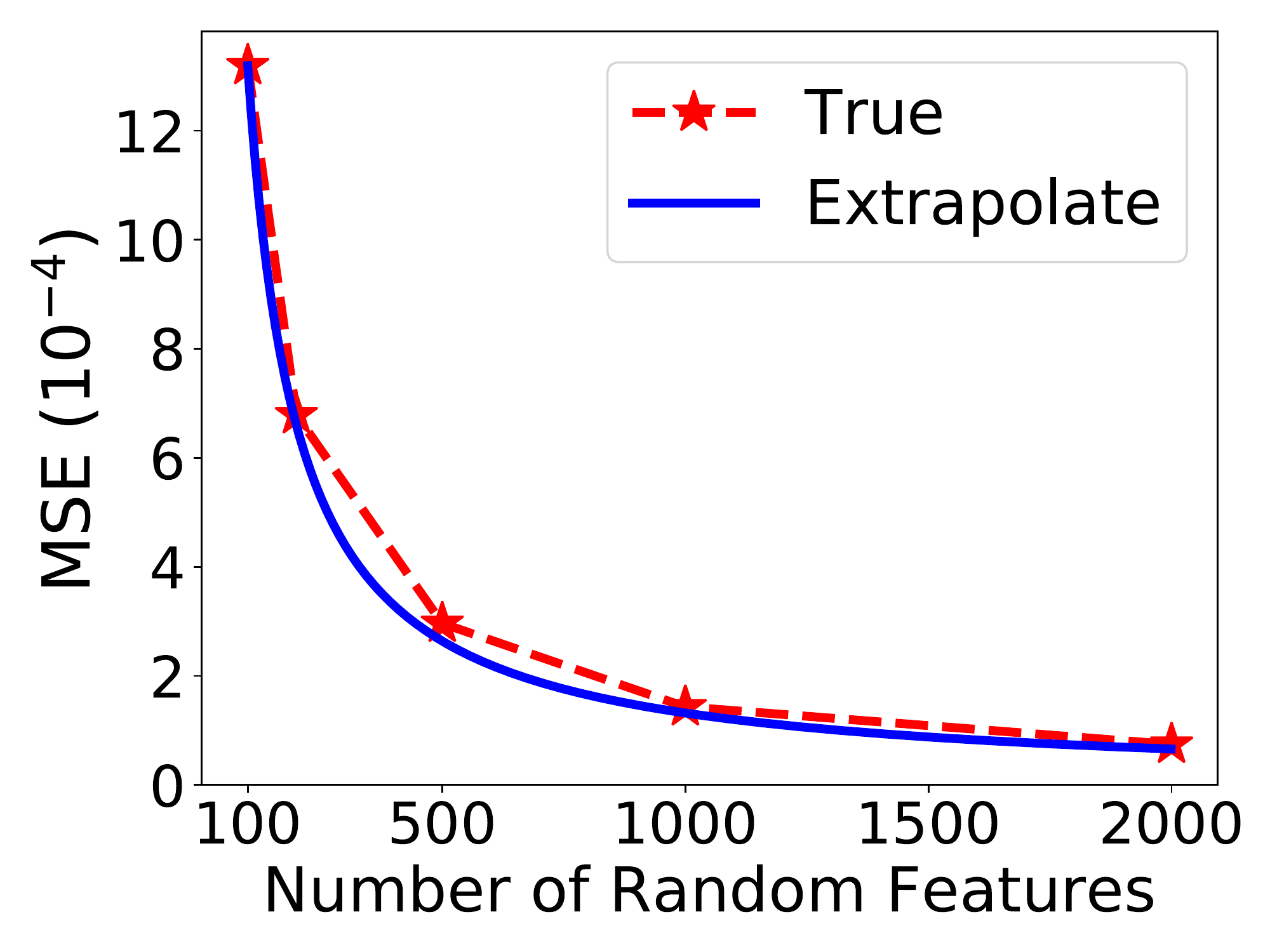}}
	\subfigure[Cpusmall, $\lambda = 5/\sqrt{n}$]{\includegraphics[width=0.3\textwidth]{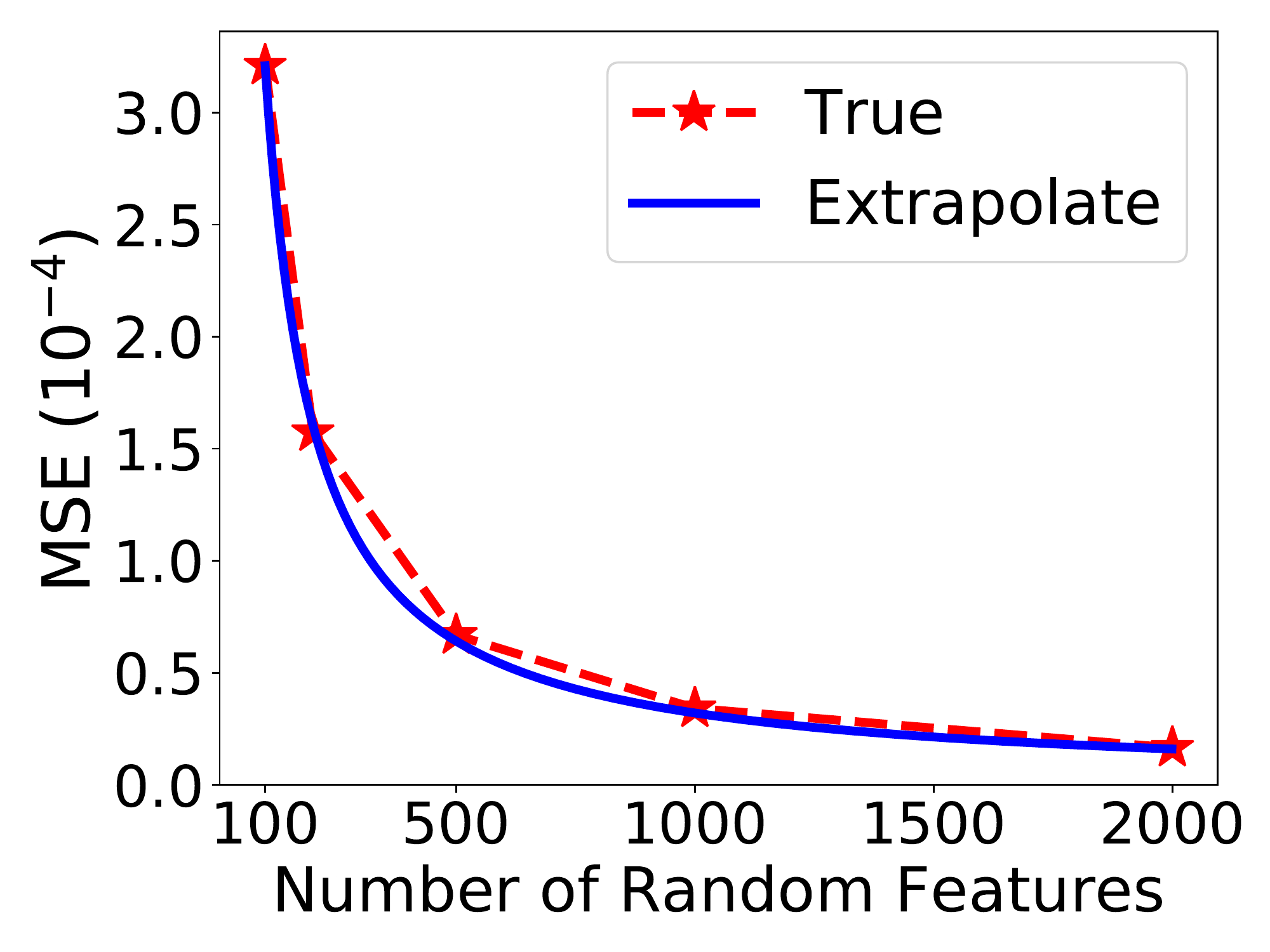}}
	\subfigure[Covtype, $\lambda = 0.2/\sqrt{n}$]{\includegraphics[width=0.3\textwidth]{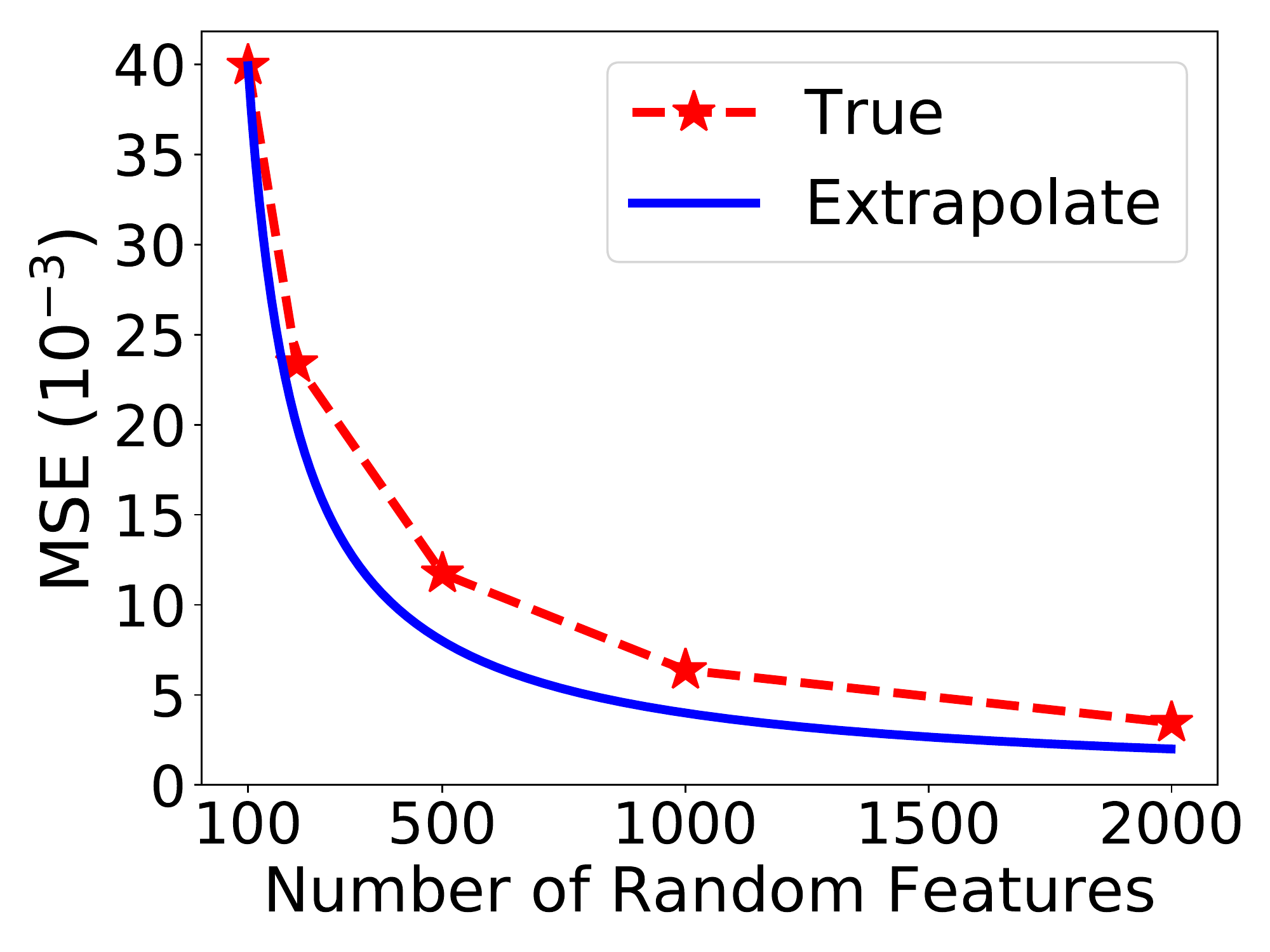}}
	\subfigure[Covtype, $\lambda = 1/\sqrt{n}$]{\includegraphics[width=0.3\textwidth]{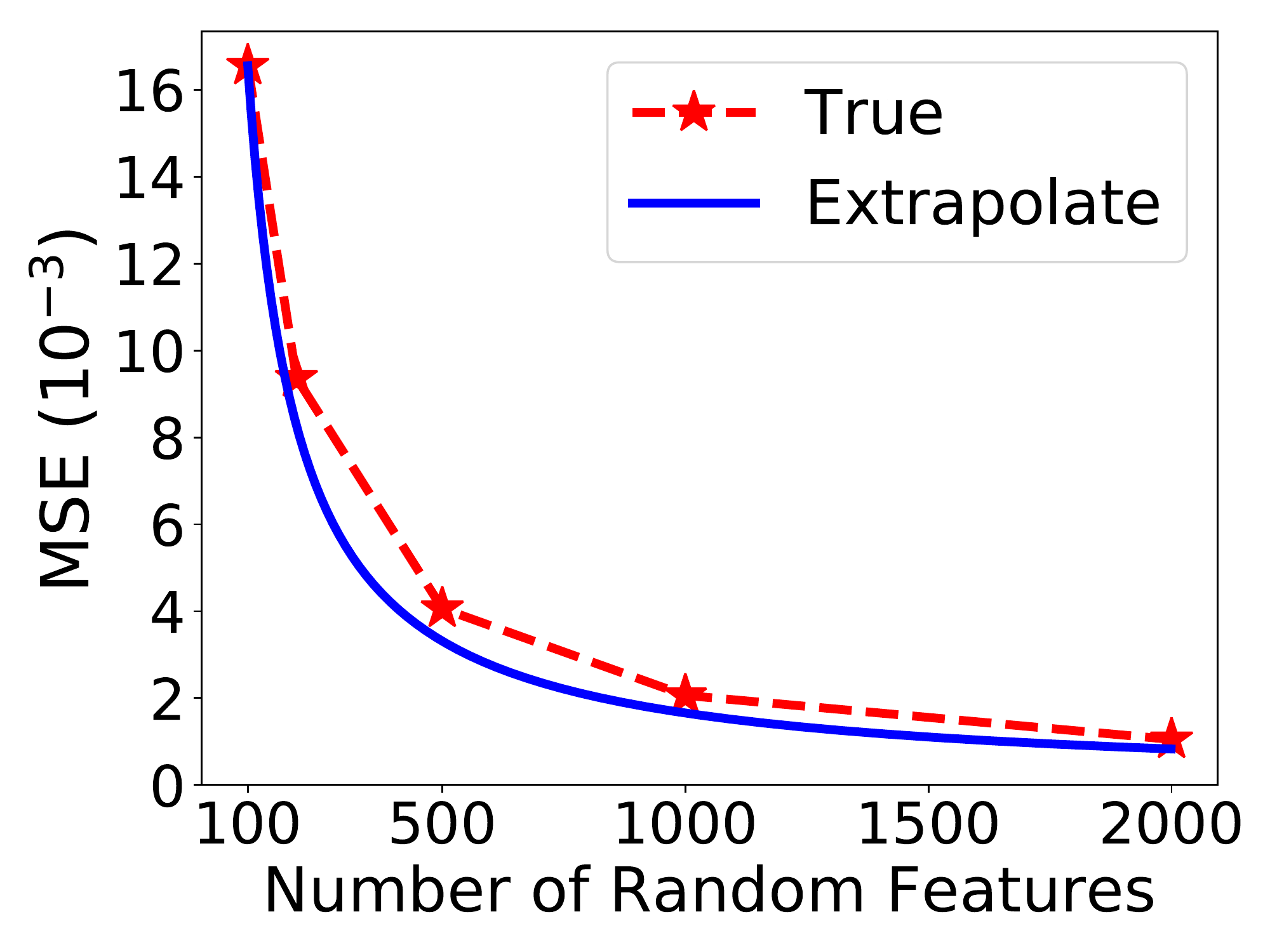}}
	\subfigure[Covtype, $\lambda = 5/\sqrt{n}$]{\includegraphics[width=0.3\textwidth]{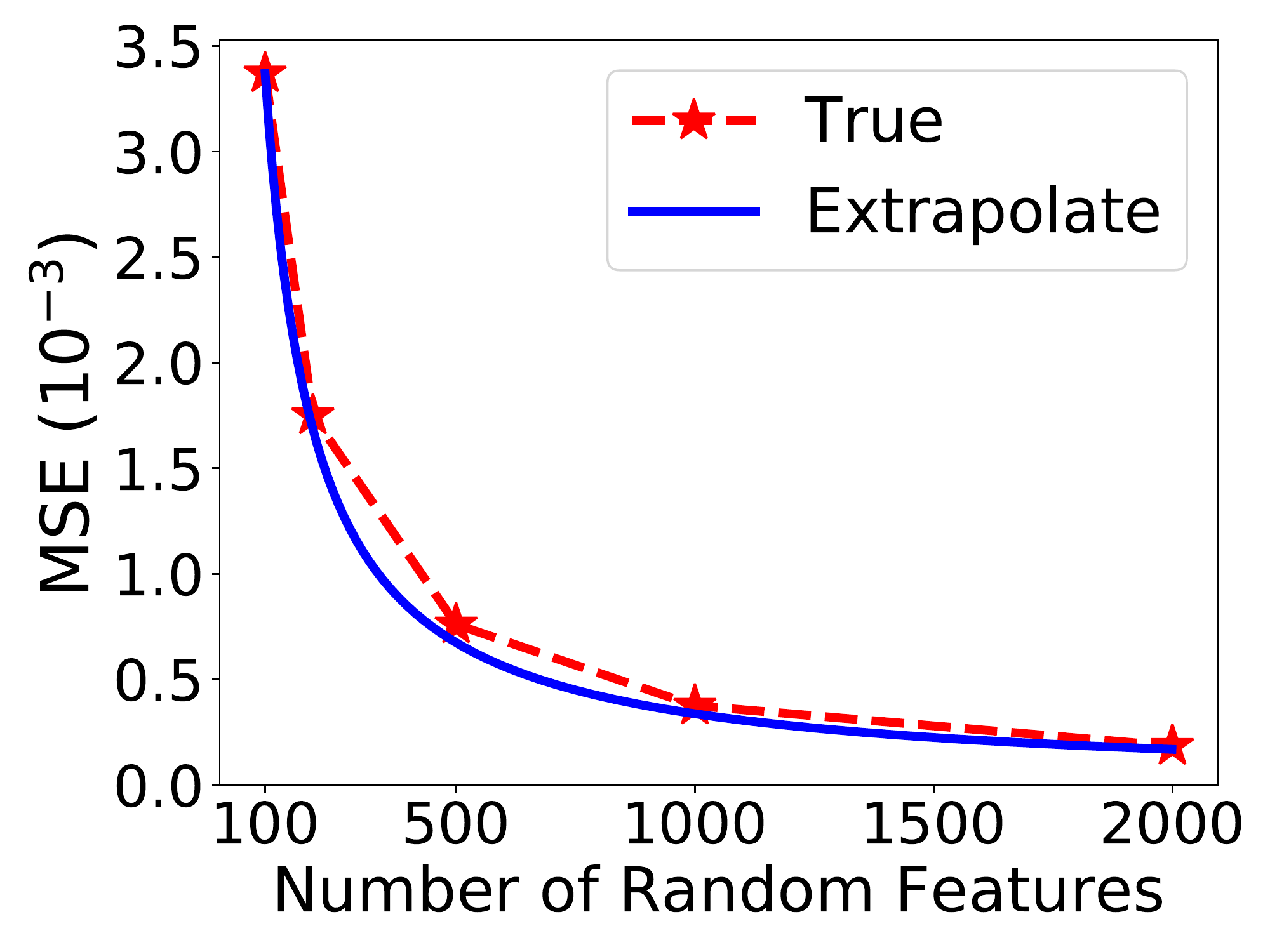}}
	\caption{Plot of the MSE $\EB \big[ ( f_{\lambda} - \tilde{f}_\lambda )^2 \big]$ against $s$ (using the Laplace kernel.)
		For MSD, Covtype, and Cadata, we use $10,000$ samples for training and $10,000$ for test.
		For Cpusmall, we use $5,000$ for training and $3,192$ for test.}
	\label{fig:laplace_s}
\end{figure}

\subsection{Settings} \label{sec:exp:setting}

We conduct experiments on the real-world data sets described in Table~\ref{tab:data}.
The data are openly available at \url{https://www.csie.ntu.edu.tw/~cjlin/libsvmtools/datasets/}.
We scale the input features to $[-1 , 1]$ using the min-max scaling.
We normalize the targets such that $\mean (\y) = 0$ and $\max (|\y| ) = 1$.

We use two types of kernels: the radial basis function (RBF) kernel $\kappa (\x, \x') = \exp ( - \frac{1}{2 \sigma^2} \| \x - \x' \|_2^2  )$ and the Laplace kernel $\kappa (\x, \x') = \exp ( - \frac{1}{ \sigma } \| \x - \x' \|_1  )$.
For the RBF kernel, we choose $\sigma$ based on the average interpoint distance in the data sets as
\begin{small}
\begin{eqnarray*}
\sigma 
& = & \sqrt{\textstyle \frac{1}{n^2} \sum_{i=1}^n \sum_{j=1}^n \|\x_i - \x_j\|_2^2 },
\end{eqnarray*}
\end{small}%
where $\x_1 , \cdots , \x_n$ denote the input data.
For the Laplace kernel, we set
\begin{small}
	\begin{eqnarray*}
		\sigma 
		& = & \textstyle \frac{1}{n^2} \sum_{i=1}^n \sum_{j=1}^n \|\x_i - \x_j\|_1 .
	\end{eqnarray*}
\end{small}%

The random features are generated in the following way.
For the RBF kernel, every entry of $\A$ ($d\times s$) is i.i.d.\ drawn from the standard normal distribution $\NM (0, 1)$.
For the Laplace kernel, every entry of $\A$ is i.i.d.\ drawn from the standard Cauchy distribution.
Then, every entry of $\bb $ ($s\times 1$) is independently and uniformly drawn from $[0, 2\pi]$.
Finally, the feature vector $\ps (\x; \VM_s) \in \RB^s$ defined \eqref{eq:def:rfm_vec} is computed by
\begin{equation*}
\ps (\x; \VM_s) \: = \: \tfrac{1}{ \sqrt{s}} \cos \big( \tfrac{1}{\sigma} \A \x + \bb  \big) ,
\end{equation*}
where $\cos (\cdot )$ is applied elementwisely.

\subsection{Verifying the $1/s$ rate} \label{sec:exp:rate}

Theorem~\ref{thm:main} shows that the mean squared error (MSE) $\EB \big[ ( f_{\lambda} - \tilde{f}_\lambda )^2 \big]$ converges to zero at a rate of $\frac{1}{s}$.
Here, $f_{\lambda}$ and $\tilde{f}_\lambda$ are respectively the out-of-sample prediction made by KRR and RFM-KRR.
We empirically verify the $\frac{1}{s}$ convergence rate by plotting the MSE against $s$.
Because KRR has $\OM (n^3)$ time complexity and $\OM (n^2)$ space complexity, we are not able to conduct large-scale experiments. 
If a dataset has more than $n=10,000$ data samples, we randomly select $10,000$ samples for training.
We use three settings of $\lambda$: $\lambda = \frac{1}{5 \sqrt{n}}$, $\frac{1}{\sqrt{n}}$, or $\frac{5}{ \sqrt{n}}$.

Figures~\ref{fig:rbf_s} and \ref{fig:laplace_s} are obtained using the RBF and Laplace kernels, respectively.
In the plots, the red stars are the actual MSEs based on $100$ repeats of the random feature mappings.
The blue lines are our extrapolations starting from the first red star and applying the $\frac{1}{s}$ rule.
Under different settings of $\lambda$, the extrapolations perfectly matches the actual MSEs, which verifies the $\frac{1}{s}$ rule in Theorem~\ref{thm:main}.

Figures~\ref{fig:rbf_s} and \ref{fig:laplace_s} show that big $\lambda$ leads to small MSE, which also corroborates our theories. 
The plots in the left columns correspond to $\lambda = \frac{5}{\sqrt{n}}$, and the MSEs in these plots are smaller than those in the middle and left.
Theorem~\ref{thm:main} shows that the MSE is proportional to 
\begin{equation*}
\big\| \K^{\frac{1}{2}} (\K + n \lambda \I_n )^{-1} \y \big\|_2^2 ,
\end{equation*}
which decreases as $\lambda$ increases.

\begin{figure}[!t]
	\centering
	\subfigure[MSD, $\lambda = 0.2/\sqrt{n}$]{\includegraphics[width=0.3\textwidth]{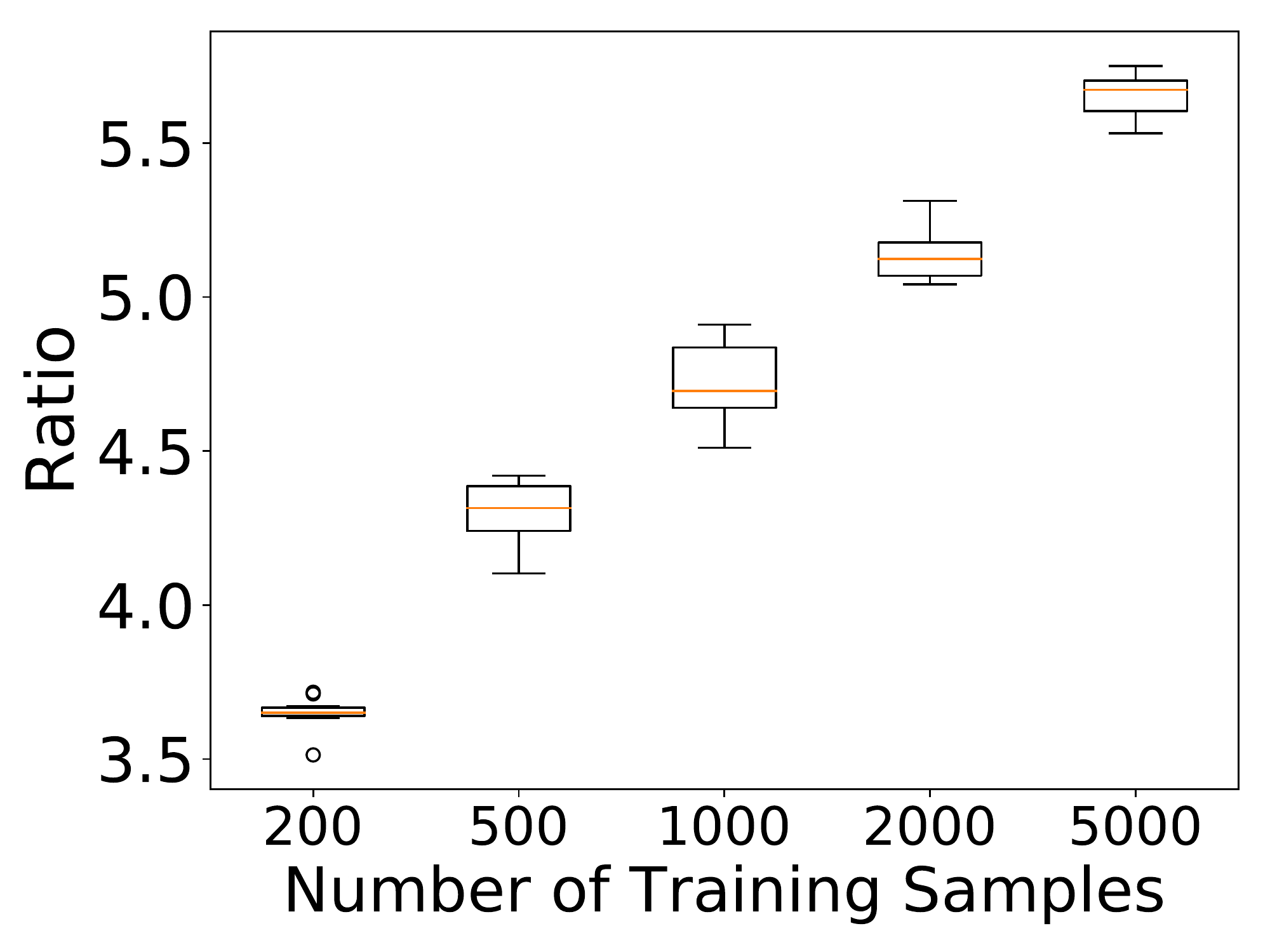}}
	\subfigure[MSD, $\lambda = 1/\sqrt{n}$]{\includegraphics[width=0.3\textwidth]{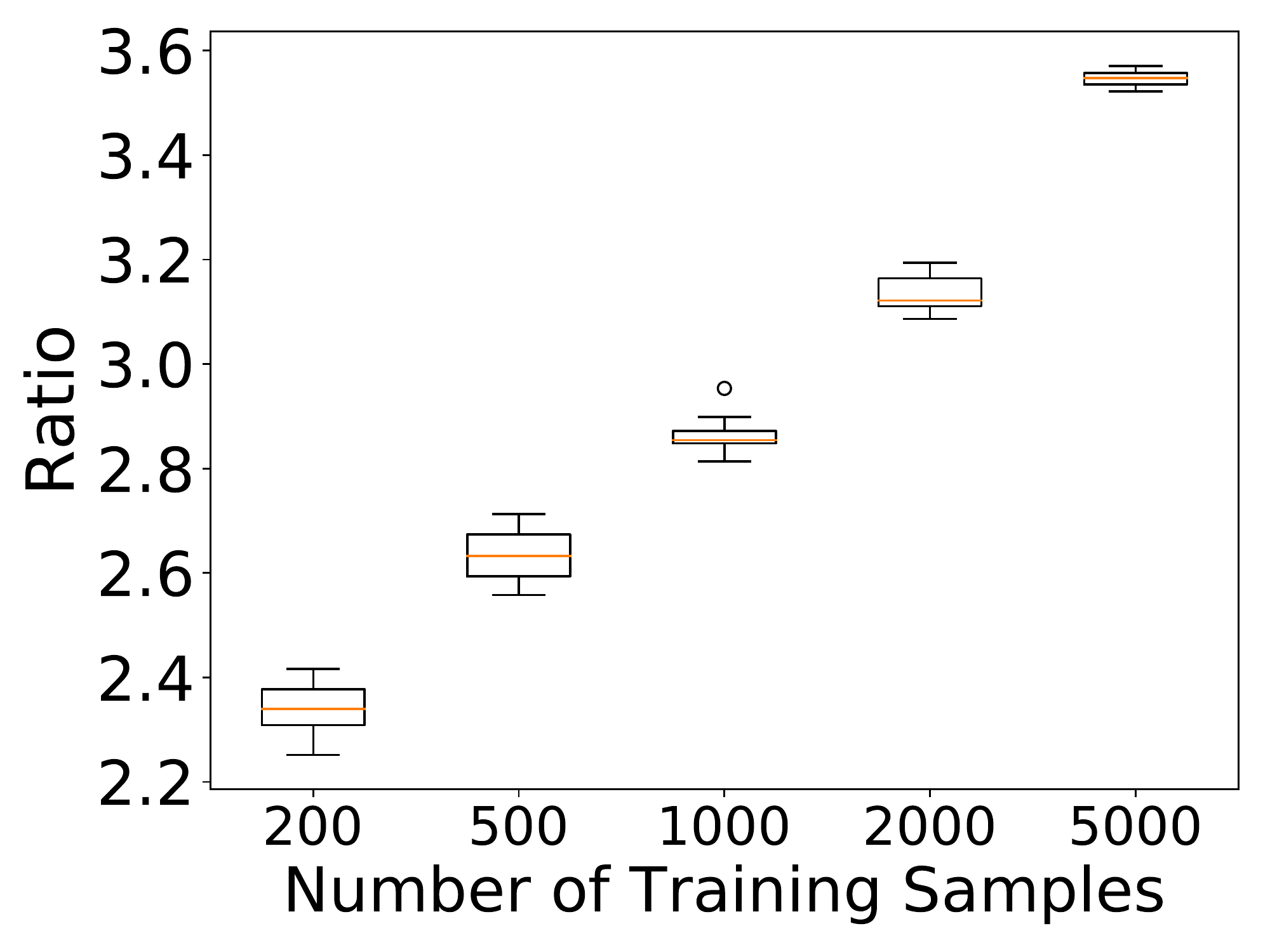}}
	\subfigure[MSD, $\lambda = 5/\sqrt{n}$]{\includegraphics[width=0.3\textwidth]{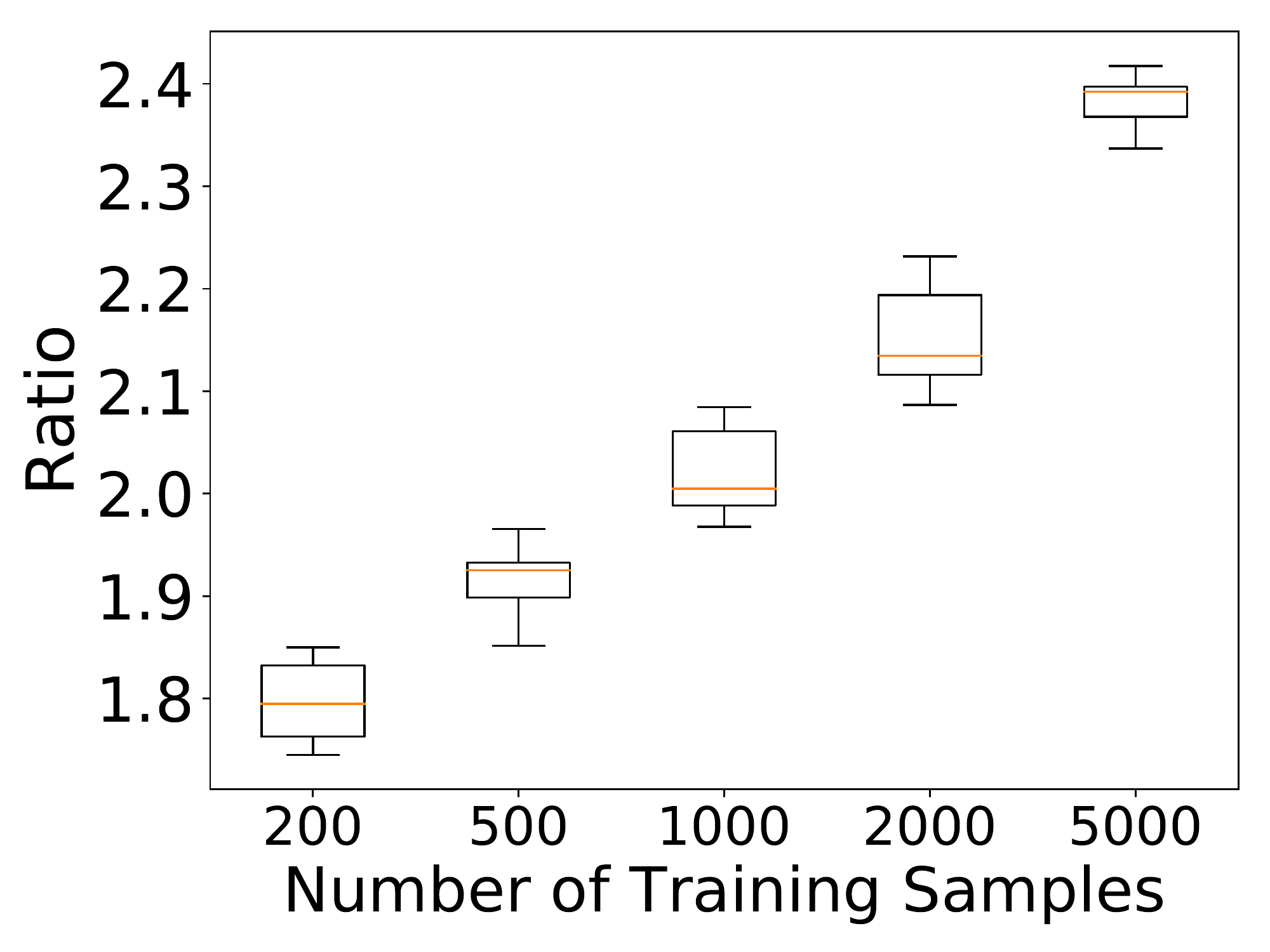}}
	\subfigure[Cadata, $\lambda = 0.2/\sqrt{n}$]{\includegraphics[width=0.3\textwidth]{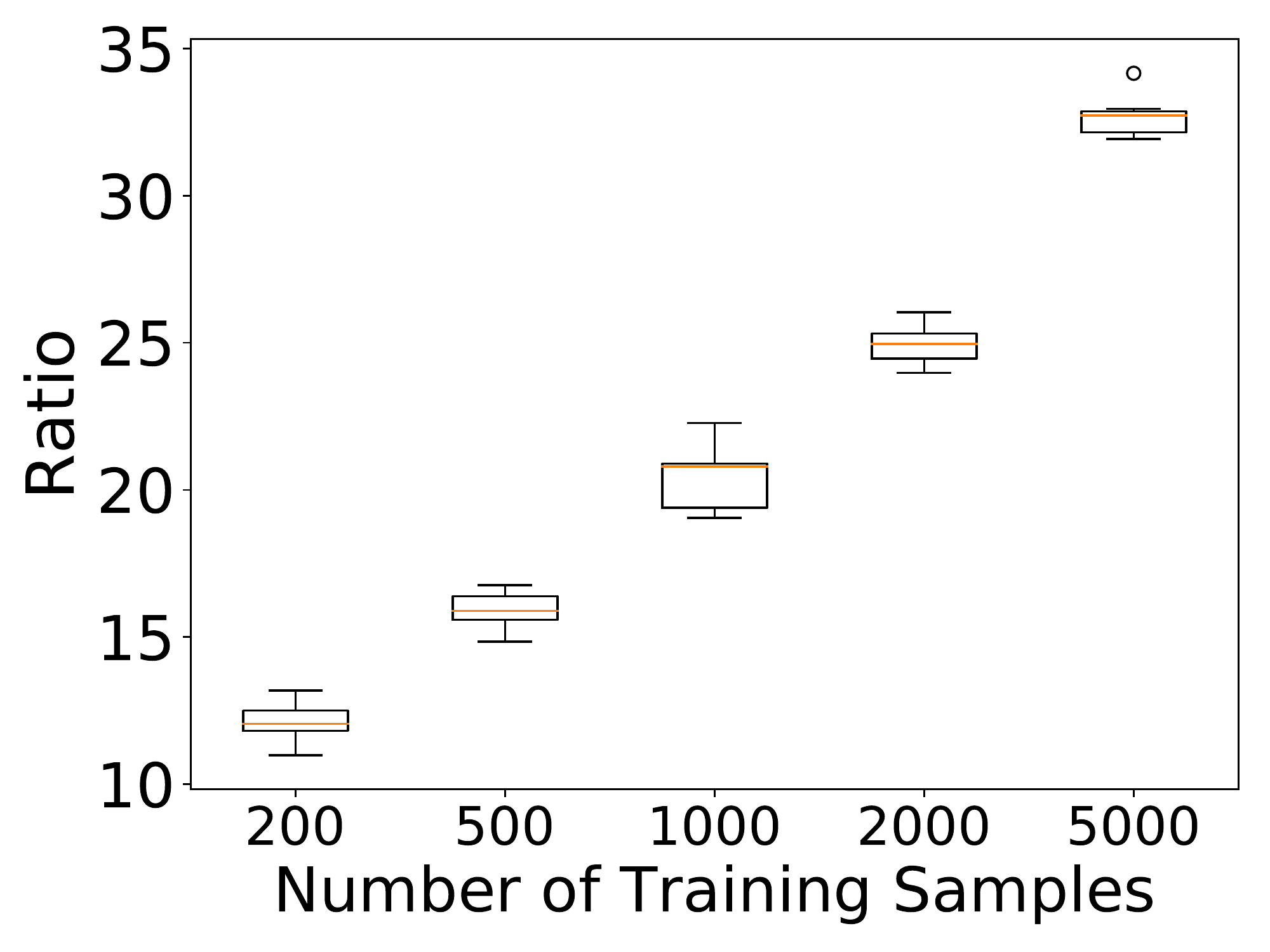}}
	\subfigure[Cadata, $\lambda = 1/\sqrt{n}$]{\includegraphics[width=0.3\textwidth]{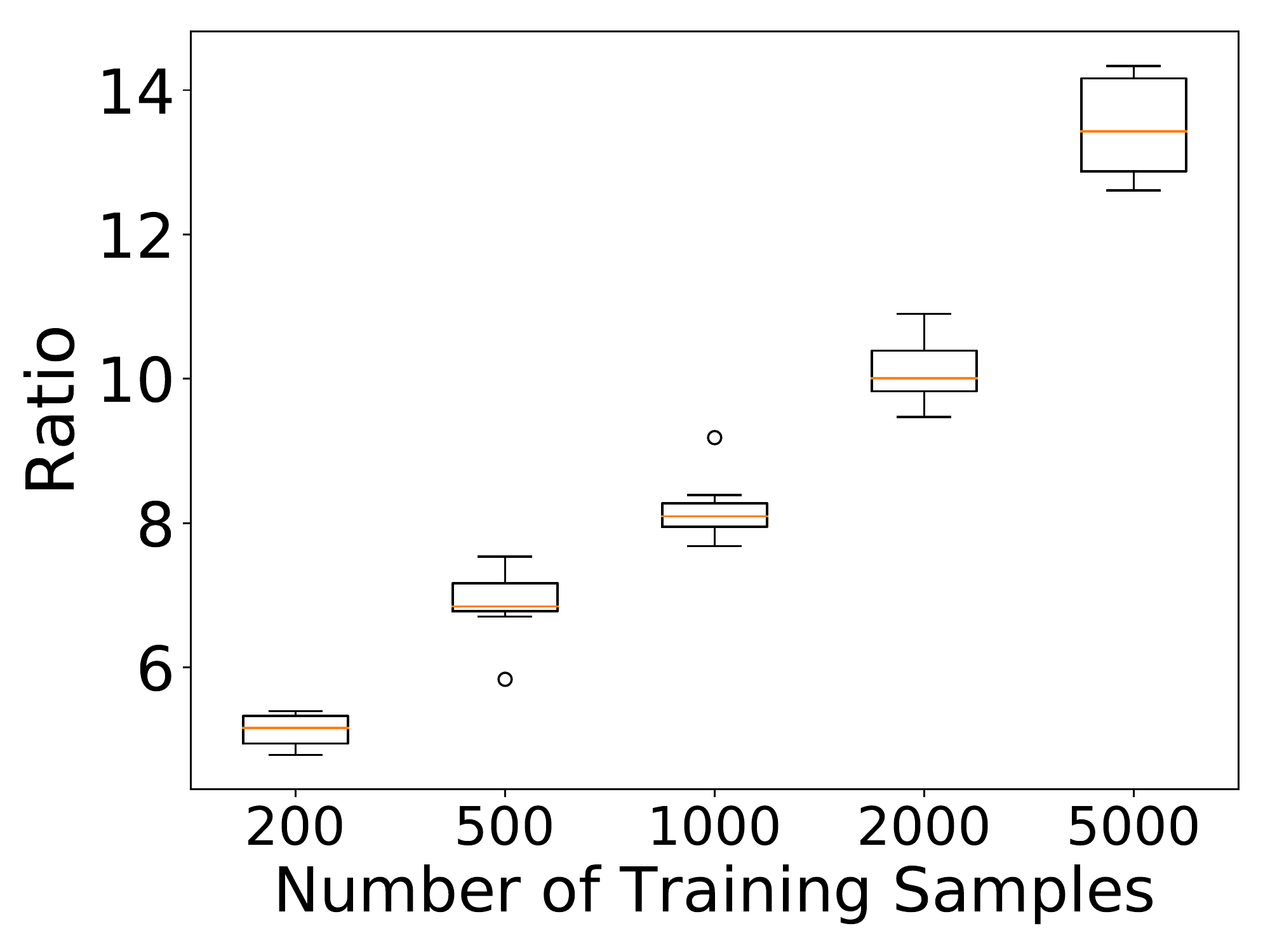}}
	\subfigure[Cadata, $\lambda = 5/\sqrt{n}$]{\includegraphics[width=0.3\textwidth]{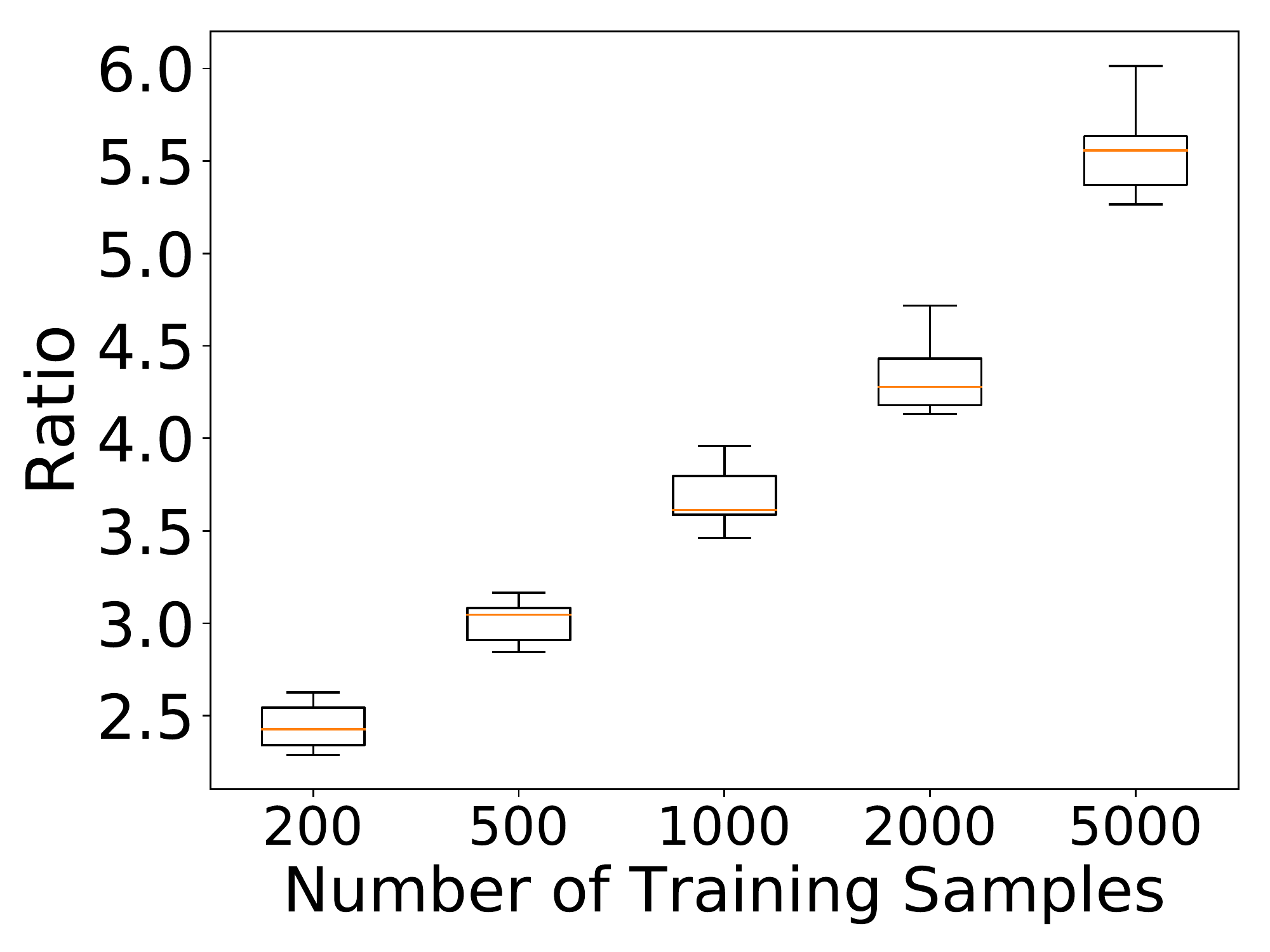}}
	\subfigure[Cpusmall, $\lambda = 0.2/\sqrt{n}$]{\includegraphics[width=0.3\textwidth]{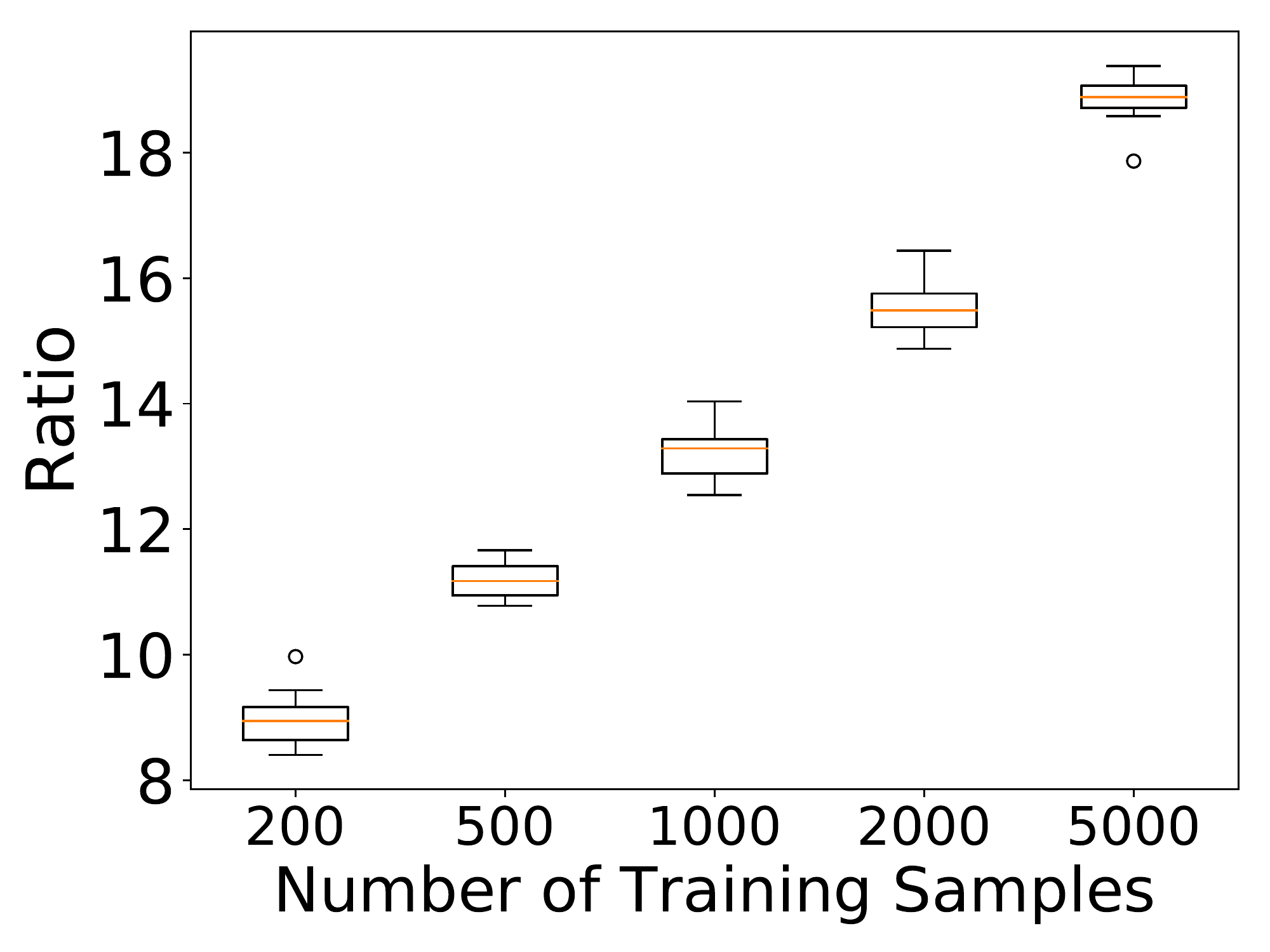}}
	\subfigure[Cpusmall, $\lambda = 1/\sqrt{n}$]{\includegraphics[width=0.3\textwidth]{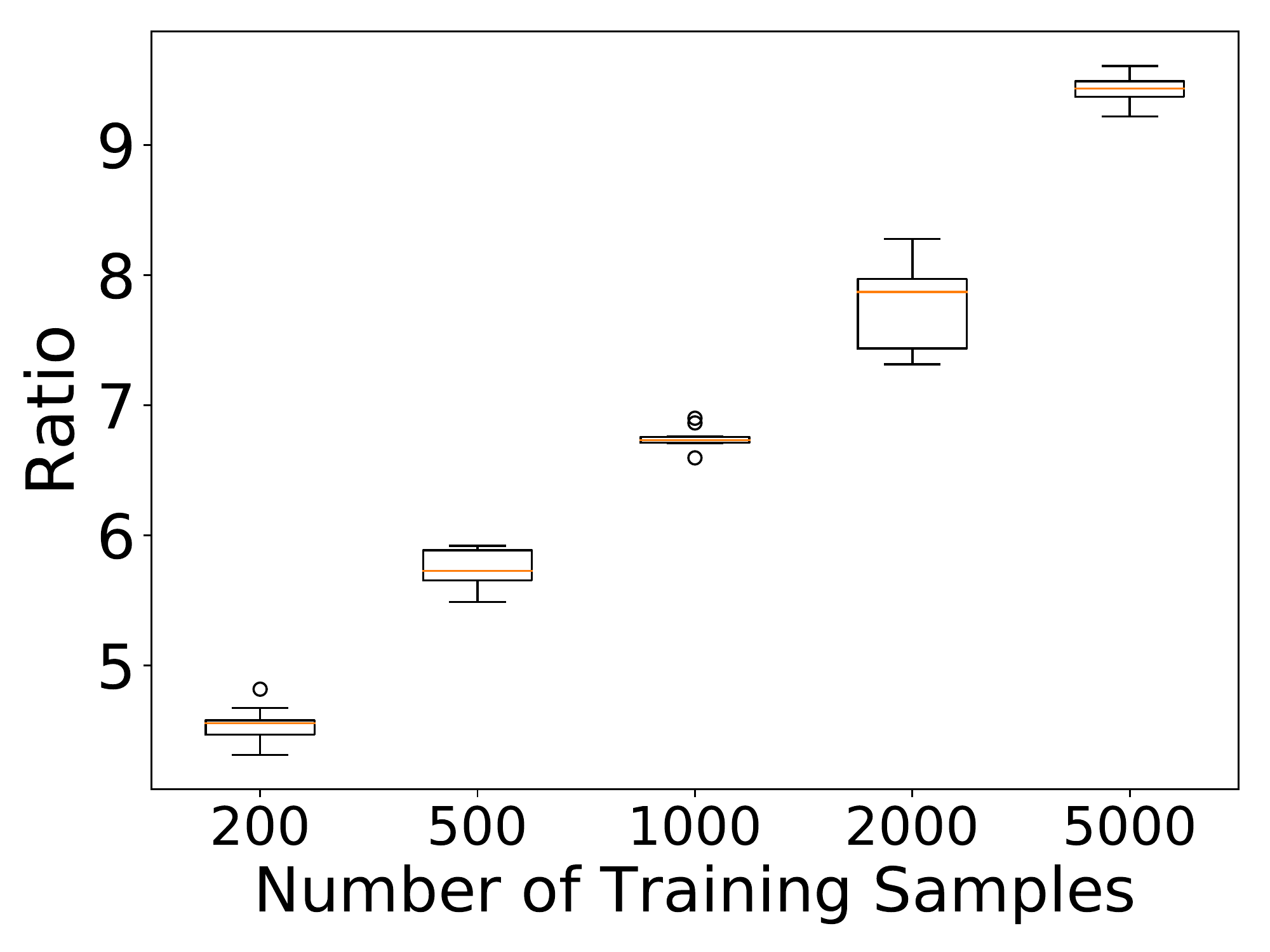}}
	\subfigure[Cpusmall, $\lambda = 5/\sqrt{n}$]{\includegraphics[width=0.3\textwidth]{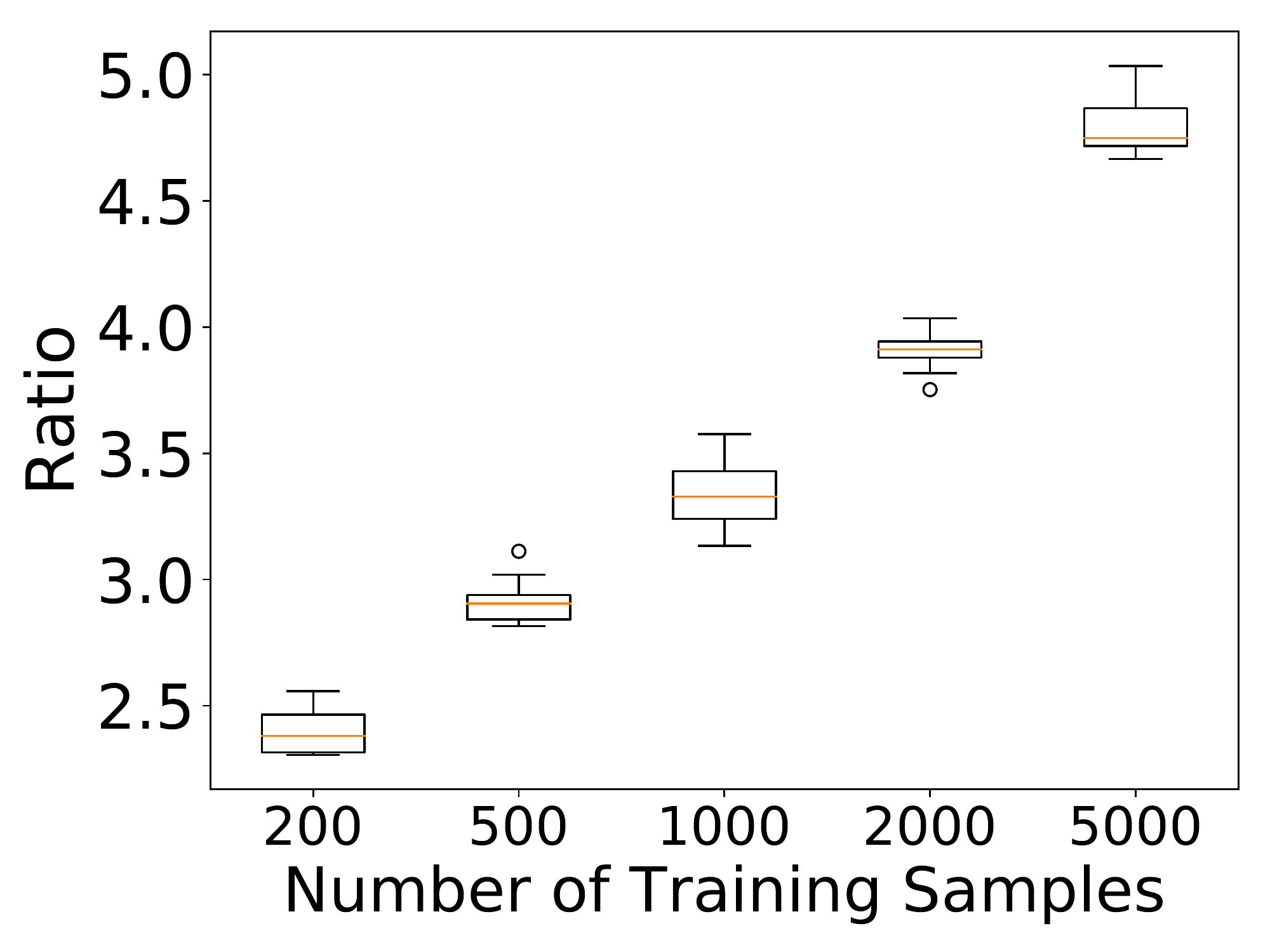}}
	\subfigure[Covtype, $\lambda = 0.2/\sqrt{n}$]{\includegraphics[width=0.3\textwidth]{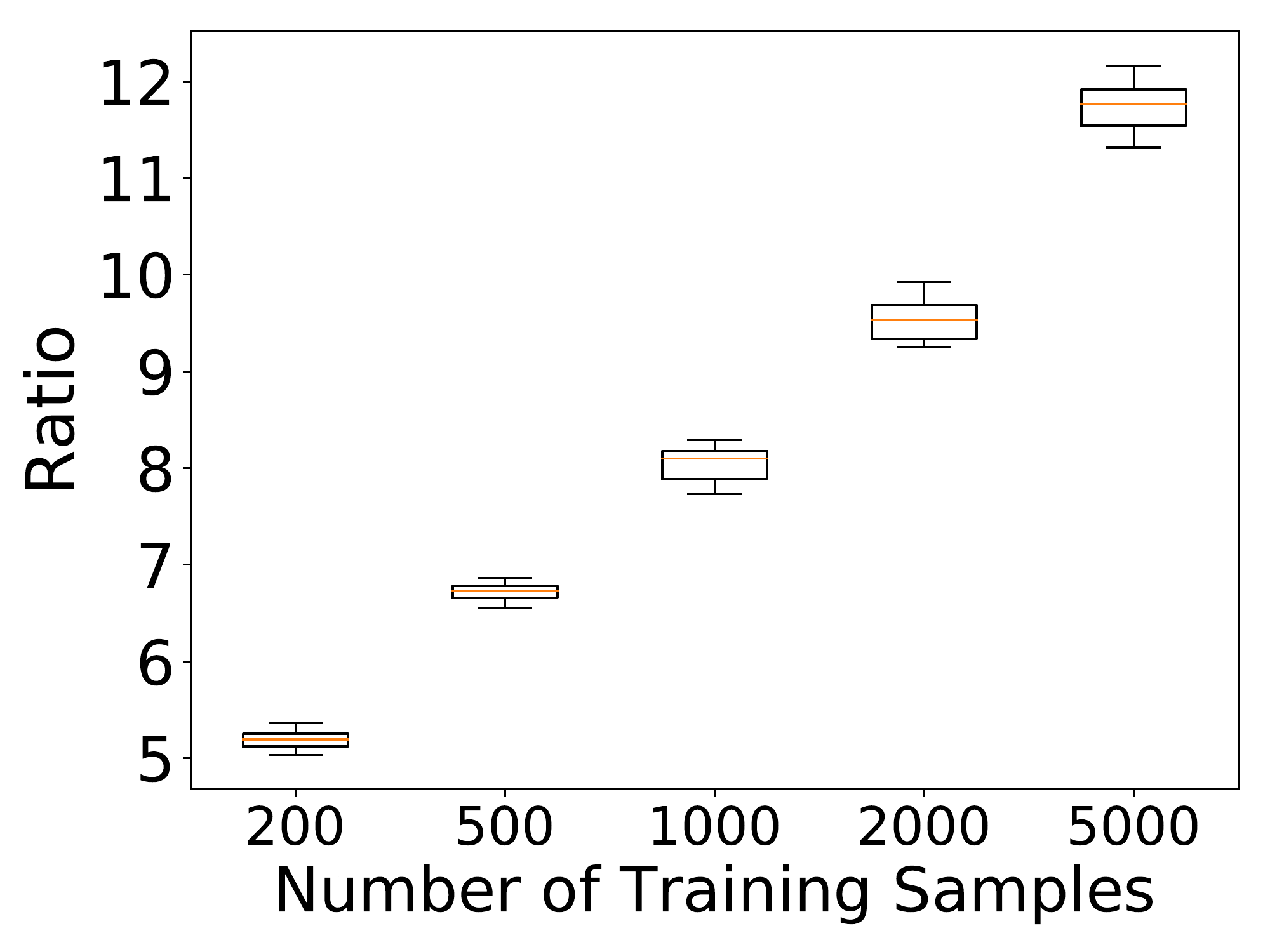}}
	\subfigure[Covtype, $\lambda = 1/\sqrt{n}$]{\includegraphics[width=0.3\textwidth]{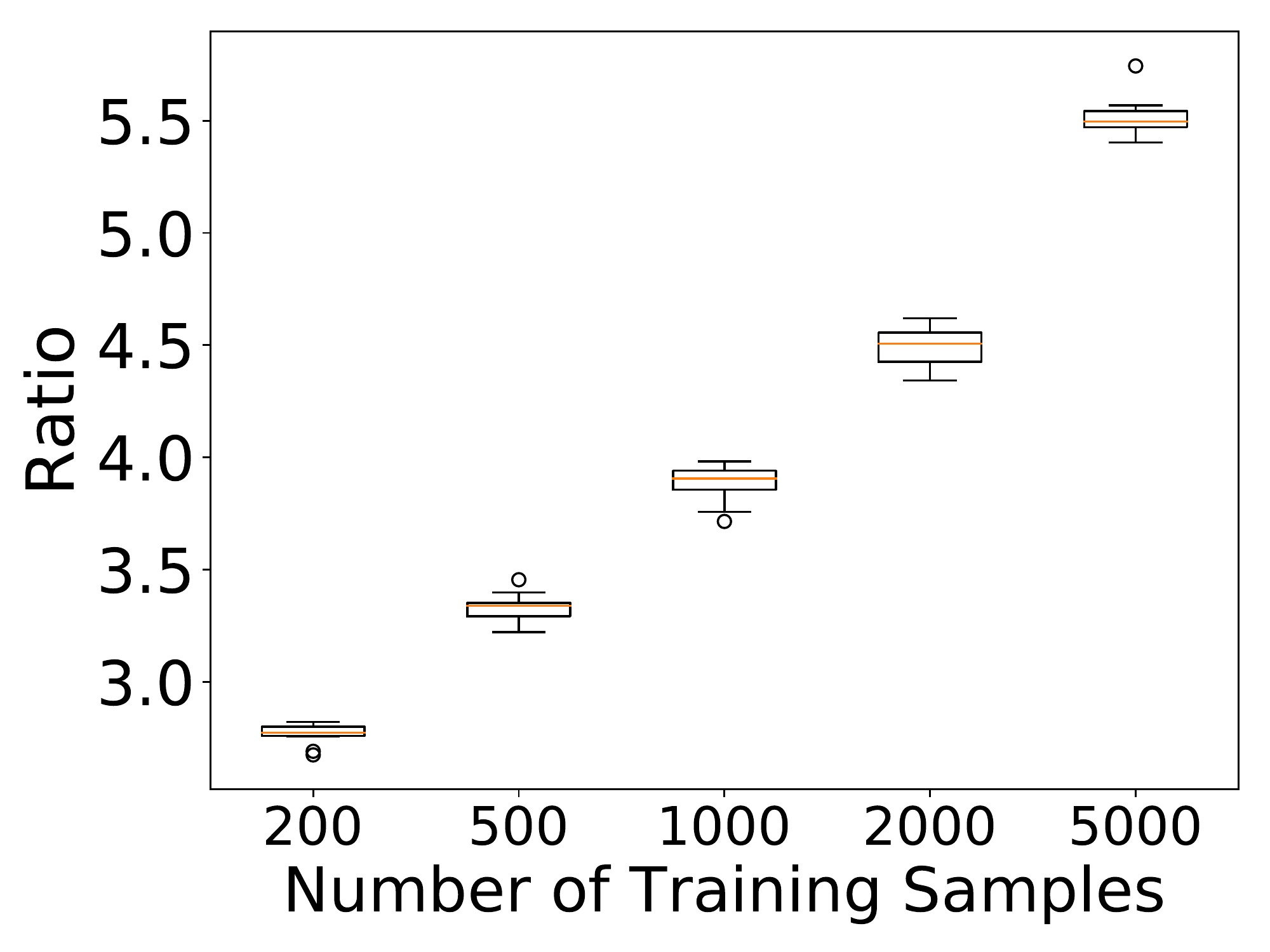}}
	\subfigure[Covtype, $\lambda = 5/\sqrt{n}$]{\includegraphics[width=0.3\textwidth]{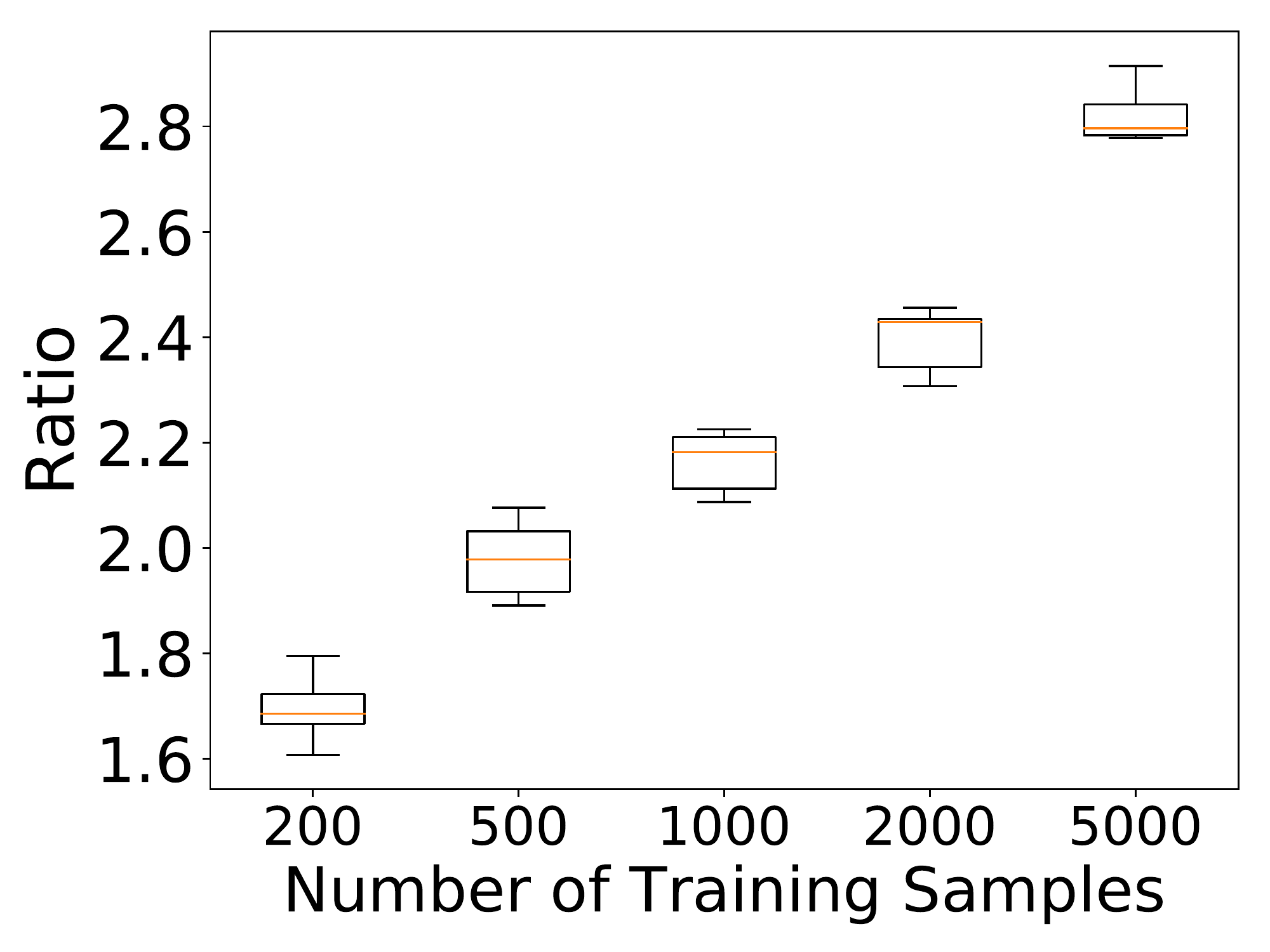}}
	\caption{Plot of the ratio $\frac{\textrm{Bound} }{ \textrm{MSE} } $ against $n$ (using the RBF kernel.)
		We fix $s = 100$.
		}
	\label{fig:rbf_n}
\end{figure}

\begin{figure}[!t]
	\centering
	\subfigure[MSD, $\lambda = 0.2/\sqrt{n}$]{\includegraphics[width=0.3\textwidth]{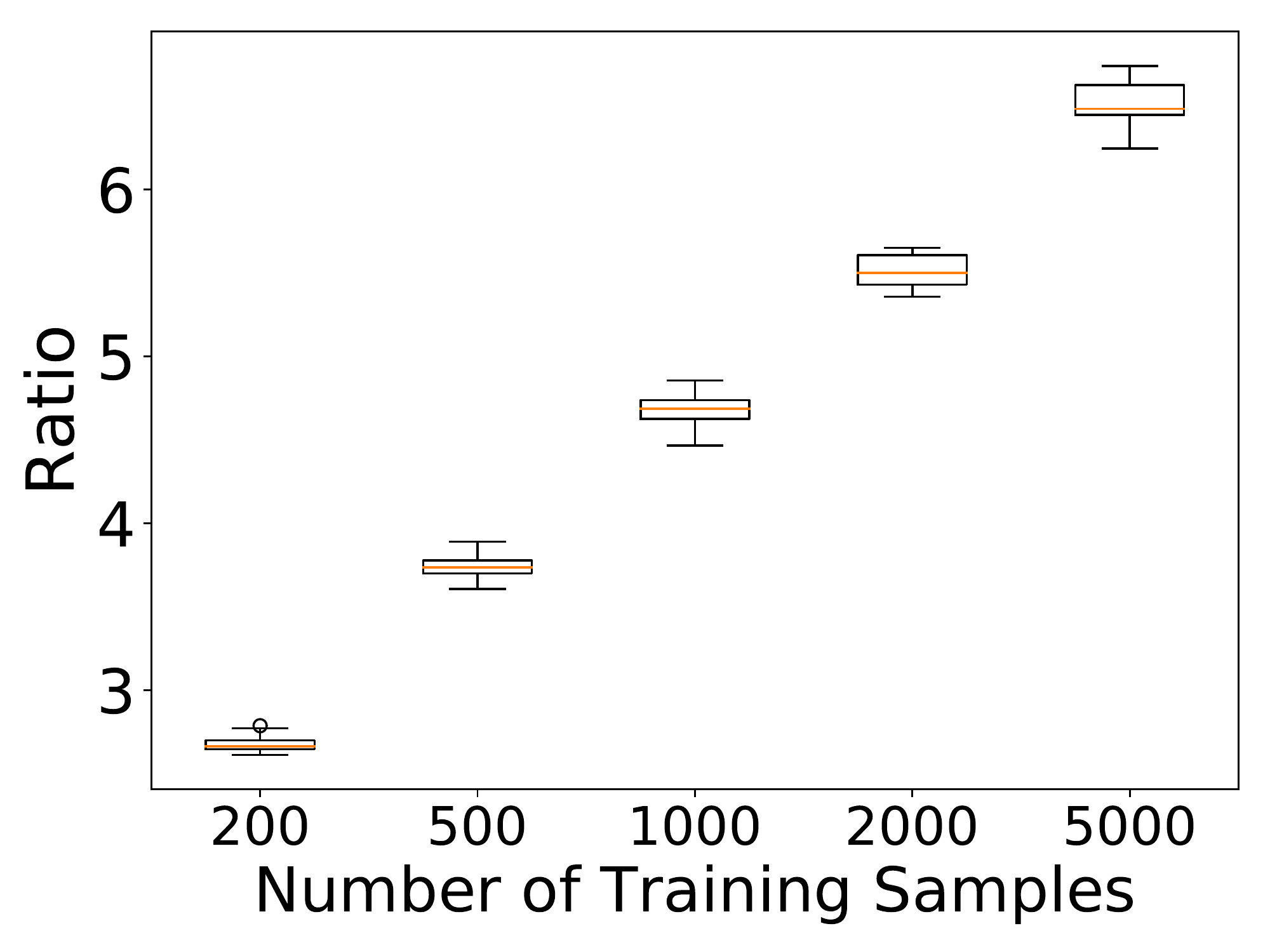}}
	\subfigure[MSD, $\lambda = 1/\sqrt{n}$]{\includegraphics[width=0.3\textwidth]{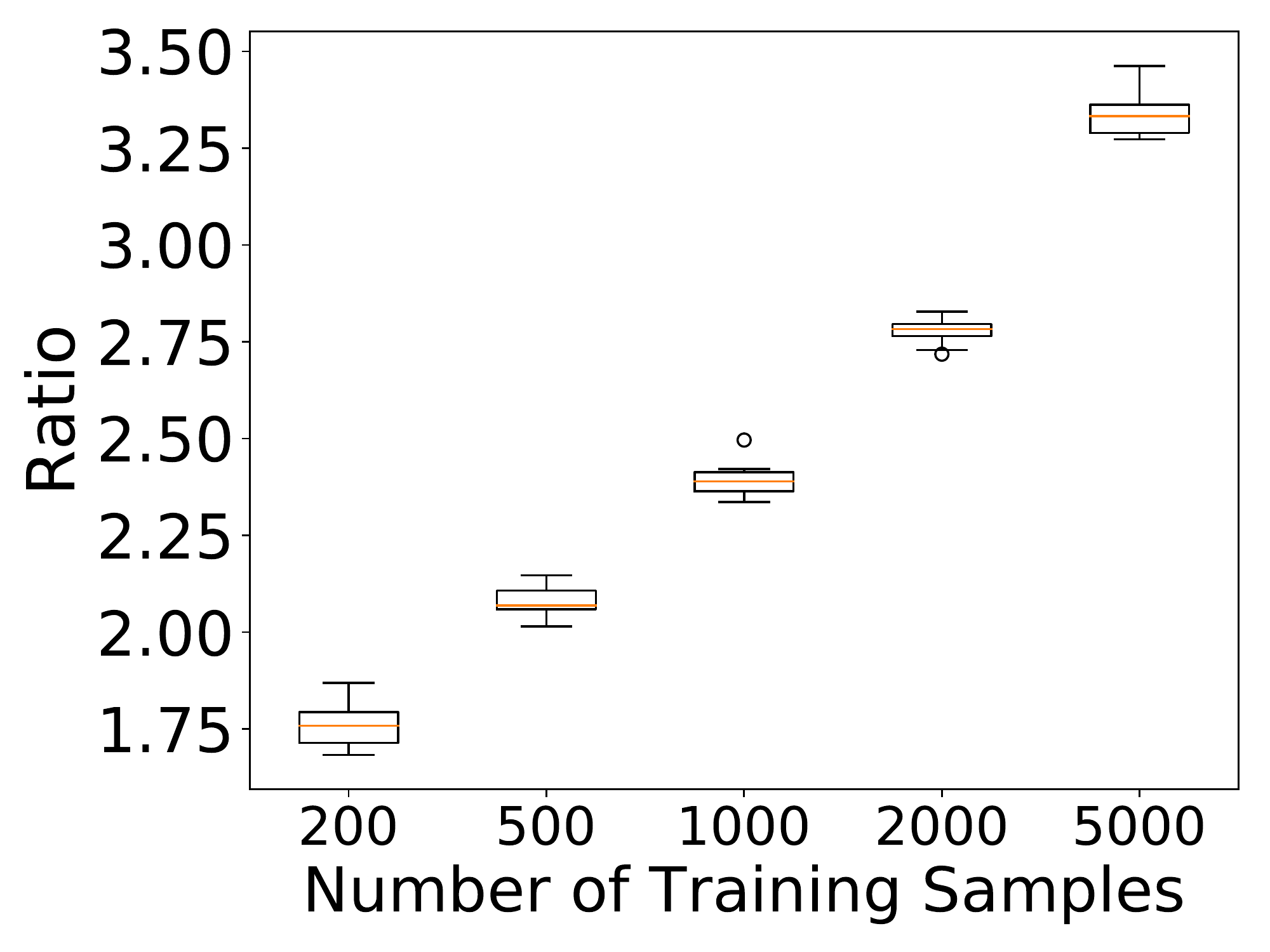}}
	\subfigure[MSD, $\lambda = 5/\sqrt{n}$]{\includegraphics[width=0.3\textwidth]{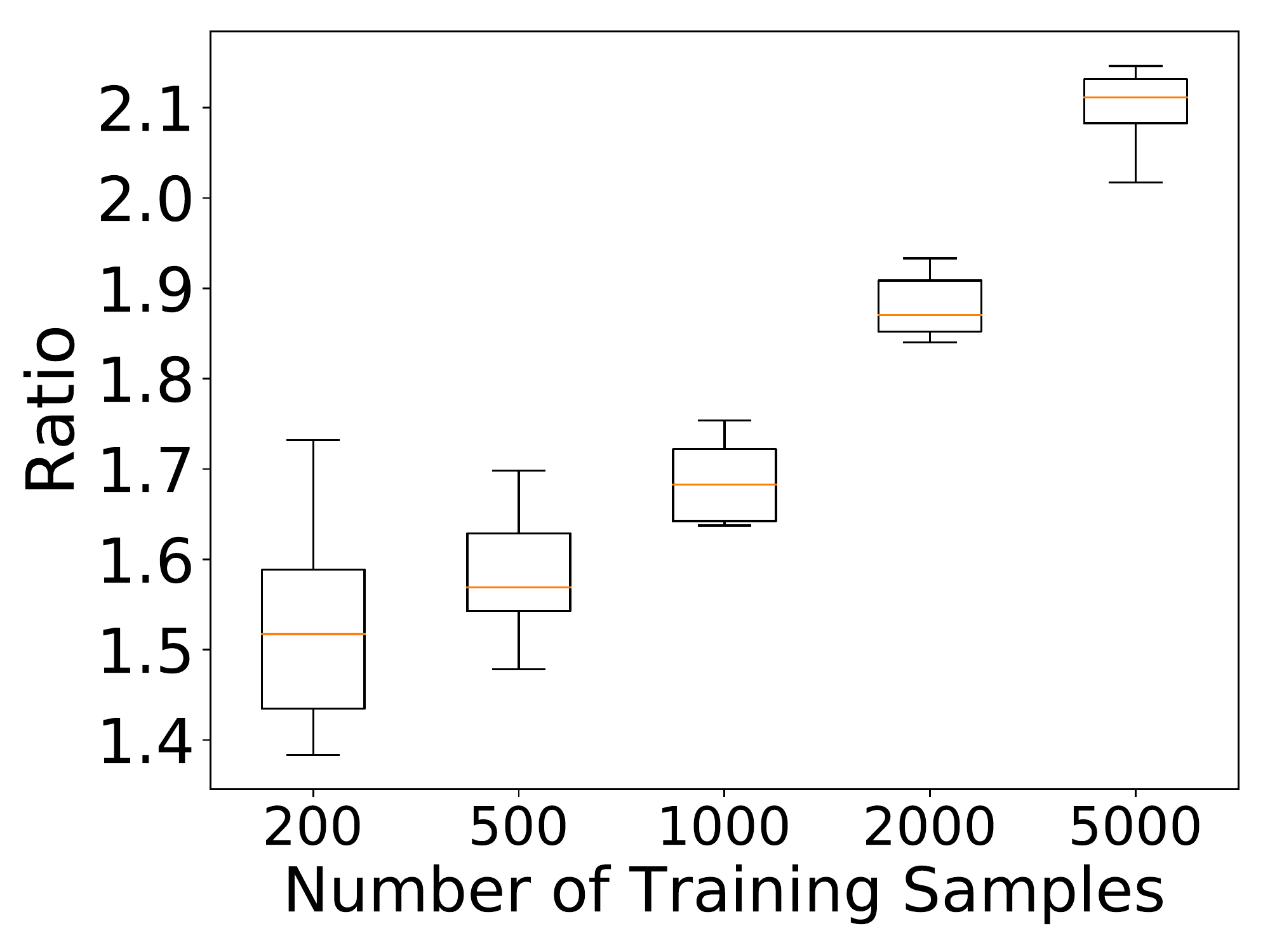}}
	\subfigure[Cadata, $\lambda = 0.2/\sqrt{n}$]{\includegraphics[width=0.3\textwidth]{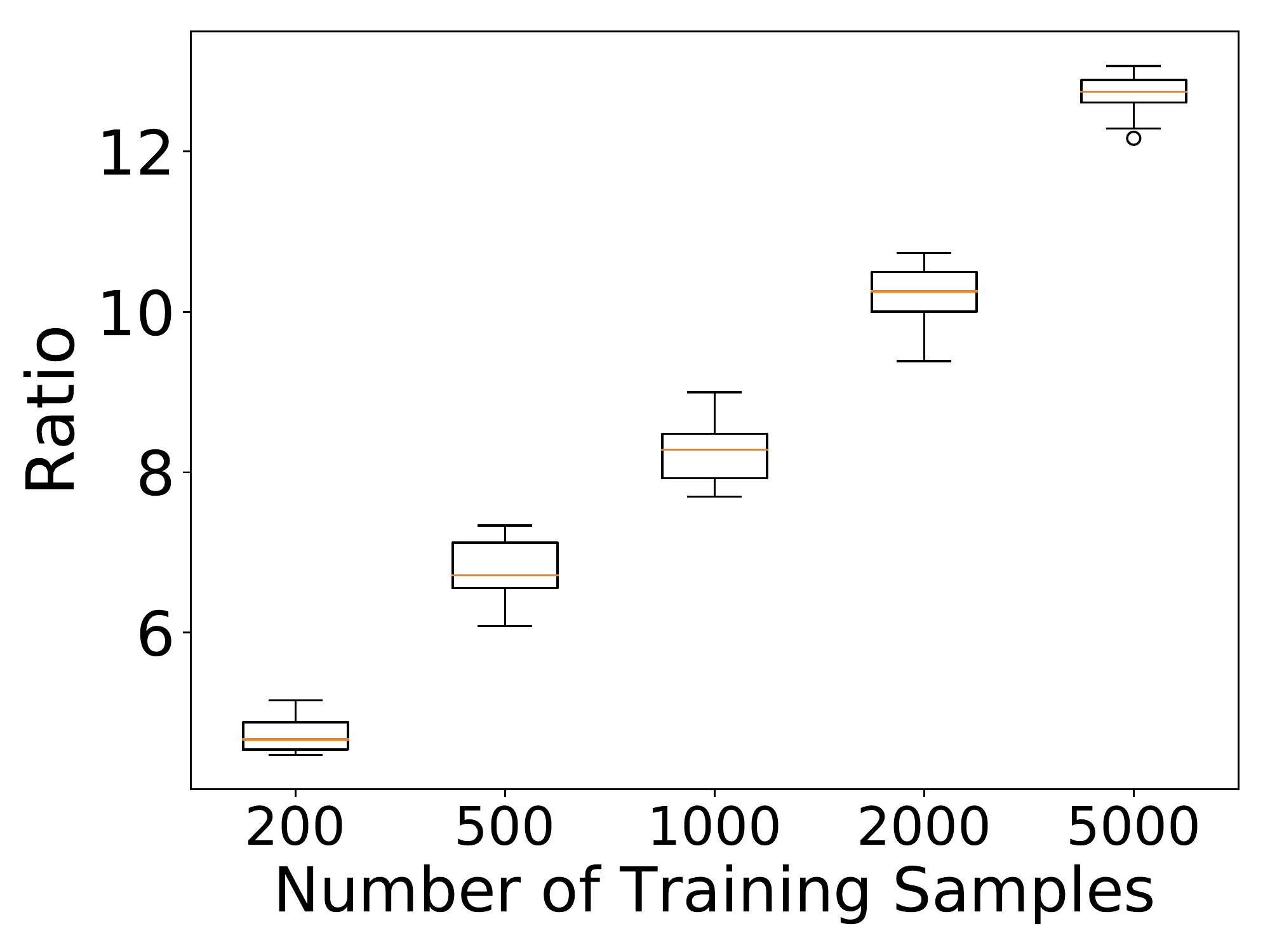}}
	\subfigure[Cadata, $\lambda = 1/\sqrt{n}$]{\includegraphics[width=0.3\textwidth]{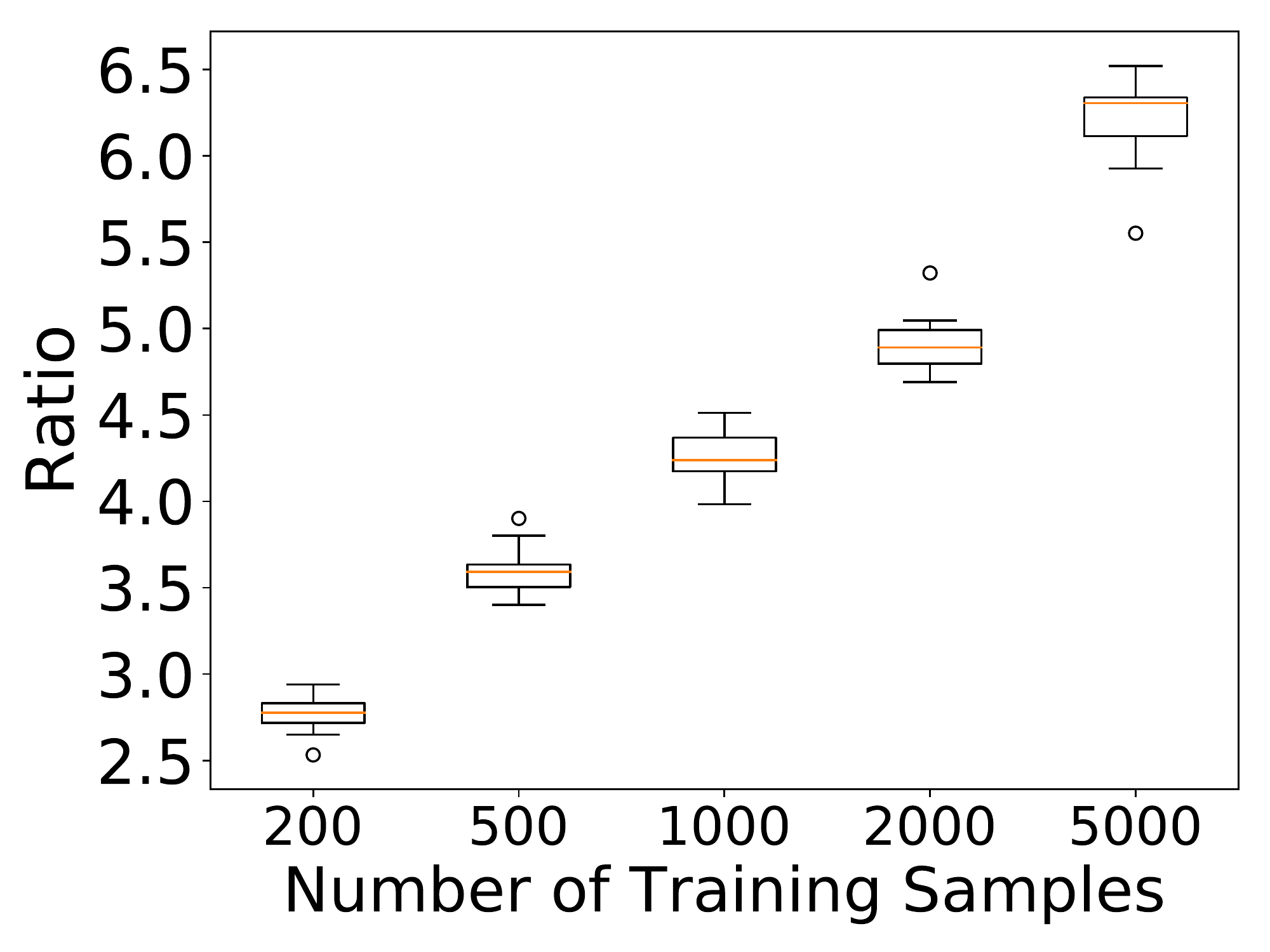}}
	\subfigure[Cadata, $\lambda = 5/\sqrt{n}$]{\includegraphics[width=0.3\textwidth]{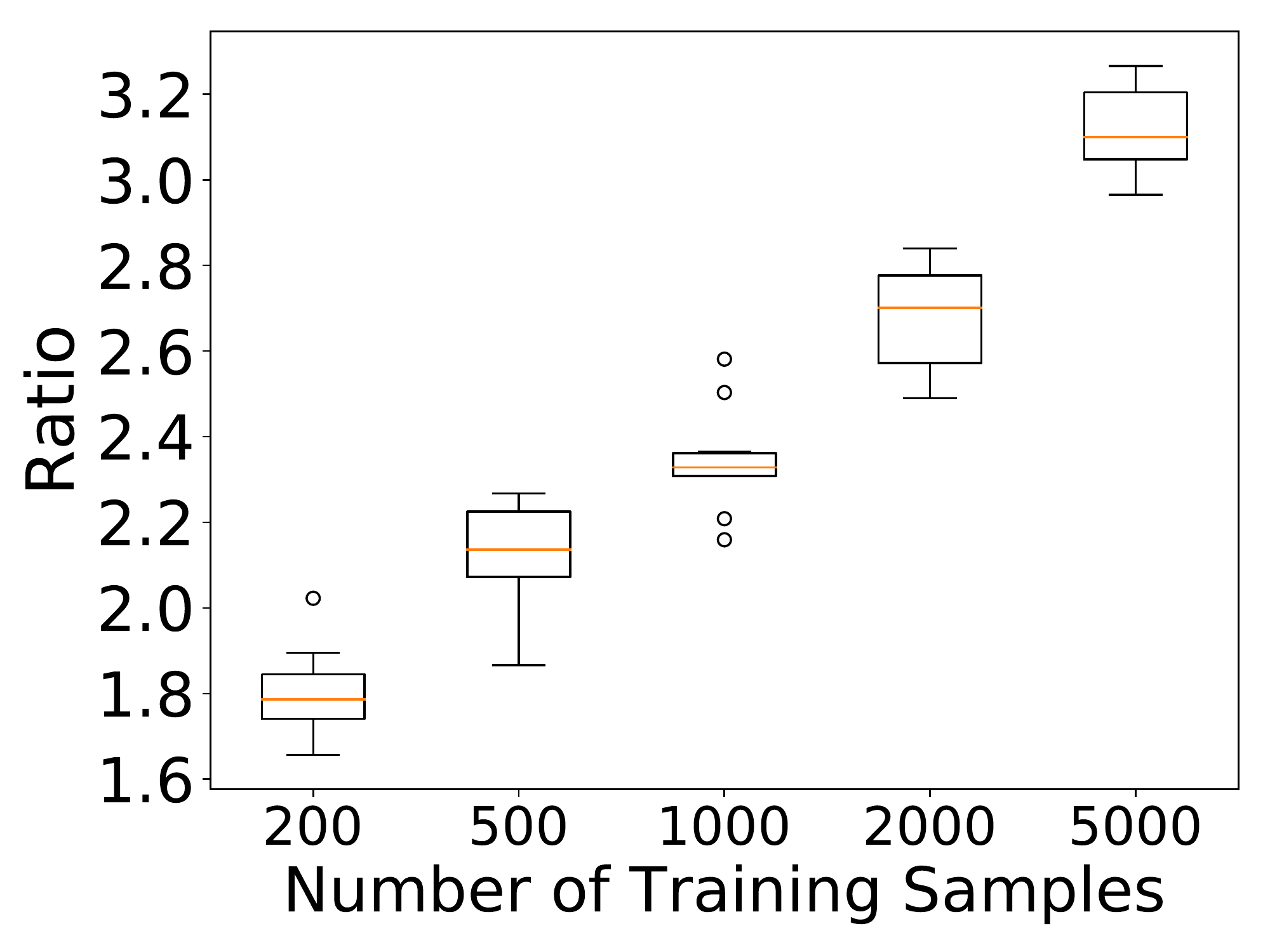}}
	\subfigure[Cpusmall, $\lambda = 0.2/\sqrt{n}$]{\includegraphics[width=0.3\textwidth]{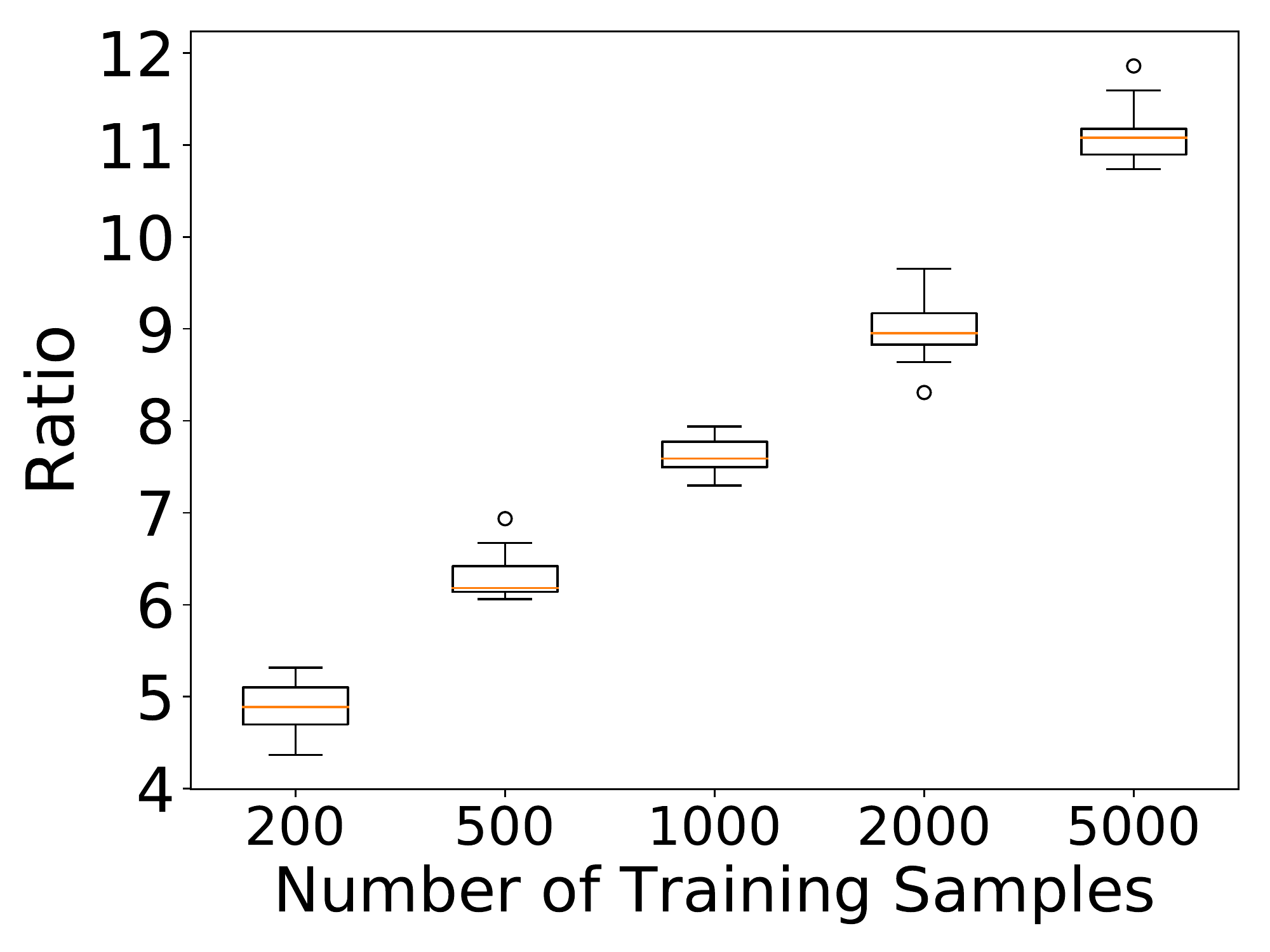}}
	\subfigure[Cpusmall, $\lambda = 1/\sqrt{n}$]{\includegraphics[width=0.3\textwidth]{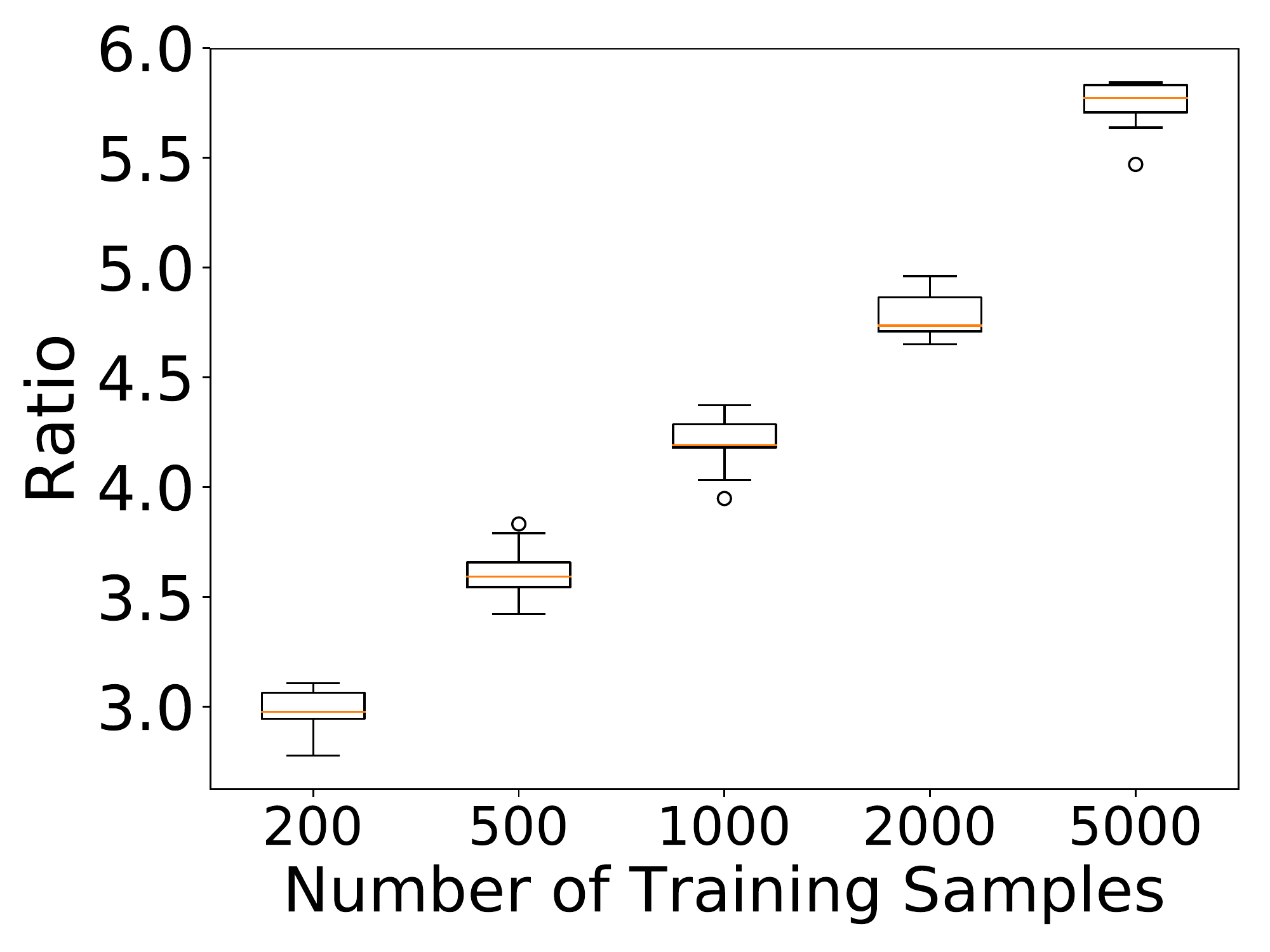}}
	\subfigure[Cpusmall, $\lambda = 5/\sqrt{n}$]{\includegraphics[width=0.3\textwidth]{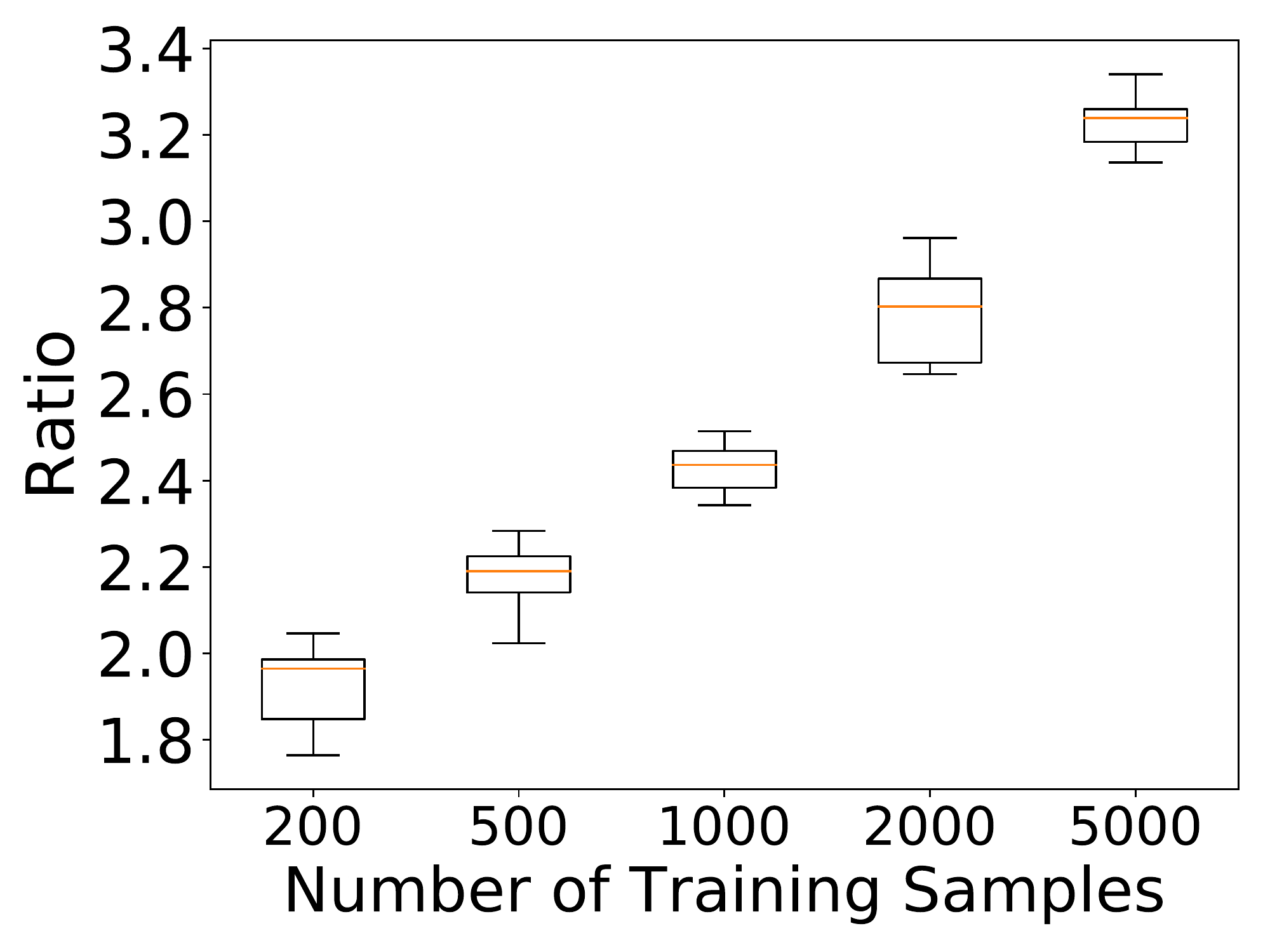}}
	\subfigure[Covtype, $\lambda = 0.2/\sqrt{n}$]{\includegraphics[width=0.3\textwidth]{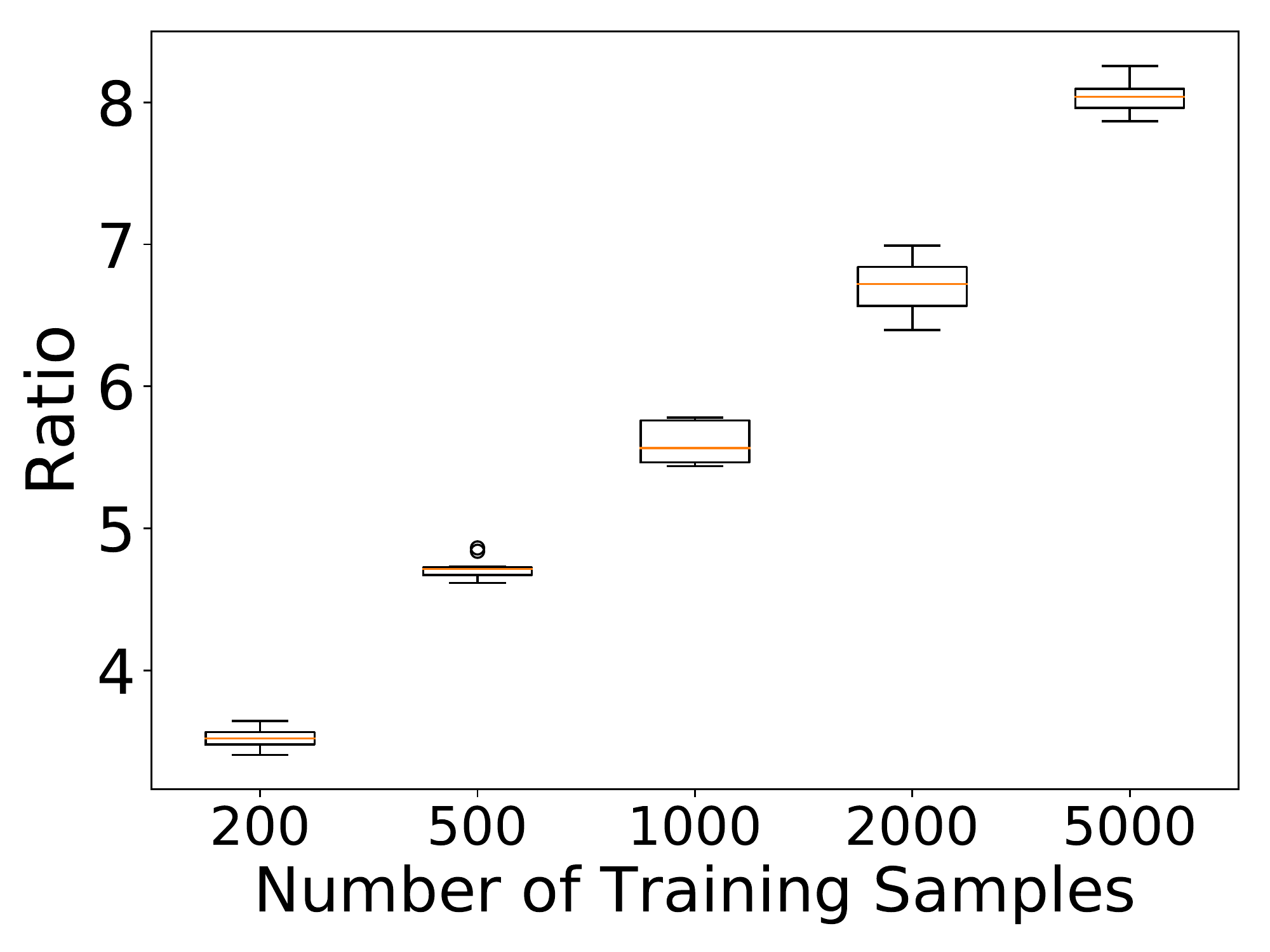}}
	\subfigure[Covtype, $\lambda = 1/\sqrt{n}$]{\includegraphics[width=0.3\textwidth]{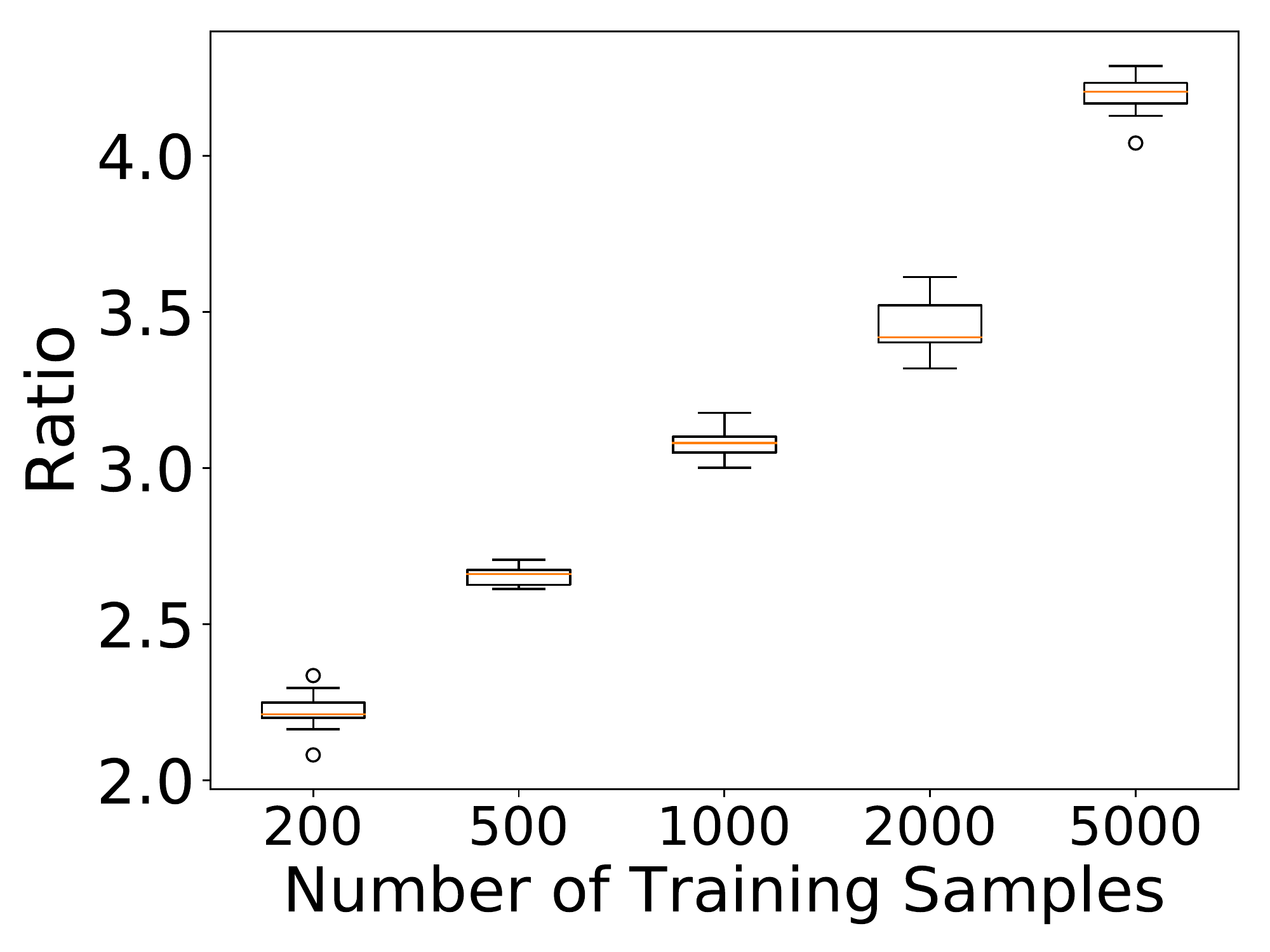}}
	\subfigure[Covtype, $\lambda = 5/\sqrt{n}$]{\includegraphics[width=0.3\textwidth]{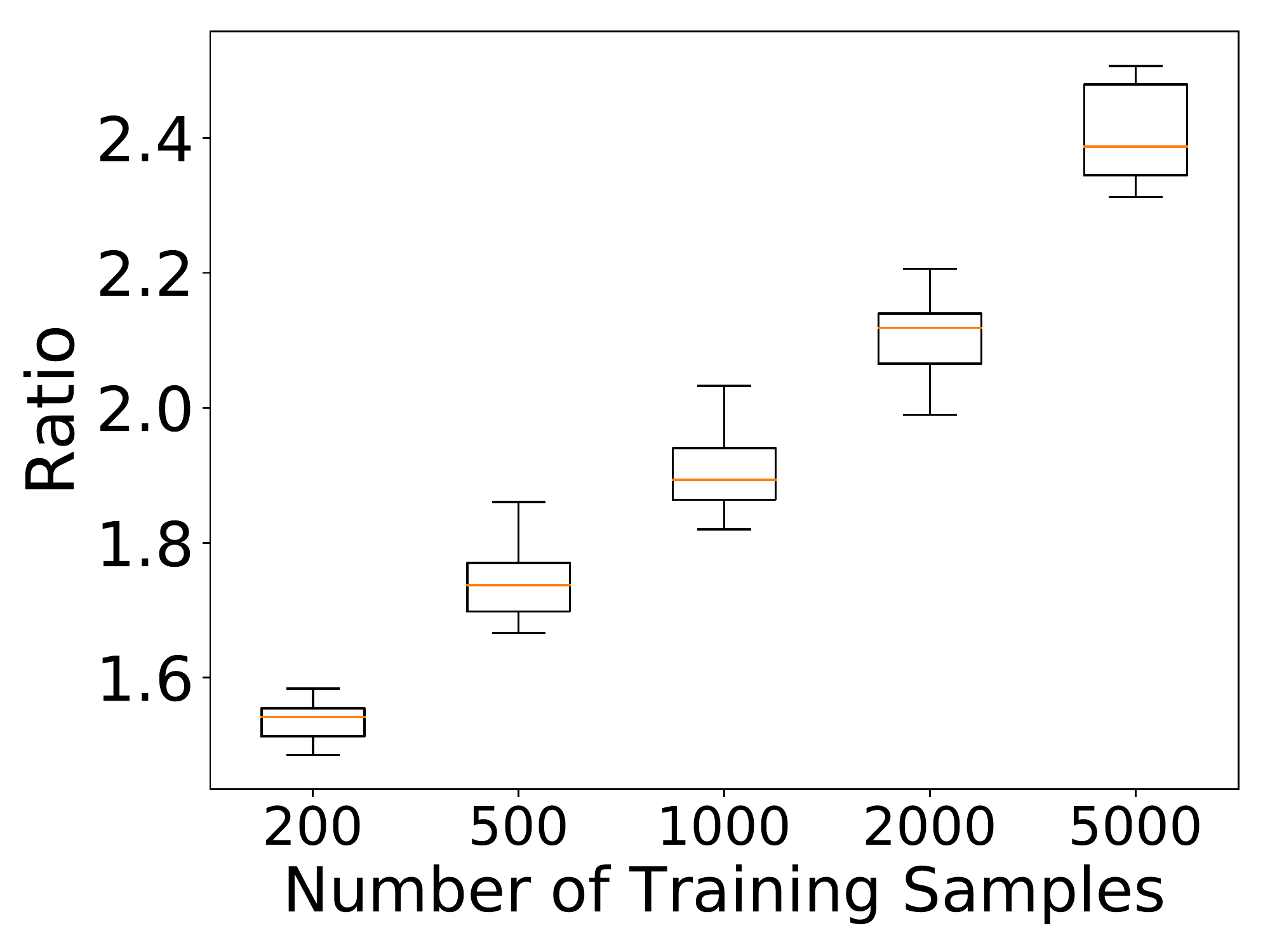}}
	\caption{Plot of the ratio $\frac{\textrm{Bound} }{ \textrm{MSE} } $ against $n$ (using the RBF kernel.)
		We fix $s = 100$.
	}
	\label{fig:laplace_n}
\end{figure}

\subsection{Evaluating the tightness of bound}\label{sec:exp:bound}

We empirically evaluate the tightness of the bound by comparing the bound with the MSE $\EB \big[ ( f_{\lambda} - \tilde{f}_\lambda )^2 \big]$.
Theorem~\ref{thm:main} establishes an upper bound for the MSE;\footnote{The bound in Theorem~\ref{thm:main} is actually $4b$ times larger than \eqref{eq:bound}. However, the $4b$ term is likely the artifect of our analysis.} we define
\begin{equation} \label{eq:bound}
\textrm{Bound} 
\: = \: \frac{1}{s} \big\| \K^{\frac{1}{2}} \big( \K + n \lambda \I_n \big)^{-1} \y \big\|_2^2 ,
\end{equation}
where $\K \in \RB^{n\times n}$ is the kernel matrix (of the training data) and $\y \in \RB^n$ contains the training targets.
We randomly select a subset of $n$ samples for training and another subset for test, and we repeat this process for $10$ times.
We fix $s=100$ and vary $n$ from $200$ to $5,000$.
We plot the ratio $\frac{\textrm{Bound}}{\textrm{MSE}}$ against $n$ in Figure~\ref{fig:rbf_n} (RBF kernel) and Figure~\ref{fig:laplace_n} (Laplace kernel).
Figures~\ref{fig:rbf_n} and \ref{fig:laplace_n} show that our bound does not much overestimate the actual MSE, especially when $\lambda \geq \frac{1}{\sqrt{n}}$.

\section{Conclusions}

We studied the generalization of random feature mapping (RFM) for kernel ridge regression (KRR).
We showed that with the regularization parameter set as $\lambda = \tilde{\Omega} (\frac{1}{\sqrt{n}})$, the prediction made by RFM-KRR converges to KRR at a rate of $\frac{1}{s}$ where $s$ is the number of random features.
This generalization bound is near optimal, as our established lower bound almost matches the upper bound.
Although stronger generalization bounds have been established by prior work,
they made restrictive and uncheckable assumptions on the data and kernel functions.
It is unclear whether the existing strong bounds are the nature of RFM or consequences of strong assumptions.
The uniqueness of this work is that we make only a checkable assumption on the RFM and no assumption on the data and kernel.

\acks{The author thanks Joel Tropp for offering very constructive suggestions.}

\bibliography{bib/matrix}
%\bibliographystyle{plain}
%\end{small}

\end{document}